\documentclass[a4paper,reqno]{amsart}

\usepackage{amsmath, amssymb, amsthm, eucal}
\usepackage{graphicx}
\usepackage[shortlabels]{enumitem}

\usepackage{caption}
\usepackage{subcaption}

\usepackage{placeins}
\usepackage{hyperref}
\usepackage{verbatim}
\usepackage{pdfsync}

\usepackage{comment}

\usepackage{tikz}
\usepackage{scalefnt}
\usepackage[textheight=630pt,
textwidth=468pt,
centering]{geometry}

\usepackage{xcolor}

\newtheorem{theorem}{Theorem}
\newtheorem{corollary}[theorem]{Corollary}
\newtheorem{lemma}[theorem]{Lemma}
\newtheorem{proposition}[theorem]{Proposition}

\newtheorem{remark}[theorem]{Remark}

\theoremstyle{definition}
\newtheorem{definition}[theorem]{Definition}

\newcommand{\norm}[1]{\left\Vert #1 \right\Vert}

\newcommand{\W}{\overrightarrow{W}}

\newcommand{\NN}{\ensuremath{\mathbb{N}}}

\newcommand{\RR}{\mathbb{R}}

\newcommand{\F}{\mathcal{F}}

\DeclareMathOperator{\E}{\mathbb{E}}

\begin{document}

	\title[Stochastic gradient descent in deep linear networks]
	{Convergence of stochastic gradient descent for learning linear neural networks} 
	
	\date{\today}

	\author{Gabin Maxime Nguegnang}
	
	\author{Bubacarr Bah}
	
	\author{Holger Rauhut}
	
	\author{Ulrich Terstiege}

\begin{abstract}
	 	Stochastic gradient descent (SGD) is the cornerstone of modern deep learning, yet its theoretical convergence properties remain poorly understood in non-convex settings. To address this gap, we study SGD dynamics in deep linear neural networks—simplified models that retain key non-convex characteristics while remaining analytically tractable. Building on recent results for gradient descent, we extend convergence guarantees to SGD under regularized squared loss. We establish almost sure boundedness of SGD iterates. 
        Using the framework of asymptotic pseudotrajectories from stochastic approximation theory, our analysis establishes then, for both deterministic and adaptive step-size choices, convergence of SGD to the set of critical points under mild conditions (without relying on Lipschitz properties or whitened data). To our knowledge, this is the first work providing almost sure convergence guarantees for SGD in deep linear networks, bridging the gap between theoretical optimization and practical stochastic training.
	
\end{abstract}

\maketitle

\section{Introduction}

Deep learning has delivered transformative breakthroughs in recent years, fueling innovations across domains such as computer vision, natural language processing, healthcare, and scientific computing. Central to this progress are iterative optimization methods, particularly stochastic gradient descent (SGD) and its variants, which enable the effective training of highly parameterized models.

Despite SGD’s empirical success in optimizing deep neural networks, its theoretical convergence properties remain only partially understood. This gap is most pronounced in non-convex settings, where the loss landscape is riddled with saddle points and local minima. Yet, SGD consistently navigates these challenges and produces models that generalize well \cite{zhang2017understanding, belkin2019reconciling}.

To bridge theory and practice, researchers often study simplified models such as deep linear neural networks multi-layer architectures composed solely of linear transformations. Although these networks ultimately represent linear mappings, their parameterization introduces non-convexity through matrix factorization. Consequently, they retain many 
characteristics of deep learning while providing a simplified entry-point for the analysis of optimization algorithms.
\cite{saxe2014exact, arora2018optimization}.

A prominent line of work examines optimization dynamics via \textit{gradient flow} (GF) the continuous-time limit of gradient descent with infinitesimal step sizes. Early studies revealed that depth-induced overparameterization can act as an implicit regularizer in certain situations, accelerating convergence \cite{arora2018optimization}. In particular, \cite{chitour2022geometric} and \cite{bah2022learning} basically showed that, under suitable dimensional assumptions, GF converges to global minima from almost all initializations. These results underscore how architectural properties such as width and depth influence training dynamics.

However, GF is an abstraction and does not capture the behavior of practical algorithms. This motivates the study of discrete-time \textit{gradient descent} (GD) with finite step sizes. For instance, \cite{bartlett2018gradient} proved GD convergence for linear residual networks initialized with identity mappings. Later, \cite{arora2019fine} generalized this to approximately balanced initializations, establishing global convergence under a deficiency margin condition. 
Extending these earlier results, \cite{elkabetz2021continuous} and \cite{nguegnang2024convergence} provided 
general convergence guarantees for GD under mild upper bounds on the step sizes, without relying on Lipschitz properties of whitened data.

In practice, \textit{stochastic} gradient descent is favored over full-batch GD for its computational efficiency on large datasets. Yet, its stochastic nature introduces additional complexities, and most existing convergence results assume bounded gradients or Lipschitz continuity conditions \cite{mertikopoulos2020almost,khaled2022better,ge2015escaping,ghadimi2013stochastic,ghadimi2016accelerated,bottou2018optimization,lei2019stochastic,vaswani2019fast}. The Lipschitz gradient property and its variants are fundamental for establishing the descent lemma, which guarantees decrease of the loss function. Moreover, Lipschitz gradient properties are also key to showing boundedness of the the second moment of the gradient, see \cite{ghadimi2013stochastic,lei2019stochastic} and references therein. In particular, Lipschitz gradient properties are key for most convergence theory of SGD, see also 
\cite{ge2015escaping,khaled2022better} and references therein. 

Unfortunately, the gradient of the empirical loss function is not globally Lipschitz for many neural network architectures. This includes deep linear networks when combined with squared loss. Therefore, existing general bounds, such as the ones in \cite{ghadimi2013stochastic, mertikopoulos2020almost, khaled2022better}, are not applicable -- unless one shows boundedness of the iterates directly. In fact, existing approaches use Lipschitz gradient properties to actually show boundedness. For this reason, there is a significant gap in the convergence theory of SGD when applied to deep learning scenarios. In order to make progress, it becomes essential to develop techniques for showing boundedness of the iterates without assuming global Lipschitz bounds on the gradient.


This paper addresses this gap by extending the convergence results of \cite{nguegnang2024convergence} from GD to SGD. We analyze deep linear networks trained with a regularized squared loss.  Our key theoretical tool is the framework of \textit{asymptotic pseudotrajectories} (APT) from stochastic approximation theory \cite{benaim1996dynamical}, which connects discrete SGD iterates to continuous-time GF trajectories.
Note that the convergence analysis of SGD (in particular, showing boundedness of the iterates) for the unregularized square-loss (as considered in \cite{arora2018optimization, arora2019fine,bah2022learning,nguegnang2024convergence} for GF/GD) turned out to be very challenging and is still open. This is the reason why we passed to a regularized loss, see \eqref{eq:regularized_loss} below, for which we are then able to show convergence (as already stated). Moreover, we believe that our analysis provides a blue-print for many steps of the convergence analysis of SGD for the unregularized loss.

Our contributions are threefold:
\begin{itemize}
	\item We prove the \textbf{almost sure boundedness} of SGD iterates under minimal assumptions (without relying on gradient Lipschitz continuity or uniform data whitening).
	\item We establish \textbf{almost sure convergence} of SGD to the set of critical points
    of the squared loss in deep linear networks, showing that SGD inherits GD’s convergence behavior under similar conditions.
	\item We develop a rigorous analytical framework linking SGD dynamics to gradient flow using stochastic approximation tools, offering new insights into stochastic training in deep networks.
\end{itemize}

To our knowledge, this is the first almost sure convergence result for SGD in deep linear networks in the present 
setting.
Our work narrows the gap between theoretical optimization results and practical stochastic training procedures in deep learning. 




	\section{Problem statement and results summary}\label{sec:problem&results}
	
%

	\subsection{Problem statement}\label{sec:problem}
We consider deep linear neural networks trained via over-parameterization, i.e., mappings of the form $X \mapsto W_N \cdots W_1 X$ with depth $N \in \NN$.  Here, $W_j \in \RR^{d_j \times d_{j-1}}$ for $j = 1, \dots, N$ denotes the weight matrix of the $j^{\mbox{th}}$ layer, and $X \in \RR^{d_x \times m}$ (where $d_x=d_0$) with data columns $x_1,\hdots,x_m \in \RR^{d_x}$ is the input data matrix. Throughout the paper, we will assume $N \geq 2$. The training loss 
for an output data matrix $Y \in \RR^{d_y \times m}$ (where $d_y=d_N$) with columns $y_1,\hdots,y_m \in \RR^{d_y}$ is defined as 
\begin{equation}\label{normal loss}
	L^N\big(\W\big) = \frac{1}{2m} \sum_{i=1}^m \|y_i - W_N \cdots W_1 x_i\|_2^2
    = \frac{1}{2m} \| Y - W_N \cdots W_1 X \|_F^2,
\end{equation}
where $\W = \big(W_1, \dots, W_N\big)$ and $\|\cdot\|_F$ denotes the Frobenius norm. Defining the product matrix $W=W_N\cdots W_1$ associated to $\W$ and 
\[
L^1\big(W) = \frac{1}{2m} \|Y-WX\|_F^2,
\]
it holds $L^N(\W) = L^1(W)$.

We introduce the regularized loss
\begin{align}\label{eq:regularized_loss}
	\Tilde{L}^N\big(\W(k)\big) = L^N\big(\W(k)\big) + R\big(\W(k)\big),
\end{align}
where for parameters $\beta > 0$ and $\lambda \geq 0$, the regularization term is defined as
\begin{equation}\label{eq:regularizer}
	R(\W) = R_{\beta,\lambda}(\W) = \lambda \left( \frac{1}{2} \|\W\|_F^2 \right)^{\beta},
\end{equation}
where we use the notation $\|\W\|_F = \left( \sum_{j=1}^N \|W_j\|_F^2\right)^{1/2}$.

The gradient flow associated with the regularized loss $\Tilde{L}^N$ is given by
\begin{align}\label{eq:sgf_w}
	\frac{dW_j}{dt} = -\nabla_{W_j} \Tilde{L}^N\big(\W\big).
\end{align}

For a subset $S \subset \{1, \dots, m\}$ chosen 
uniformly at random among all subsets of cardinality $T \leq m$ ($T = 1$ being allowed), we define the corresponding (mini batch) stochastic loss (over the randomly selected data points) as
\begin{align}
    L^N_S\big(\W\big) &= \frac{1}{2T} \sum_{i\in S} \|y_i - W_N\cdots W_1 y_i\|_2^2 = \frac{1}{2T} \| Y_{S} - W_N \cdots W_1 X_{S} \|_F^2,
\end{align}
where $Y_S$ and $X_S$ denote the submatrices of $Y$ and $X$ consisting only of the columns indexed by $S$.
Then the regularized stochastic loss 
$\Tilde{L}^N_{S}$, and its corresponding stochastic gradient $\nabla \Tilde{L}^N_{S}$ are defined as
\begin{align}
	\label{stochastic loss}
	\Tilde{L}^N_{S}\big(\W\big) &= L^N_S\big(\W\big) + R\big(\W\big), \\
	\label{stochastic gradient}
	\nabla \Tilde{L}^N_{S}\big(\W\big) &= \nabla L^N_{S}\big(\W\big) + \nabla R\big(\W\big).
\end{align}
At each iteration $k \in \NN$ of stochastic gradient descent (SGD), we select a subset $S_k \subset [m] = \{1, \dots, m\}$ uniformly at random among all subsets of size $T \leq m$ (with $T = 1$ allowed), independently of $\W(k)$. It then holds that 
\[
\E\big[\nabla \Tilde{L}^N_{S_k}\big(\W(k)\big) \mid \W(k)\big] = \nabla \Tilde{L}^N\big(\W(k)\big).
\]
The corresponding SGD update for each weight matrix $W_j$ is
\begin{align}\label{eq:sgd}
	W_j(k+1) &= W_j(k) - \eta_k \nabla_{W_j} \Tilde{L}^N_{S_k}\big(\W(k)\big) \\
	&= W_j(k) - \eta_k \Big( \nabla_{W_j} L^N_{S_k}\big(\W(k)\big) + \nabla_{W_j} R\big(\W(k)\big) \Big). \nonumber
\end{align}
Our goal is to analyze the convergence properties of the mini-batch SGD algorithm \eqref{eq:sgd} for training deep linear networks using the regularized stochastic loss \eqref{stochastic loss}.

\subsection{Results Summary}\label{sec:summary}

Our main result (Theorem \ref{thm:sgd-conv2cset}) shows that under mild conditions on the step size $\eta_k$ in \eqref{eq:sgd} and the regularization parameter $\beta$—specifically, $\beta \geq N$—we establish:
\begin{itemize}
	\item convergence of the SGD iterates to the critical set of $\Tilde{L}^N$;
	\item convergence of the regularized loss $\Tilde{L}^N$ evaluated along the SGD iterates; and
	\item validity of these results for both deterministic and adaptive step-size schedules.
\end{itemize}

The key intermediate results include bounds on the sequence of SGD iterates under both deterministic and adaptive step sizes (Propositions \ref{prop:bounded} and \ref{prop:bounded_adaptive}). A crucial component of our analysis is the concept of an asymptotic pseudotrajectory for the gradient flow, which corresponds to the continuous-time counterpart of gradient descent using the full gradient rather than stochastic estimates. The SGD sequence defined by \eqref{eq:sgd} can be viewed as an approximate solution to the gradient flow equation \eqref{eq:sgf_w}. To ensure convergence, we show that the interpolated process \eqref{eq:interpoolated_process_w} constructed from these iterates forms an asymptotic pseudotrajectory for the gradient flow (Proposition \ref{prop:apt_property}). This property guarantees that the interpolated process accurately tracks the gradient flow, enabling analysis of the long-term behavior of SGD. Furthermore, Proposition \ref{prop:bounded_adaptive} ensures that SGD with non-constant step sizes remains bounded as time tends to infinity.

\section{Preliminaries}\label{sec:prelim}
\subsection{Key ingredients}\label{sec:prelim-ingredients}
  
In this section we  summarize the main theoretical tools used in the derivation of the results in this manuscript. We mostly use stochastic approximation theory as presented in \cite{benaim1996asymptotic, benaim1999dynamics}, which focuses on the analysis of stochastic processes of the form 
\begin{align}\label{eq:SA}
	x_{k+1} = x_k + \eta_{k+1} H\left(x_k,\xi_{k+1}\right).
\end{align}
Note that to make the comparison to \cite{benaim1996asymptotic, benaim1999dynamics} easy, the involved stepsize  has index $k+1$ here as in \cite{benaim1996asymptotic, benaim1999dynamics}, whereas in our setting \ref{eq:sgd} the corresponding stepsize has index $k$.
In $\eqref{eq:SA}$, $H:\RR^m\times\RR^d\to\RR^m$ is a measurable function that defines the algorithm, the $\eta_k$ are step sizes, and $\xi_k$ are  random variables. 
A common approach in this field is to consider a vector field obtained by averaging the function $H$ as
\begin{align*}
	F(x)= \E \left[H(x,\xi)\right],
\end{align*}
where $\xi$ is the random variable (assuming the $\xi_k$ are identically distributed as $\xi$). 
Then, one re-interprets the discrete process in $\eqref{eq:SA}$ as an approximation to the solutions of the ODE
\begin{align}\label{eq:ODE_for_SA}
	\frac{dx}{dt}= F(x).
\end{align}
This way the stochastic term $H(x_k,\xi_{k+1})$ can be split into a deterministic term $F(x_k)$ and a noise term $U_{k+1}$,
\begin{align*}
	H\left(x_k,\xi_{k+1}\right)= F\left(x_k\right)+ U_{k+1}.
\end{align*}
Taking this into account \eqref{eq:SA} becomes 
\begin{align}\label{eq:SA_for_Robbin_Monro}
	x_{k+1} = x_k + \eta_{k+1} \left(F(x_k)+ U_{k+1} \right).
\end{align}
The concept of an asymptotic pseudotrajectory (APT) has proven very useful to analyze the long-term behavior of a continuous function through the flow trajectory of a dynamical system, defined for instance via the ODE \eqref{eq:ODE_for_SA}.
In order to introduce this concept we recall that a (semi-)flow on some 
metric space

$\mathcal{M}$ (later $\mathcal{M} = \RR^{d_x\times d_1} \times \RR^{d_1 \times d_2} \times \cdots \times \RR^{d_N \times d_y}$) is
a continuous function $\Phi : \RR_+ \times \mathcal{M} \to \mathcal{M}$, $\Phi(t,x) = \Phi_t(x)$, that satisfies $\Phi_0(x) = x$ and $\Phi_t \circ \Phi_s (x) = \Phi_{t+s}(x)$ for all $x \in \mathcal{M}$ and $(t,s) \in \RR_+^2$. We will always use the flow $\Phi$ mapping an initialization $x_0 = x(0)$ to the solution $\Phi_t(x_0) = x(t)$, of the ODE \eqref{eq:ODE_for_SA} at time $t\in [0,\infty)$. 

\begin{definition}[Asymptotic Pseudotrajectory (APT), \cite{benaim1996asymptotic}]\label{def:apt}
	Let $(M,d)$ be a metric space. An asymptotic pseudo trajectory for a flow $\Phi$ is a continuous function $X:\RR_+\to M$ that satisfies 
	\begin{align*}
		\lim_{t\to\infty}\sup_{h\in[0, T]}d\left(X(t+h), \Phi_h(X(t))\right)=0,
	\end{align*}
	for all $T>0.$
\end{definition}
We will work with the so-called interpolated process of the discrete process in \eqref{eq:SA_for_Robbin_Monro} as a candidate for an asymptotic pseudotrajectory, defined for $t \in [0,\infty)$ as 
\begin{align}\label{eq:interp_process}
	X(t)=x_k+\frac{t-\tau_k}{\tau_{k+1}-\tau_k}\left(x_{k+1}-x_k\right) \text{ if } t\in [\tau_k, \tau_{k+1}],~ k\in\NN,
\end{align}
where $\tau_0=0,$ $\tau_k=\displaystyle\sum_{l=1}^{k}\eta_l,$ for $k\geq1.$ We define the piecewise constant interpolation $\overline{U}$ and $\overline{\eta}$ on $[0 , \infty)$ as 
\begin{align}\label{def: U-bar eta-bar}
	\overline{U}(s)={U_{k+1}},\quad \text{and}\quad \overline{\eta}(s)={\eta_{k+1}}\quad \text{for $s\in [\tau_k, \tau_{k+1})$}.
\end{align} 
The number of iterations required for time $\tau_k$ to  attain the timestamp $t$ (in the next step) is given by
\begin{align*}
	m(t)=\sup\{k\geq0:~t\geq\tau_k\}.
\end{align*}
Under the APT property, for a fixed $T>0$ and a sufficiently large $t$, the interpolated process $X(t)$ closely follows the trajectory of the flow induced by the vector field $F$ over the interval $[0, T].$

Proposition~\ref{prop:apt_property} stated next provides conditions for the interpolated process \eqref{eq:interp_process} of the stochastic process \eqref{eq:SA_for_Robbin_Monro} to be an asymptotic pseudotrajectory for the flow induced by the vector field $F$. 
\begin{proposition}[Proposition 4.1 in \cite{benaim1999dynamics}]\label{prop:apt_property}
	Let $F$ be a continuous globally integrable 
	vector field 
    (i.e., the integral curves are defined on $\RR$ and are uniqueley determined). Consider a discrete process $\{x_k\}_{k\in\NN}$ in $\RR^m$ generated by F perturbed with a deterministic or random sequence $\{U_k\}_{k\in\NN}$ in $\RR^m$  defined as  
    \[x_{k+1}-x_k=\eta_{k+1}\left(F(x_k)+U_{k+1}\right).
    \]
    We assume that the stepsizes $\eta_k > 0$ satisfy 
    $\eta_k\to 0$ and $\displaystyle\sum_{k=0}^{\infty}\eta_k=\infty$. Suppose that 
	\begin{enumerate}
		\item For all $T>0,$
		\begin{align*}
			\limsup_{k\to\infty}\left\{\left\|\sum_{i=k}^{n-1}\eta_{i+1}U_{i+1}\right\| : n=k+1,\dots,m(\tau_k+T)\right\}=0,
		\end{align*}
		or equivalently
		\begin{align*}
			\lim_{t\to\infty}\Delta(t,T)=0 \text{ with} ~\Delta(t,T)=\sup_{0\leq h\leq T}\left\|\int_{t}^{t+h}\overline{U}(s)ds\right\|,
		\end{align*}
		where $\overline{U}$ is as defined in $\eqref{def: U-bar eta-bar}.$
		\item $\displaystyle\sup_{k}\|x_k\|<\infty,$ or
		\item $F$ is Lipschitz and bounded on a neighborhood of $\{x_k : k\geq0\}$.
	\end{enumerate}
	Then, the interpolated process X is an asymptotic pseudotrajectory of the flow $\phi$ induced by $F$. 
\end{proposition}

\begin{remark} Inspecting the proof of this proposition in \cite{benaim1999dynamics} shows that its conclusion remains valid if instead of the assumption that $F$ is globally integrable, we only assume that the trajectories of $F$ exist on $\RR_+$ and are uniquely determined. 
\end{remark}

The boundedness property required by the second condition is central to studying the long-term behavior of the process. However, establishing this property for discrete processes is challenging. 
It is this difficulty that leads most prior work on stochastic processes to simply assume boundedness. Other authors impose the Lipschitz condition from the third condition instead, and use it to derive an almost sure bound on the iterates \cite{mertikopoulos2020almost}.

In contrast, 
we will give a careful analysis of the SGD iterates and prove that they remain bounded almost surely, see Section \ref{sec:boundedness}.

As seen in Proposition \ref{prop:apt_property}, the first critical condition for APT is technical.
However, Proposition \ref{prop:apt_property} gives no route to establish it. Proposition \ref{prop:crucial_apt_tool} fills this gap and constitutes a reference tool that helps to show this technical condition. It relies on Robbins-Monro algorithms, also known as Martingale Difference Noise. We first define Robbins-Monro algorithms and then state Proposition \ref{prop:crucial_apt_tool}.

\begin{definition}[Robbins-Monro Algorithm \cite{benaim1996asymptotic}]\label{def:robbins_monro_algo}
	Consider a sequence of positive integers $\{\eta_k\}_{k\in\NN}$ which satisfies $\displaystyle\sum_{k=1}^{\infty}\eta_k=\infty$ and $\displaystyle \lim_{k\to\infty}\eta_k=0.$ Let $F: \RR^m\to\RR^m$ be a continuous map and $\{x_k\}_{k\in\NN}$ a process in $\RR^m$ defined as follows:
	\begin{align}\label{eq:robbins_monro_process}
		x_{k+1} = x_k + \eta_{k+1} \left(F(x_k) + U_{k+1}\right),
	\end{align}
	where $U_k\in\RR^m$ are (deterministic or random) perturbations. Let $\left(\Omega, \mathcal{F}, P\right)$ be a probability space and $\{\mathcal{F}_k\}$ an increasing sequence of sub-$\sigma$-algebras of $\mathcal{F}.$ 
	
	The stochastic process $\{x_k\}$ defined in \eqref{eq:robbins_monro_process} satisfies the Robbins-Monro or Martingale Difference Noise Condition if the following hold
	\begin{itemize}
		\item [(a)] $\{\eta_k\}_{k\in\NN}$ is a deterministic sequence.
		\item [(b)] $\{U_k\}_{k\in\NN}$ is adapted: $U_k$ is measurable with respect to $\mathcal{F}_k$ for each $k\geq0.$
		\item[(c)] $\E[U_{k+1}|\mathcal{F}_k]=0.$ 
	\end{itemize}
\end{definition}

\begin{proposition}[Proposition 4.2, \cite{benaim1996asymptotic}] \label{prop:crucial_apt_tool}
	Let the process $\{x_k\}_{k\in\NN}$ defined in \eqref{eq:robbins_monro_process} be a Robbins-Monro algorithm. Assume that for some $q\geq 2$
	\begin{align*}
		\sup_{k}\E\left[\|U_{k}\|^q\right] < \infty 
	\end{align*}
	and
	\begin{align*}
		\sum_{k=0}^{\infty}\eta_k^{1+\frac{q}{2}} < \infty.
	\end{align*}
	Then with probability one condition (1) in Proposition \ref{prop:apt_property} holds, i.e., 
	for all $T>0,$
	\begin{align*}
		\limsup_{n\to\infty}\left\{\left\|\sum_{i=n}^{k-1}\eta_{i+1}U_{i+1}\right\| : k=n+1,\dots,m(\tau_n+T)\right\}=0,
	\end{align*}
	or equivalently
	\begin{align*}
		\lim_{t\to\infty}\Delta(t,T)=0 \text{ with} ~\Delta(t,T)=\sup_{0\leq h\leq T}\left\|\int_{t}^{t+h}\overline{U}(s)ds\right\|
	\end{align*}
	where $\overline{U}$ is as defined in $\eqref{def: U-bar eta-bar}$
\end{proposition}
Proposition \ref{prop:crucial_apt_tool} only covers deterministic stepsizes $\eta_k$ as it uses the Robbins-Monro sequence. Later, in Proposition \ref{prop:crucial_apt_tool_update}, we discuss additional conditions on the noise $U_k$ and the stepsizes $\eta_k$ that lead to the conclusion of Proposition \ref{prop:crucial_apt_tool} with stochastic (adaptive) stepsizes.

In our setting, the vector field $F$ is the negative gradient of the loss function $\Tilde{L}^N$ and the noise term is the difference between the stochastic gradient and the full gradient.
Within this stochastic approximation framework, Section~\ref{sec:theory} proves that the stochastic gradient iterations \eqref{eq:sgd} satisfy the Martingale Difference Noise condition.
The interpolated discrete SGD process will then define a continuous SGD process. With this at hand, we will establish 
that the interpolated process of SGD in \eqref{eq:sgd} is an APT for the gradient flow \eqref{eq:sgf_w}. 

Lyapunov functions provide another key tool for understanding the stability of the dynamical system \eqref{eq:sgf_w}. The basic characteristics of a Lyapunov function and the notion of an equilibrium point (an important concept for defining a Lyapunov function) are provided in Definitions~\ref{def:lyapunov_fn} and \ref{def:eqlbrm_pt}, respectively, in the Appendix.
%
Lyapunov functions help to determine whether a system converges to its equilibrium point as it evolves over time. 
As we will see in Section \ref{sec:theory}, the functional $\Tilde{L}^N$ itself is a  Lyapunov function, which will be a key ingredient for our convergence analysis.

We proceed by recalling further concepts from \cite{benaim1999dynamics} that are critical to the understanding of asymptotic pseudotrajectories.
For completeness, we provide the definition of a pseudo-orbit in the appendix (Definition~\ref{def:pseudo_orbit}), describing a particular finite sequence of flow trajectories that is useful in characterizing the flow and the set upon which it is defined.
The limit set of an asymptotic pseudotrajectory is given in 
the next defintion. 
\begin{definition}[Limit Set of an APT]\label{def:apt_limit_set}
	Let $\mathcal{M}$ be a metric space and let $X~:[0,\infty)\to \mathcal{M}$ be an asymptotic pseudotrajectory of a semiflow $\Phi$. We define $$\displaystyle L(X){ := }\left\{m  \in \mathcal{M}: \;\text{there exists $(t_k)_k$ 
    with $\lim_{k\to \infty} t_k =\infty$
    such that} \; \lim_{t_k\to \infty} X(t_k)=m \right\},$$ 
	the set of all limits of convergent sequences $X(t_k),$ as $t_k\to\infty.$
\end{definition}
Theorem~\ref{thm:chain_transitive_set} in the appendix characterizes the limit set of a precompact pseudotrajectory $X$ for a flow $\Phi$, in particular, it indicates that the flow $\Phi$ is internally chain transitive (see Definition~\ref{def:pseudo_orbit} (2)) on the limit set.
%
%
Moreover, Proposition \ref{prop:chain_transitive_set} in the appendix shows that a Lyapunov function can be constant on a chain internally transitive set. We will see in Section \ref{sec:theory} how Theorem~\ref{thm:chain_transitive_set} and Proposition~\ref{prop:chain_transitive_set} contribute to deriving the convergence property of SGD.

Proposition~\ref{prop:crucial_apt_tool_update} below is concerned with random (possibly adaptive) stepsizes. It provides conditions that the noise $U_k$ and the stepsizes $\eta_k$ must meet for condition $(1)$ in Proposition~\ref{prop:apt_property} to hold.
Proposition~\ref{prop:crucial_apt_tool_update} appeared in \cite[Remark 4.3]{benaim1999dynamics}, but is  explicitly proved in that paper only in the deterministic case. Thus for completeness we provide a proof in  Section~\ref{prf:crucial_apt_tool_update}.

\begin{proposition}\label{prop:crucial_apt_tool_update}
	Let $\{\F_k\}$ and $\{U_k\}$  be as in Definition \ref{def:robbins_monro_algo}. Assume that the stepsizes $\eta_{k+1}$ are $\F_k$ measurable for all $k \in \NN$. Suppose that for some $q\geq 2$ the noise $U_k$ satisfies
	\begin{align*}
		\sup_{k}\E\left[\|U_{k+1}\|^q\mid \F_k\right] \leq C,\quad \text{for some deterministic constant $C<\infty$}
	\end{align*}
	and the stepsizes $\eta_k\geq 0$ satisfy $\displaystyle\sum_{k=1}^{\infty}\eta_k=\infty$ almost surely and 
	\begin{align*} 
		\E\left[\sum_{k=1}^{\infty}\eta_k^{1+\frac{q}{2}}\right] < \infty.
	\end{align*}
Then the conclusion of Proposition \ref{prop:crucial_apt_tool} holds, i.e.,  with probability one, 
	for all $T>0$
	\begin{align*}
		\limsup_{n\to\infty}\left\{\left\|\sum_{i=n}^{k-1}\eta_{i+1}U_{i+1}\right\| : k=n+1,\dots,m(\tau_n+T)\right\}=0,
	\end{align*}
	or equivalently
	\begin{align*}
		\lim_{t\to\infty}\Delta(t,T)=0 \quad \text{ with} \quad \Delta(t,T)=\sup_{0\leq h\leq T}\left\|\int_{t}^{t+h}\overline{U}(s)ds\right\|.
	\end{align*}
\end{proposition}


Corollary \ref{cor:apt_convergence} is a key tool for the convergence analysis of stochastic approximation. It provides conditions on a Lyapunov function that ensure convergence of an APT.  
We will use it to prove  convergence of $\Tilde{L}^N\big(\W(k)\big)$ to the critical set.
\begin{corollary}[Corollary 6.7, \cite{benaim1999dynamics}]\label{cor:apt_convergence}
	Assume $\mathcal{M}$ is a smooth $C^r$ Riemannian manifold of dimension $m\geq 1.$ Consider a $C^r$ map $V: \mathcal{M}\to\RR$ and let $F$ be the gradient vector field
	\begin{align*}
		F(x)=-\nabla V(x).
	\end{align*}
	Further assume
	\begin{enumerate}
		\item F induces a global flow $\Phi$
		\item $X$ is a precompact asymptotic pseudotrajectory of $\Phi$
		\item $r\geq m$
	\end{enumerate}
	Then $L(X)$ consists of equilibria and $V(X(t))$ converges as $t$ tends to $\infty.$  
\end{corollary}

Condition (1) here means that for any $x$, the differential equation $\dot \Phi_x(t)=F(\Phi_x(t)),\ \Phi_x(0)=x$ has a unique solution defined for all $t\in \mathbb R$. 

 To highlight some of the crucial ingredients to our result and because we need a slightly stronger variant of this corollary (see the remark below), we repeat the proof that can be found in \cite{benaim1999dynamics}, involving a few adaptions and  additional details, see Section \ref{prf:apt_convergence}.  

\begin{remark} \label{adapted_assumption} Corollary  \ref{cor:apt_convergence} remains valid if  assumption (1) is replaced by the  assumption: 
    For any $x\in \mathcal M$, there is an $\varepsilon(x)>0$ such that the differential equation $\dot \Phi_x(t)=F(\Phi_x(t)),\ \Phi_x(0)=x$ has a  solution defined for all $t\in (-\varepsilon(x),\infty)$ and the solution is unique where it exists (that is,  if $u$ and $v$ are solutions defined on intervals $I_u$ resp. $I_v$ that both contain $0$ and  if $u(0)=v(0)=x$, then $u=v$ on $I_u\cap I_v$). 
\end{remark}

Under this assumption (and using that $F$ is continuous) the semiflow map $\Phi\colon \mathbb R_+\times \mathcal{M} \to \mathcal M, \ (t,x)\mapsto \Phi_x(t)$ is continuous, this follows e.g. from \cite[Theorem 4.3]{CoddingtonLevinson1955}.
Also the flow properties $\Phi_0(x) = x$ and $\Phi_t \circ \Phi_s (x) = \Phi_{t+s}(x)$ are clearly fulfilled. So we obtain indeed a semiflow.

\subsection{Our setting}\label{sec:prelim-setting}
Now we adapt some of these analysis tools into our problem setting. The basic training problem of linear neural networks, we are considering here, corresponds to the nonconvex optimization problem of the square loss function with a regularization term. 
We aim at minimizing the regularized square loss function $\Tilde{L}^N$ $\eqref{eq:regularized_loss}$ from Section \ref{sec:problem&results} via stochastic gradient descent. 
We study the long-term behavior of SGD \eqref{eq:sgd} through the asymptotic properties of the gradient flow equation \ref{eq:sgf_w}. 
The key link is that the SGD process \eqref{eq:sgd} approximates the solutions of the gradient flow equation \eqref{eq:sgf_w}.

We denote the flow originating from \eqref{eq:sgf_w} as $\phi.$ 
Definition \ref{def:apt4lnn} adapts the notion of asymptotic pseudotrajectory (APT) introduced in \cite{benaim1999dynamics} to the context of learning linear neural networks.
\begin{definition}\label{def:apt4lnn}
	A continuous function 
    $\overrightarrow{X}:\RR_+\to \RR^{d_N\times d_{N-1}}\times \dots \times\RR^{d_1\times d_{0}}$
    is an asymptotic pseudotrajectory for $\phi$ if \begin{align*}
		\lim_{t\to\infty}\sup_{0\leq h\leq T}\norm{\overrightarrow{X}(t+h) - \phi_h(\overrightarrow{X}(t))}_F = 0,
	\end{align*} 
	for any $T>0$ and $j=1,\dots,N.$
\end{definition}

We introduce the noise term
\begin{align}\label{grad noise}
	\overrightarrow{U}(k+1)=\nabla \Tilde{L}^N_{S_k}\big(\W(k)\big)-\nabla \Tilde{L}^N\big(\W(k)\big),
\end{align}
so that the stochastic gradient can be written over iterations as the deterministic gradient plus the noise
\begin{align*}
	\nabla\Tilde{L}^N_{S_k}\big(\W(k)\big)= \nabla \Tilde{L}^N\big(\W(k)\big) +\overrightarrow{U}(k+1).
\end{align*}
We denote by $\mathcal{F}_k:=\sigma\left(\left\{\W(0),\W(1),\dots,\W(k)\right\}\right)$  the $\sigma$-algebra generated by $\left\{\W(0),\W(1),\dots,\W(k)\right\}$, for $k\in\NN$. 
Note that this gives rise to a natural filtration of the $\mathcal{F}_k$.
 

We set $\tau_k=\displaystyle\sum_{l=1}^{k}\eta_l,$ and introduce the interpolated process $\overrightarrow{W}(t)$, $t \in [0,\infty)$, as
\begin{align}\label{eq:interpoolated_process_w}
	\overrightarrow{W} (t)=\W(k)+\frac{t-\tau_k}{\tau_{k+1}-\tau_k}\left(\W(k+1)-\W(k)\right) \text{ for } t\in [\tau_k, \tau_{k+1}],~ k\in\NN.
\end{align}

\section{Theoretical Results}\label{sec:theory}  
In this section, we will always consider the SGD iterates $\overrightarrow{W}(k)$ defined in \eqref{eq:sgd} for the regularized loss \eqref{eq:regularized_loss} with the regularizer \eqref{eq:regularizer}. 
 The convergence analysis of the discrete process $\W(k)$ hinges on proving that its interpolation $\overrightarrow{W}(t)$ is an asymptotic pseudotrajectory. We achieve this by showing that $\overrightarrow{W}(t)$ meets the conditions of Proposition~\ref{prop:apt_property}.
We start by defining the constants
\begin{align}
	P = P_T & :=\frac{1}{T} \max_{S\subset [m], |S| = T} \|X_S\|^2,\label{def:p_t}\\
	Q = Q_T & := \frac{1}{T}\max_{S\subset [m], |S| = T} \|Y_S X_S^T\|.\label{def:q_t}
\end{align}

\subsection{Boundedness of SGD iterates}\label{sec:boundedness} 
  We start by showing boundedness of the SGD iterates both for deterministic and stochastic (adaptive) cases. 
  
  \subsubsection{Deterministic step sizes}
  The next proposition 
  guarantees boundedness of the SGD iterates for learning deep linear neural networks with 
  deterministic  stepsizes.
\begin{proposition}\label{prop:bounded} Let $\lambda > 0$ and $\beta \geq N$ be regularization parameters.
	Set $c_\beta = \beta 2^{1-\beta}$.
	If $\beta > N,$  we set
	\[
	M = \max\left\{\left(\frac{2P_T\sqrt{N}}{\lambda c_\beta}\right)^{\frac{1}{2(\beta - N)}},  \left(\frac{2Q_T\sqrt{N}}{\lambda c_\beta}\right)^{\frac{1}{2\beta-N}}\right\}.\]
	If $\beta = N$, we assume that $\lambda > \frac{P_T\sqrt{N}}{c_\beta}$ and set
	\[
	M = \left(\frac{Q_T\sqrt{N}}{\lambda c_\beta - P_T\sqrt{N}}\right)^{\frac{1}{N}}.\]
	Let
	\begin{equation}
		B := \frac{\sqrt{N}\left(P_T M^{2N-1} + Q_T M^{N-1}\right)}{\lambda \beta \max\left\{\frac{1}{2}\|\W(0)\|_F^2, \frac{1}{2}M^2\right\}^{\beta-1}}, \label{def:B}
	\end{equation}
	and schedule the stepsizes as
	\begin{equation}\label{def-eta-k}
		\eta_k = \frac{b_k}{\beta\lambda \max\left\{\left(\frac{1}{2}\|\W(0)\|_F^2\right), \frac{1}{2}(M+B)^2 \right\}^{\beta-1}} ,
	\end{equation}
	where $(b_k)$ is a (deterministic) sequence satisfying $0 < b_k \leq 1$. 
	Then the stochastic gradient descent sequence $\left(\W(k)\right)_{k \in \NN}$ is bounded. More precisely,
	\begin{equation}\label{eq:bound_w}
		\|\W(k)\|_F \leq \max\{ \|\W(0)\|_F, M+B\} \quad\mbox{ for all } k \in \NN_0.
	\end{equation}
\end{proposition}
The proof of this proposition is given in Section \ref{prf:bounded}.
   
\subsubsection{Adaptive step sizes}
Proposition \ref{prop:bounded} employs deterministic stepsizes to provide the boundedness of SGD. Proposition \ref{prop:bounded_adaptive} below indicates that SGD run with adaptive stepsizes also remains bounded.

%
%
%
\begin{proposition}\label{prop:bounded_adaptive}
    For $\lambda, \beta, M$ as in  Proposition~\ref{prop:bounded} and an  arbitrary constant $D>0$, assume that 
	\begin{align}\label{def adaptive etak}
		\eta_k = b_k \cdot \min\left\{  \frac{1}{\beta\lambda(\frac12\|\W(k)\|_F^2)^{\beta-1}},\, D\, \right\}, 
	\end{align} 
	where $(b_k)_{k \in \NN_0}$ is a (deterministic) 
	sequence satisfying $0 < b_k \leq 1$. (In the case $\W(k)=0$, the condition means $\eta_k=Db_k$.) Then the stochastic gradient descent sequence $\left\{\W(k)\right\}_{k\geq0}$ is bounded. More precisely,
	\begin{equation}\label{eq:bound_adaptive}
		\|\W(k)\|_F \leq \max\left\{ \|\W(0)\|_F, M +D\sqrt N(P_TM^{2N-1}+Q_TM^{N-1})\right\} \quad\mbox{ for all } k \in \NN_0.
	\end{equation}
\end{proposition}
Note that in practice, one would choose the constant $D$ rather large, so that the (adaptive) first term in the minimum in \eqref{def adaptive etak} applies most of the time.
The proof of this proposition is given in Section~\ref{prf:bounded_adaptive}.

The boundedness of the SGD iterates is crucial, see condition (2) in Proposition \ref{prop:apt_property}. It will be shown in Propositions~\ref{prop:bounded} and \ref{prop:bounded_adaptive}
and is instrumental 
to show that the interpolated process \eqref{eq:interpoolated_process_w} is an asymptotic pseudotrajectory.



\subsection{Convergence of SGD to the set of critical points}\label{sec:convergence}

This section states our main result on convergence of SGD. It identifies practical step sizes that guarantee convergence of the SGD to the critical set when training linear neural networks.


We denote by 
$$\operatorname{Crit}(\Tilde{L}^N): =\left\{\W\ |\ \nabla\Tilde{L}^N(\W)=0\right\}$$ the critical set of $\Tilde{L}^N$ 
and 
$$\operatorname{d}\big(\W, \operatorname{Crit}(\Tilde{L}^N) \big) = \inf_{\overrightarrow{V} \in \operatorname{Crit}(\Tilde{L}^N)} \|\overrightarrow{W} - \overrightarrow{V}\|_F
$$ 
the Euclidean distance of $\W$ to $\operatorname{Crit}(\Tilde{L}^N)$.
The next theorem summarizes the main findings of this chapter. 

\begin{theorem}\label{thm:sgd-conv2cset}
Let  $(b_k)$ be a deterministic sequence that satisfies $0 < b_k \leq 1,$ $\displaystyle\sum_{k=0}^{\infty}b_k=\infty$ and $\displaystyle\sum_{k=0}^{\infty}b^2_k<\infty$,
and schedule the stepsizes $\eta_k$ as in Proposition \ref{prop:bounded} or as in Proposition \ref{prop:bounded_adaptive}. 
    Then with probability one, the sequence $\left\{\W(k)\right\}_{k}$ generated by stochastic gradient descent 
	has the following convergence properties:
	\begin{enumerate}
		\item The sequence $\left(\Tilde{L}^N\big(\W(k)\big)\right)_{k \in \NN}$ of the loss  converges.
		\item The gradients $\nabla\Tilde{L}^N\big(\W(k)\big)$ converge to zero as $k \to \infty$.
		\item The Euclidean distance  $\operatorname{d}\big(\W(k), \operatorname{Crit}(\Tilde{L}^N) \big)$ converges to zero as $k \to \infty$.
	\end{enumerate}
\end{theorem}
The proof of this theorem is given in Section \ref{prf:sgd-conv2cset}. The crucial ingredient is to establish that the interpolated process $\overrightarrow{W}(t)$ is an APT for the continuous flow. To this end, we show condition $(1)$ in Proposition \ref{prop:apt_property} by leveraging Proposition~\ref{prop:crucial_apt_tool} for deterministic stepsizes and Proposition~\ref{prop:crucial_apt_tool_update} for adaptive stepsizes. Moreover, we then exploit Corollary~\ref{cor:apt_convergence} to establish the stated convergence result, also leveraging that $\Tilde{L}^N$ is a Lyapunov function. 

In contrast to previous works, Theorem~\ref{thm:sgd-conv2cset} directly addresses convergence of SGD in linear neural networks.
It avoids the global Lipschitz or bounded-gradient assumptions commonly imposed in SGD convergence theory.

\subsection{Discussion about the constants in our results}\label{sec:constants} 

The following discussion is heuristic and intended only to provide intuition regarding the dependence of the theoretical step-size bounds on the mini-batch size.
Precisely, we conclude this section with a heuristic analysis in which we approximate constant B and examine the influence of constants $P_T$ and $Q_T$ in the SGD algorithm. The goal is to intuitively compare SGD and GD on specific cases. 
Assume $\|X_S\|^2 \ge \|Y_SX_S^T\|$ and $\|\W(0)\|^2 \le M^2.$ 
It follows from \eqref{def:p_t}, \eqref{def:q_t} and \eqref{def:B} that 
\begin{align*}
	M  = \left( \frac{2 P_T \sqrt{N}}{\lambda c_\beta} \right)^{\frac{1}{2(\beta - N)}} = \left( \frac{2 \sqrt{N}}{\lambda c_\beta} \right)^{\frac{1}{2(\beta - N)}} \, P_T^{\frac{1}{2(\beta - N)}} = \Lambda(\lambda, \beta, N) \, P_T^{\frac{1}{2(\beta - N)}},
\end{align*}
where constant $\Lambda(\lambda,\beta,N)=\left( \frac{2 \sqrt{N}}{\lambda c_\beta} \right)^{\frac{1}{2(\beta - N)}}$  and
\begin{align*}
	B = \frac{\sqrt{N} \left( P_TM^{2N - 1} + Q_TM^{N - 1} \right)}{\lambda \beta \left( \frac{1}{2} M^2 \right)^{\beta - 1}}. 
\end{align*}
We approximate $B$ by only keeping the leading term in $M$ from the numerator as follows
\begin{align*}
	B \approx \frac{\sqrt{N} \, P_TM^{2N - 1}}{\lambda \beta \left( \frac{1}{2} M^2 \right)^{\beta - 1}} = \frac{2^{\beta - 1} \sqrt{N} \, P_T}{\lambda \beta \, M^{2(\beta - N) - 1}},  
\end{align*}
that is $B\approx \frac{2^{\beta - 1} \sqrt{N} \, P_T}{\lambda \beta\left( \frac{2 P_T \sqrt{N}}{\lambda c_\beta} \right)^{\frac{2(\beta - N) - 1}{2(\beta - N)}}}= C(\lambda, \beta, N) \, P_T^{\frac{1}{2(\beta - N)}}$
with constant \\ $C(\lambda, \beta, N)=\frac{2^{\beta - 1} \sqrt{N}}{\lambda \beta\left( \frac{2 \sqrt{N}}{\lambda c_\beta} \right)^{\frac{2(\beta - N) - 1}{2(\beta - N)}}}.$ Since $\|\W(0)\|^2 \le M^2$ then the deterministic $\eta_k$ reads
\begin{align}\label{eq:etak to compare}
	\eta_k &= \frac{b_k}{\beta \lambda \left[ \frac{1}{2}(M + B)^2 \right]^{\beta - 1}}=\frac{b_k}{\beta \lambda \left[ \frac{1}{2} \left( \Lambda(\lambda, \beta, N) + C(\lambda, \beta, N) \right)^2 \right]^{\beta - 1} \, P_T^{\frac{\beta -1}{\beta - N}}}.
\end{align}
We now compare SGD and GD in two special cases. In the first, the input data $X$ contains only zeros except for one data point $x_i$ at position $i$. In the second, $X$ consists of $m$ copies of one data point $x$.
\begin{enumerate}
	\item Let $x_i\in\RR^{d_x}$ such that $x_i\neq0,$ and consider a matrix $X$ that contains $x_i$ at column  $i$ and zeros everywhere else $(X = [0 ~\cdots~ 0~ x_i~ 0~\cdots~ 0]).$ Then using a single mini batch $i.e.,$ $T=1,$ we have  $P_T=\|x_i\|_2^2.$  
	Hence, by substituting $P_T$ into $\eqref{eq:etak to compare}$ and by assuming that the stepsizes $\eta_{SGD}$ of SGD are approximately the same, we get after $m$ steps 
	\begin{align*}
		m\cdot\eta_{SGD} &= \frac{mb_1}{\beta \lambda \left[ \frac{1}{2} \left( \Lambda(\lambda, \beta, N) + C(\lambda, \beta, N) \right)^2 \right]^{\beta - 1} \, (\|x_i\|_2^2)^{\frac{\beta -1}{\beta - N}}}.
	\end{align*}
	For the full batch gradient, i.e. $T=m,$ we have $P_m=\frac{1}{m}\|x_i\|_2^2,$ and
	\begin{align*}
		\eta_{GD} &= \frac{b_1}{\beta \lambda \left[ \frac{1}{2} \left( \Lambda(\lambda, \beta, N) + C(\lambda, \beta, N) \right)^2 \right]^{\beta - 1} \, \left(\frac{1}{m}\|x_i\|_2^2\right)^{\frac{\beta -1}{\beta - N}}}\\
		&=\frac{m^{\frac{\beta -1}{\beta - N}}b_1}{\beta \lambda \left[ \frac{1}{2} \left( \Lambda(\lambda, \beta, N) + C(\lambda, \beta, N) \right)^2 \right]^{\beta - 1} \, \left(\|x_i\|_2^2\right)^{\frac{\beta -1}{\beta - N}}}.
	\end{align*}
	
    Therefore, over the corresponding epoch, the GD stepsize $\eta_{GD}$ exceeds the effective SGD stepsize $m\cdot\eta_{SGD}$. Hence, GD may outperform SGD on such data. Intuitively, progress requires a pass over all data points, and GD completes this pass in one step, while SGD needs many steps for it.
	
	\item Let $x\in \RR^{d_x},$ such that $\|x\|_2\neq0,$ and form matrix $X$ with $m$ copies of $x$ as $X=[x~\cdots~ x].$
	For a single mini batch $T=1,$ we have $P_T=\|x\|_2^2.$ By assuming the stepsizes of SGD are approximately the same within this epoch,  after $m$ steps, substituting $P_T$ into $\eqref{eq:etak to compare}$ yields 
	\begin{align*}
		m\cdot\eta_{SGD} &= \frac{mb_1}{\beta \lambda \left[ \frac{1}{2} \left( \Lambda(\lambda, \beta, N) + C(\lambda, \beta, N) \right)^2 \right]^{\beta - 1} \, (\|x\|_2^2)^{\frac{\beta -1}{\beta - N}}}.
	\end{align*}
	In the case of the full batch $T=m,$ (GD), we have $P_T=\frac{1}{m}\left(\sqrt{m}\|x\|_2\right)^2=\|x\|_2^2$ and
	\begin{align*}
		\eta_{GD} &= \frac{b_1}{\beta \lambda \left[ \frac{1}{2} \left( \Lambda(\lambda, \beta, N) + C(\lambda, \beta, N) \right)^2 \right]^{\beta - 1} \, \left(\|x\|_2^2\right)^{\frac{\beta -1}{\beta - N}}}.
	\end{align*}
    Here, a single SGD step already covers the whole dataset, and the effective SGD stepsize $m\cdot\eta_{SGD}$ exceeds the GD stepsize $\eta_{GD}$. So on datasets of this kind, SGD may outperform GD.
    
\end{enumerate}
Overall, these observations suggest that SGD is likely to be more efficient than GD in scenarios where SGD can access the entire data in fewer steps. 

  
\section{Numerical experiments}\label{sec:experiments}

In this section we present numerical experiments illustrating our theoretical findings. We test both deterministic and adaptive step size choices and investigate how sharp the conditions on the maximal step sizes in Theorem~\ref{thm:sgd-conv2cset} are.

Our experiments use a randomly generated dataset, for
which we select $m=300$ data points $x_j$ of dimension $d=70$ from a Gaussian distribution of mean zero and variance $\sigma^2=\frac{1}{d}$ to form the input data $X\in\RR^{d\times m}$, where $d_x=d$. 
The dimensions of the weight matrices $W_j \in \RR^{d_j \times d_{j-1}}$ 
are chosen, for a parameter $r \in \NN$ (determining the (generic) rank of the corresponding product matrix), as
\begin{equation}\label{eq:def:dimensions}
d_0 = d, \qquad d_1 = r , \qquad d_j=\operatorname{round}\left(r+(j-1)\frac{d-r}{N-1}\right), \quad j=2,\dots,N, 
\end{equation}
where $\operatorname{round}(z)$ rounds the value of 
$z$ to the nearest integer. The weight matrices are initialized 
such that $\W(0)=(W_1,\hdots,W_N)$ has a balancedness constant $0$. 
This is achieved by randomly choosing $d_j\times d_j$ and $d\times d$ orthogonal matrices $V_j$, $j=1,\hdots,N$, and $U_1$ from the uniform distribution on the respective orthogonal group (by taking SVDs of a Gaussian matrix). 
The first $d_1$ columns of $V_{j-1}$ are then used to form the $d_j\times d_1$ matrices $U_j,$ for $j=2,\hdots,N$. For any $n_1,n_2\in \NN$, 
$I_{n_1,n_2}$ denotes the $n_1\times n_2$ rectangular diagonal matrix with ones on the diagonal. 
Then matrices $V_j,~U_j$ and $I_{n_1,n_2}$ are used to define $W_j$ as
\[
W_j = V_j I_{d_j,d_1} U_j^T.
\]
This guarantees that $\W(0) = (W_1,\hdots,W_N)$ is balanced, 
\begin{align*}
	W_{j+1}^T W_{j+1} &= U_{j+1} I_{d_1,d_{j+1}} V_{j+1}^T V_{j+1}                            I_{d_{j+1},d_{1}} U_{j+1}^T
	= U_{j+1} U_{j+1}^T 
	= V_j I_{d_{j},d_{1}} U_j^T U_j I_{d_{1},d_j} V_j^T \\
	&= W_j W_j^{T}.
\end{align*}
The target matrix $Y \in \RR^{d \times m}$ is also generated 
randomly by chosing
matrices $\widetilde{W_j}$ in the same way as $W_j$ above and then setting $Y = \widetilde{W_N} \cdots \widetilde{W_1}X$.

Various values for the regularization parameter $\lambda$ and the rank parameter $r$ are explored. 
The constant $c_\beta$ is set as $c_\beta = \beta 2^{1-\beta},$ and $\beta$ is chosen such that $\beta>N$. 
Recall that the constant $M$ in Proposition \ref{prop:bounded}
is given as
	\begin{align}\label{exp: M def}
		M = \max\left\{\left(\frac{2P_T\sqrt{N}}{\lambda c_\beta}\right)^{\frac{1}{2(\beta - N)}},  \left(\frac{2Q_T\sqrt{N}}{\lambda c_\beta}\right)^{\frac{1}{2\beta-N}}\right\}.
	\end{align}
The constants $P_T$ and $Q_T$ are computed by taking the maximum of $\|X_S\|^2$ and $\|Y_SX_S^T\|$ over the actually employed mini-batches of size $T=30$ used during the run of SGD. (Inspecting the proof of Proposition~\ref{prop:bounded} reveals that only these mini-batches are needed.)

We conduct our tests on linear neural networks with depths 
$3$ and $4$ of the architecture described above. 
Our first experiments employ deterministic stepsizes that satisfy the requirements of Theorem \ref{thm:sgd-conv2cset}, 
i.e., they are scheduled as  
	\begin{align}\label{exp: deterministic etak}
		\eta_k = \frac{b_k}{\beta\lambda \max\left\{\left(\frac{1}{2}\|\W(0)\|_F^2\right), \frac{1}{2}(M+B)^2 \right\}^{\beta-1}},
	\end{align}
where $B$ is defined as in Theorem~\ref{thm:sgd-conv2cset} and $\displaystyle b_k=\frac{1}{(k+1)^{\gamma}}$ represents a sequence converging to zero at decay rate $\gamma.$ 

Figure~\ref{fig:regularized losses and Frobenius norm of individual weights with deterministic stepsize} presents the regularized loss and Frobenius norms of the weight matrices during the training of the linear networks of widths $N=3,4$, each optimized via SGD with deterministic step sizes as scheduled in $\eqref{exp: deterministic etak}$. 
Subfigures ~\ref{fig:Regularized_loss_deterministicetak_r5_beta5_batchsize30_lambdapara0_000_epochs10000_N3} and \ref{fig:Regularized_loss_deterministicetak_r4_beta7_batchsize30_lambdapara0_000_epochs10000_N4} illustrate the evolution of the regularized loss for networks of depth $N=3,4$ and with corresponding parameter values $r=5,4$ in \eqref{eq:def:dimensions}. 
\begin{figure}[h!] 
				\begin{subfigure}{0.45\textwidth}
					\centering\includegraphics[width=0.99\textwidth]{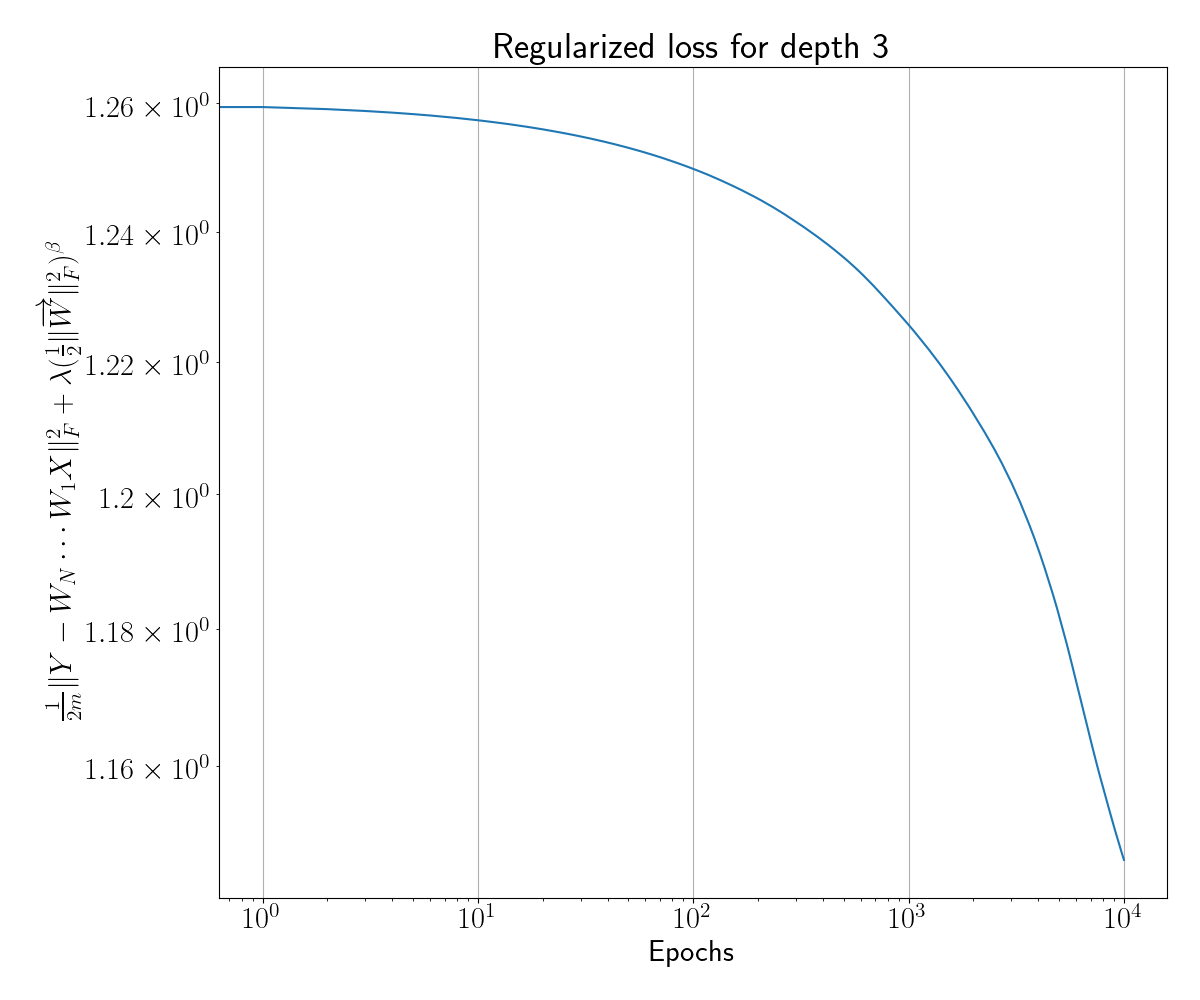}
					\caption{Evolution of the 
	                   regularized 
	                   loss for a network of depth 3}
				\label{fig:Regularized_loss_deterministicetak_r5_beta5_batchsize30_lambdapara0_000_epochs10000_N3}
			\end{subfigure}\hfill\begin{subfigure}{0.45\textwidth}\centering\includegraphics[width=0.99\textwidth]{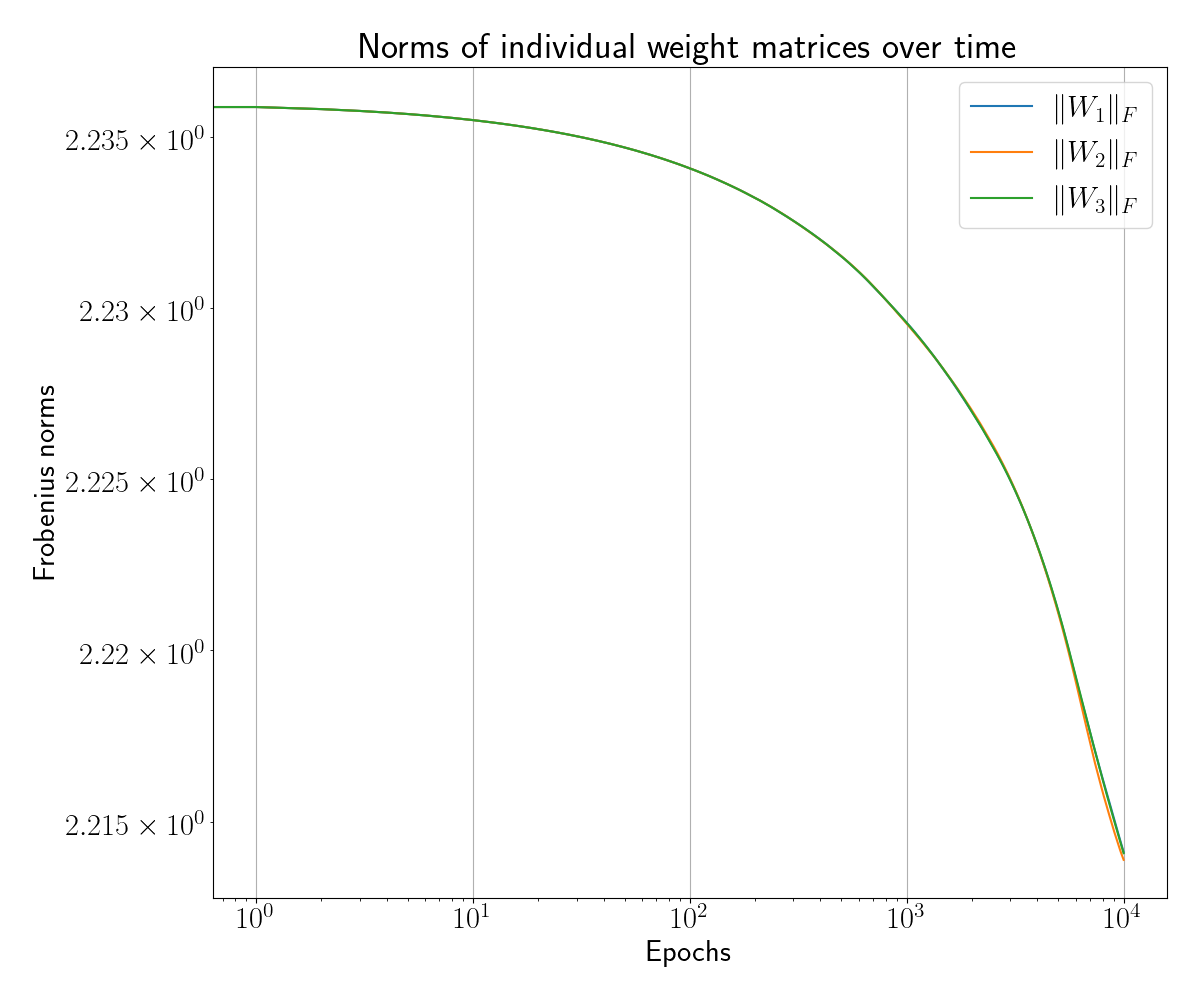}
				\caption{
					Frobenius norm of the individual weight matrices of a network of depth 3}
			\label{fig:Norms_of_Wj_deterministicetak_r5_beta5_batchsize30_lambdapara0_000_epochs10000_N3}
		\end{subfigure}
           				\begin{subfigure}{0.45\textwidth}
					\centering\includegraphics[width=0.99\textwidth]{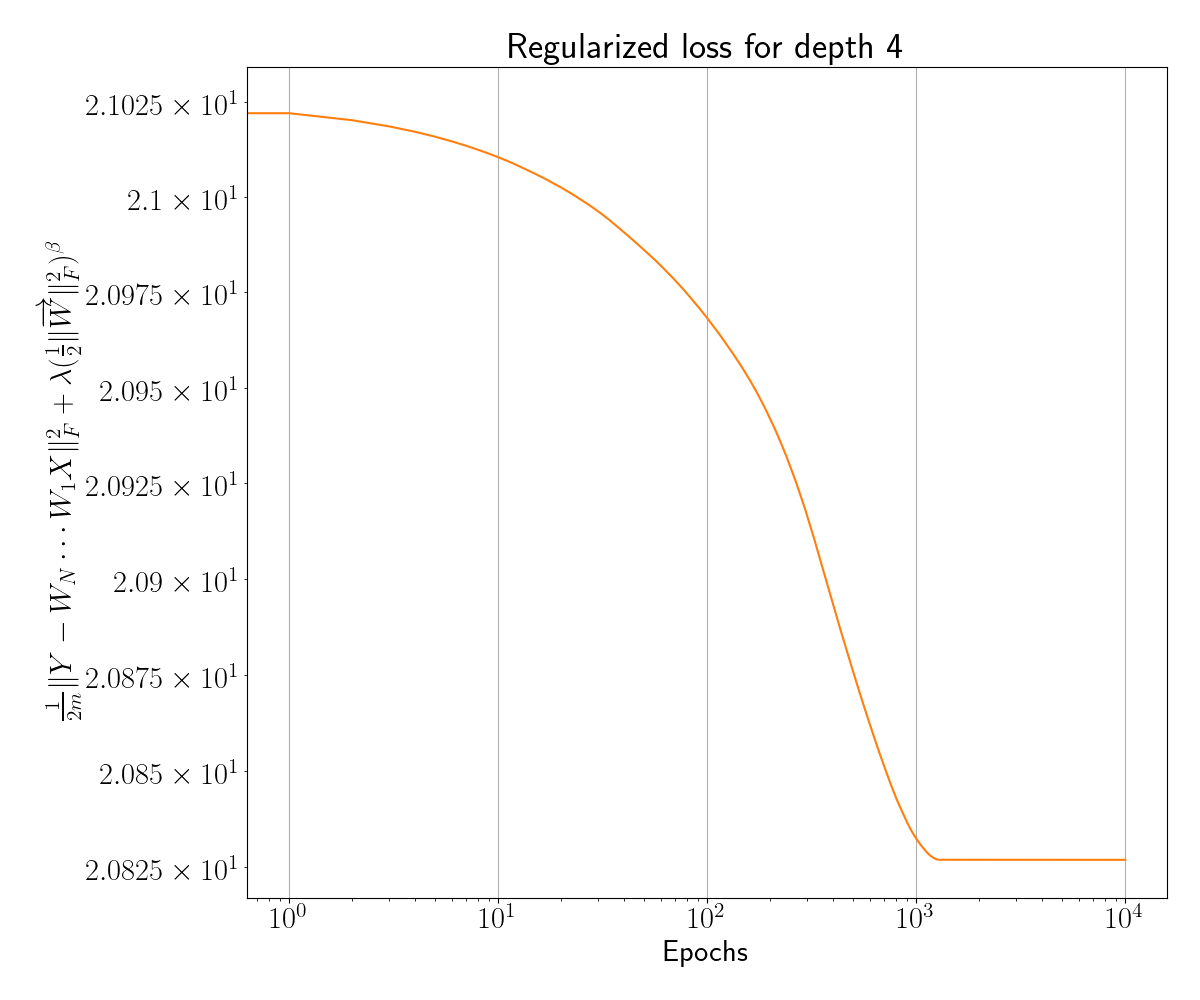}
					\caption{Evolution of the 
                   regularized 
                   loss for a network of depth 4}
				\label{fig:Regularized_loss_deterministicetak_r4_beta7_batchsize30_lambdapara0_000_epochs10000_N4}
			\end{subfigure}\hfill\begin{subfigure}{0.45\textwidth}\centering\includegraphics[width=0.99\textwidth]{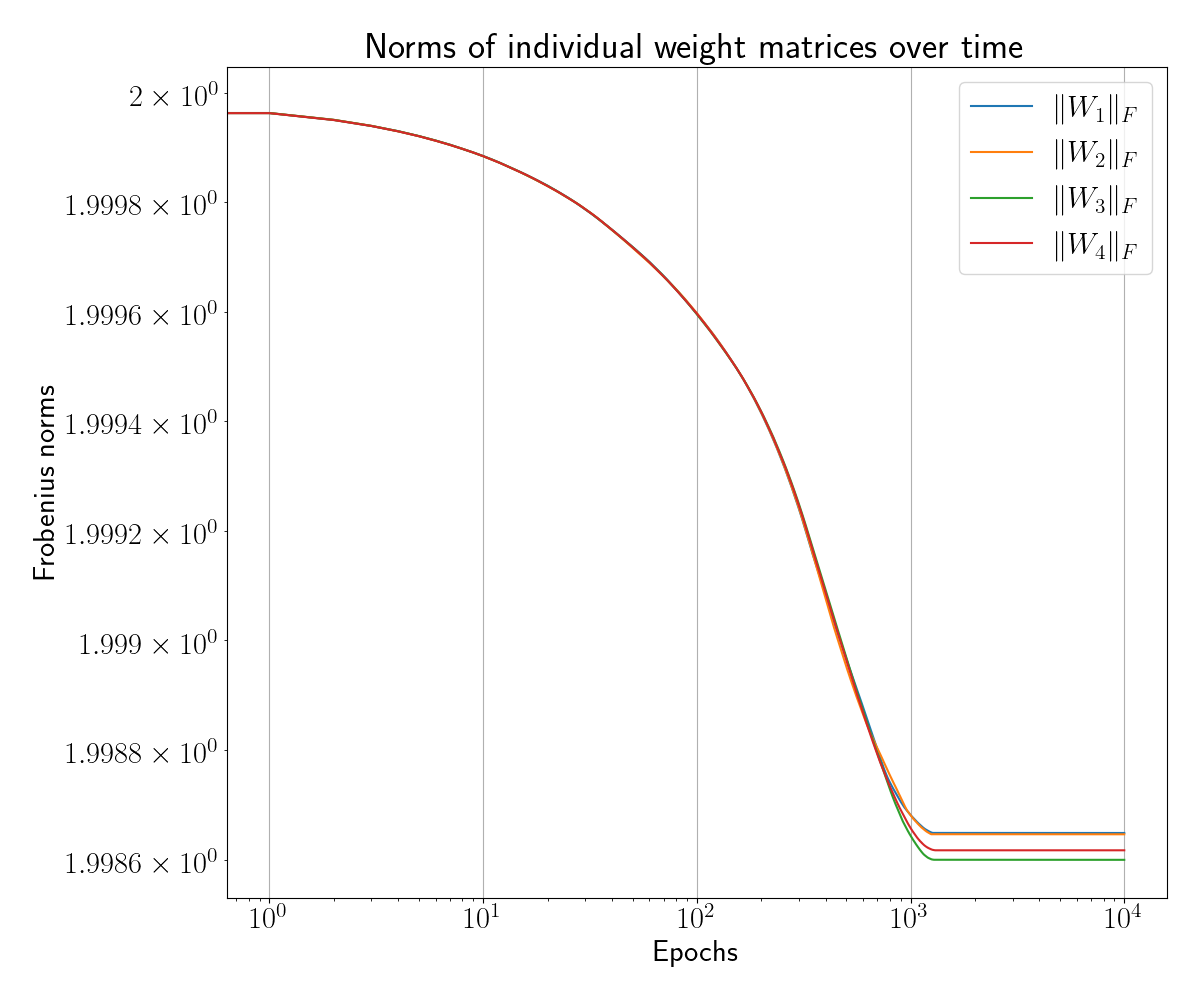}
				\caption{
					Frobenius norm of the individual weight matrices of a network of depth 4}
			\label{fig:Norms_of_Wj_deterministicetak_r4_beta7_batchsize30_lambdapara0_000_epochs10000_N4}
		\end{subfigure}
		\caption{The left panel illustrates the evolution 
       of the regularized loss functions across networks with deterministic step sizes as specified in $\eqref{exp: deterministic etak}$. The right panel presents the Frobenius norms of individual weight matrices for networks with depths 3 and 4 during training. 
       }\label{fig:regularized losses and Frobenius norm of individual weights with deterministic stepsize}
\end{figure}
These tests validate the convergence of the regularized loss function $\tilde{L}^N$ applied to the SGD iterates under deterministic step sizes. The Frobenius norms of the weight matrices for these networks are reported in Subfigures~\ref{fig:Norms_of_Wj_deterministicetak_r5_beta5_batchsize30_lambdapara0_000_epochs10000_N3} and \ref{fig:Norms_of_Wj_deterministicetak_r4_beta7_batchsize30_lambdapara0_000_epochs10000_N4}.
These figures demonstrate that the norms of the SGD iterates remain bounded throughout training. The 
regularization parameters 
used in the experiments for Figure~\ref{fig:regularized losses and Frobenius norm of individual weights with deterministic stepsize} are $\lambda=10^{-5}$ for $N=3$ and $\lambda=5\cdot 10^{-5}$ for $N=4$. 
The parameter $\beta$ is set as $\beta=2N-1,$ for both $N=3$ and $N=4$. The resulting step size constants 
are $B=5.98,$ $M=11.97$ for $N=3$ and $B=4.18,$ $M=8.37$ for $N=4.$

In the next experiments, we explore SGD using deterministic step size selections that are not directly derived from the conditions outlined in our theoretical findings. 
Figure~\ref{fig:Regularized losses and Frobenius norm of individual weights with deterministic stepsize without our conditions} presents the evolution of the regularized loss across networks in the left panel.
The right panel 
displays the Frobenius norm of the weight matrices for each network during training. 
The step sizes are given by $\eta_k=\frac{c}{(1+k)^{\gamma}}$ with $c=0.9,$ $\gamma=0.4$    
and $\beta=2$ for all three networks, 
violating the condition $\sum_{k=0}^{\infty}\eta_k^2<\infty$.
\begin{figure}[h!] 
				\begin{subfigure}{0.45\textwidth}
					\centering\includegraphics[width=0.99\textwidth]{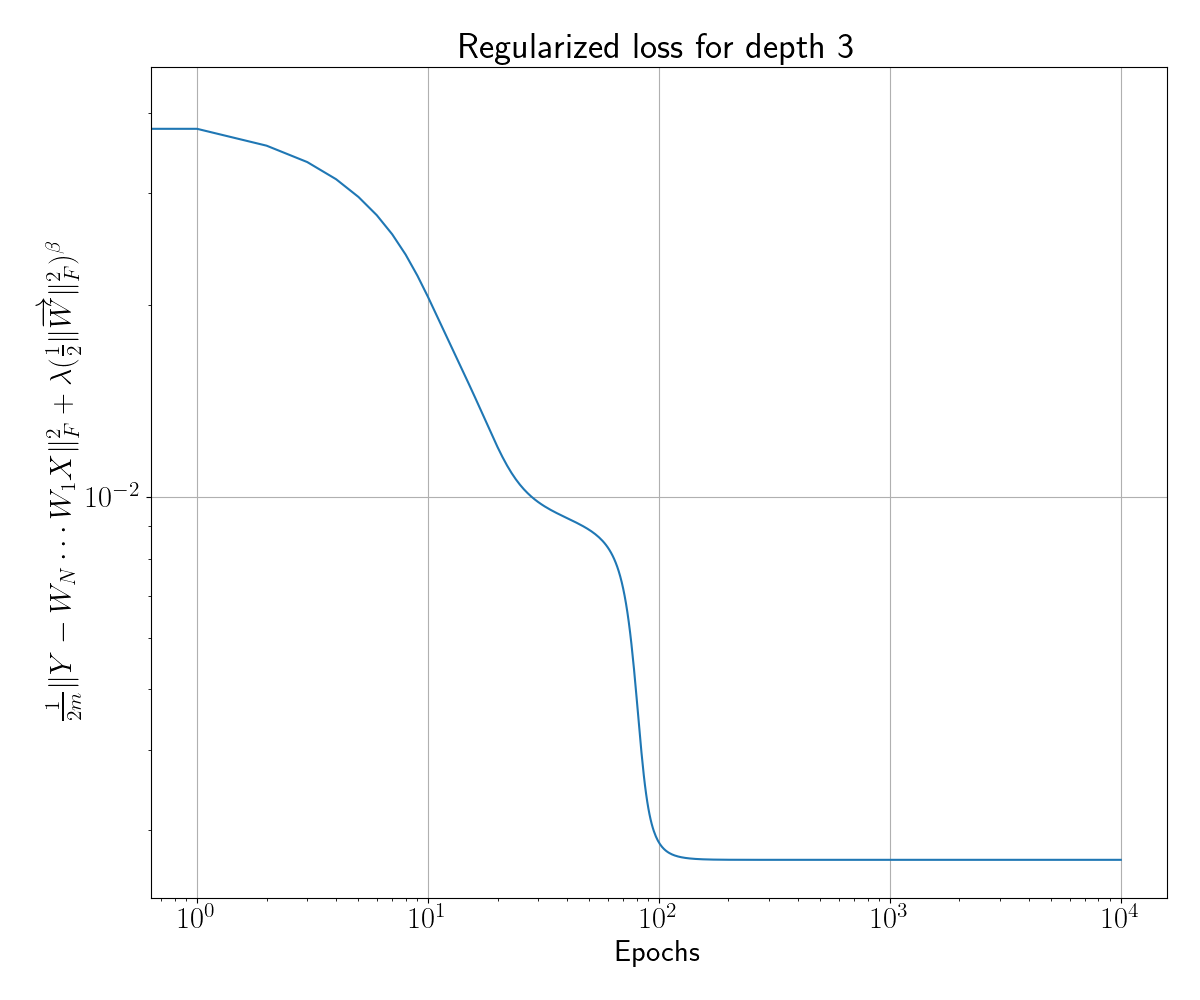}
					\caption{Evolution of the regularized loss for a neural network of depth 3}
				\label{fig:Regularized_loss_deterministicetak_without_our_conditions_r5_beta0.5_batchsize30_lambdapara0_00005_epochs10000_N3}
			\end{subfigure}\hfill\begin{subfigure}{0.45\textwidth}\centering\includegraphics[width=0.99\textwidth]{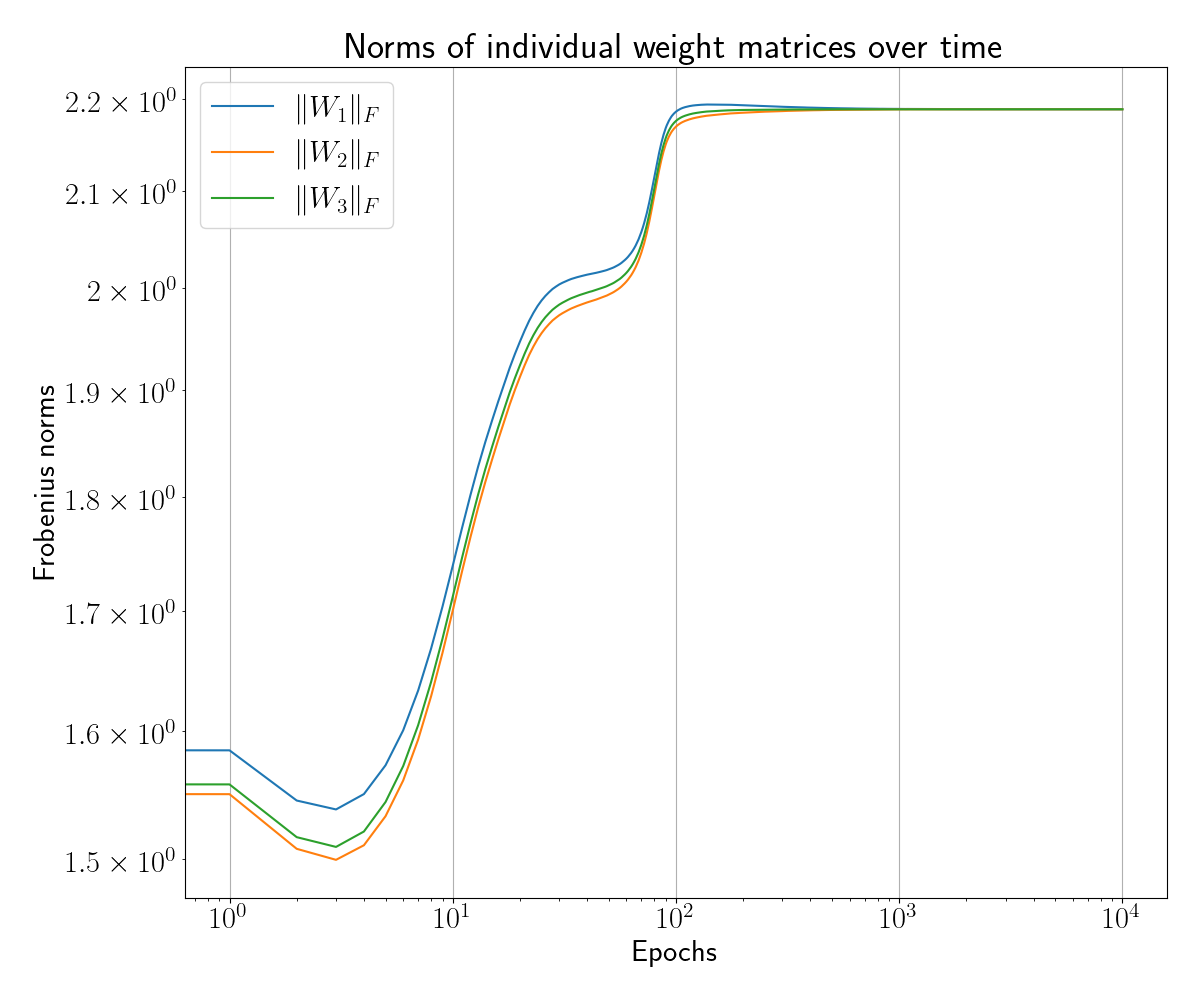}
				\caption{
					Frobenius norm of the individual weight matrices of a network of depth 3}
			\label{fig:Norms_of_Wj_deterministicetak_without_our_conditions_r5_beta1_batchsize30_lambdapara0_00005_epochs10000_N3}
		\end{subfigure}
           				\begin{subfigure}{0.45\textwidth}
					\centering\includegraphics[width=0.99\textwidth]{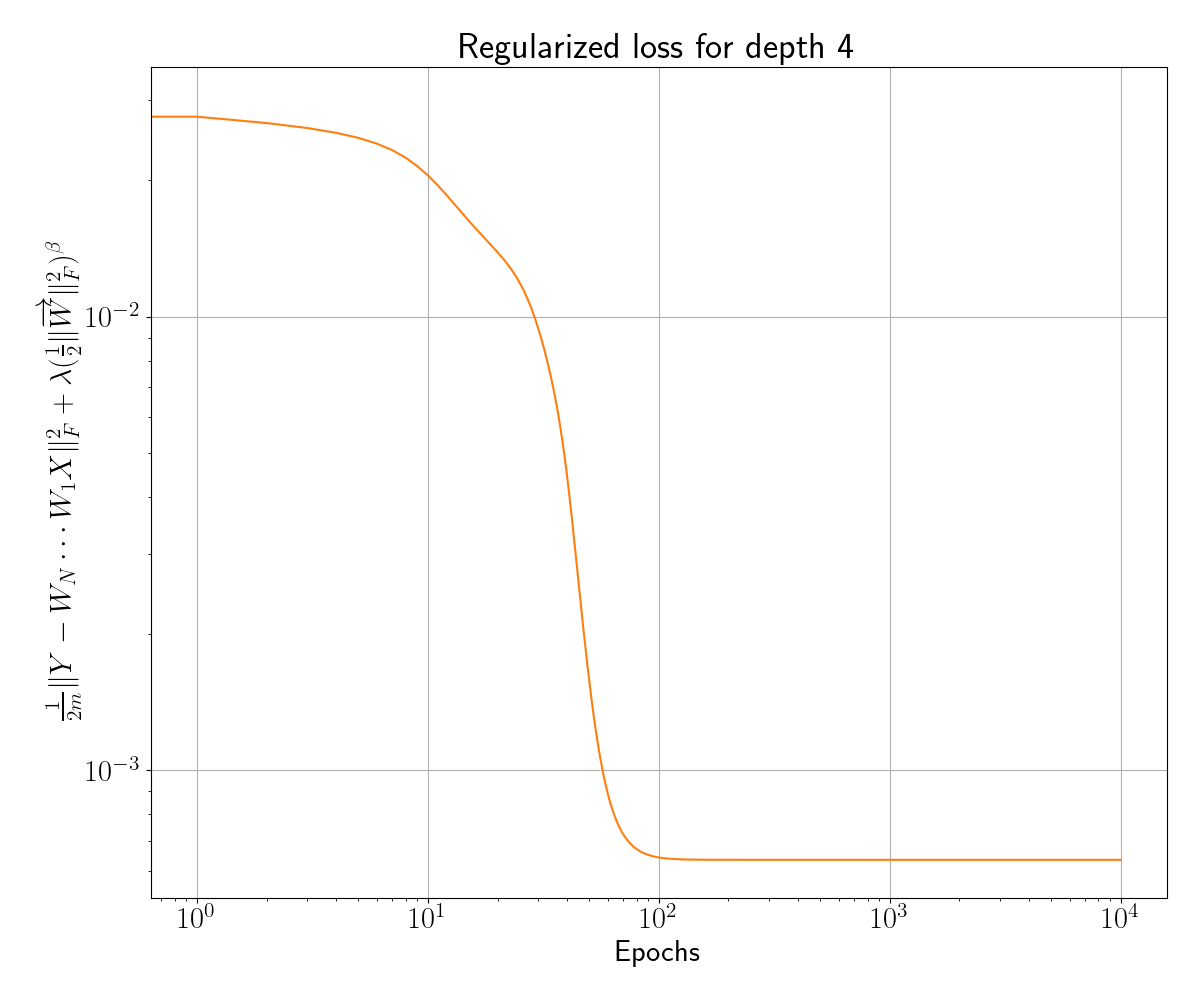}
					\caption{Evolution of the regularized loss for a network of depth 4}
				\label{fig:Regularized_loss_deterministicetak_without_our_conditions_r4_beta1_batchsize30_lambdapara0_00001_epochs10000_N4}
			\end{subfigure}\hfill\begin{subfigure}{0.45\textwidth}\centering\includegraphics[width=0.99\textwidth]{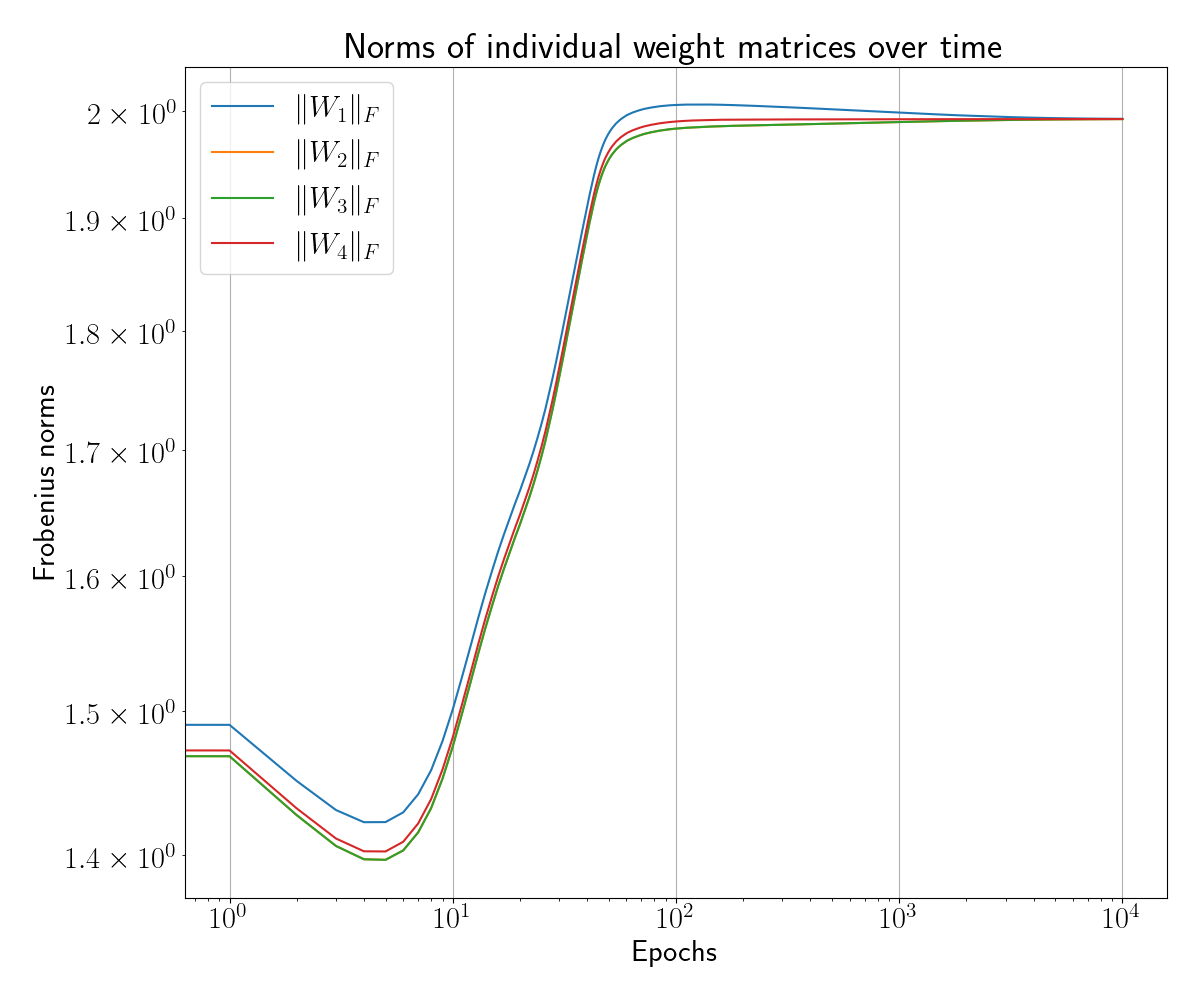}
				\caption{
					Frobenius norm of the individual weight matrices of a network of depth 4}
			\label{fig:Norms_of_Wj_deterministicetak_without_our_conditions_r4_beta1_batchsize30_lambdapara0_00001_epochs10000_N4}
		\end{subfigure}
		\caption{The left panel displays the evolution of 
       the regularized losses across networks with deterministic stepsizes that do not meet requirements in Theorem~\ref{thm:sgd-conv2cset}. The right panel shows the Frobenius norms of individual matrices for networks of depths 2, 3, and 4 during training. 
       }\label{fig:Regularized losses and Frobenius norm of individual weights with deterministic stepsize without our conditions}
\end{figure}
Subfigures \ref{fig:Norms_of_Wj_deterministicetak_without_our_conditions_r5_beta1_batchsize30_lambdapara0_00005_epochs10000_N3} and \ref{fig:Norms_of_Wj_deterministicetak_without_our_conditions_r4_beta1_batchsize30_lambdapara0_00001_epochs10000_N4} indicate that the SGD iterates still remain bounded, suggesting that the established bound for SGD is not tight.


The second part of the experiments 
tests SGD with adaptive stepsizes. The regularization parameters 
remain consistent with those previously described, and the stepsize $\eta_k$ is scheduled as
   \begin{align}\label{exp: adaptive etak}
       \eta_k = b_k \cdot \min\left\{  \frac{1}{\beta\lambda(\frac12\|\W(k)\|_F^2)^{\beta-1}},\, D\, \right\}, 
   \end{align}
for $D$ 
as in Theorem~\ref{thm:sgd-conv2cset} and $b_k=\frac{1}{(1+k)^{\gamma}}$, $\gamma>\frac{1}{2}$. 

Plots~\ref{fig:Regularized_loss_adaptivestepsize_r5_beta5_batchsize30_lambdapara0_00005_epochs500000_N3} and \ref{fig:Regularized_loss_adaptivestepsize_r4_beta7_batchsize30_lambdapara0_00001_epochs500000_N4} in the left panel of Figure \ref{fig:regularized losses and Frobenius norm of individual weights with adaptive stepsize} display the regularized loss 
versus iterations for networks of depth $3$ and $4$ respectively. These 
validate the convergence of $\left(\tilde{L}^N\big(\W(k)\big)\right)_k$ as established in Theorem~\ref{thm:sgd-conv2cset}. The right panel of Figure~\ref{fig:regularized losses and Frobenius norm of individual weights with adaptive stepsize} shows the evolution of the Frobenius norm of individual weight matrices during training for networks of depth $3$ and $4$ on plots \ref{fig:Norms_of_Wj_adaptivestepsize_r4_beta7_batchsize30_lambdapara0_00005_epochs500000_N3} and \ref{fig:Norms_of_Wj_adaptiveetak_r4_beta7_batchsize30_lambdapara0_00001_epochs500000_N4}, respectively.
\begin{figure}[h!] 
				\begin{subfigure}{0.45\textwidth}
					\centering\includegraphics[width=0.99\textwidth]{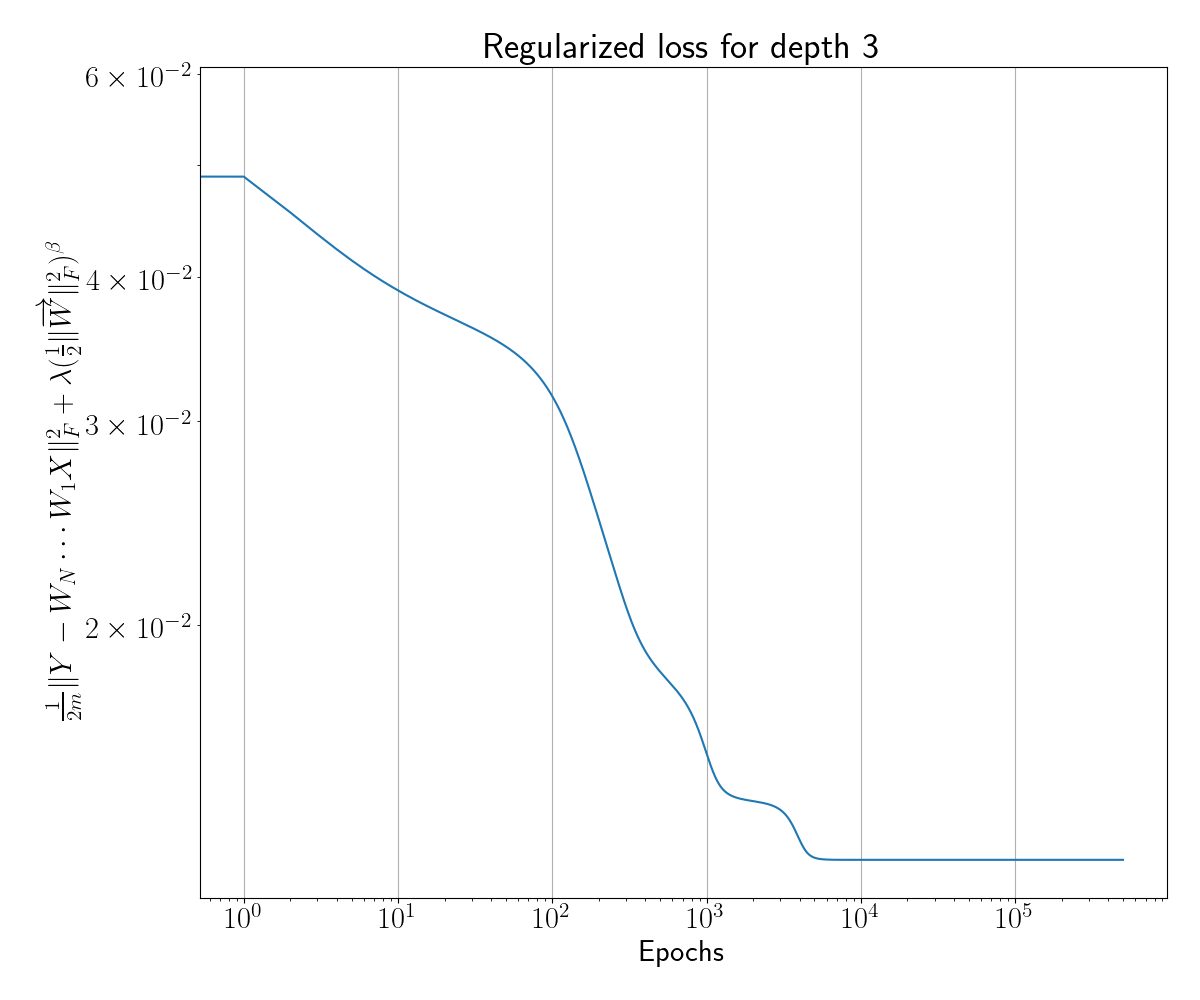}
					\caption{Evolution of the regularized loss for a network of depth 3}
				\label{fig:Regularized_loss_adaptivestepsize_r5_beta5_batchsize30_lambdapara0_00005_epochs500000_N3}
			\end{subfigure}\hfill\begin{subfigure}{0.45\textwidth}\centering\includegraphics[width=0.99\textwidth]{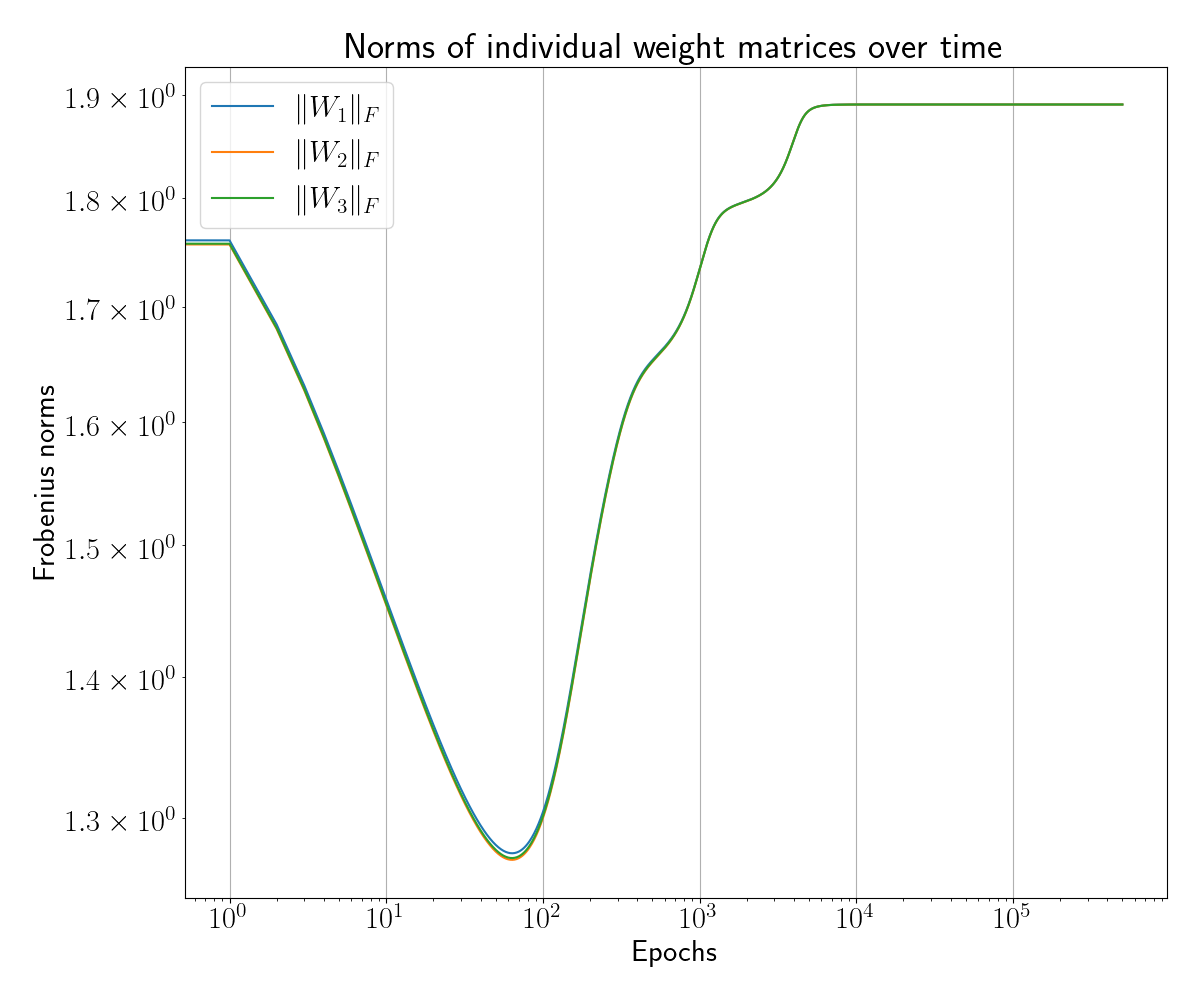}
				\caption{
					Frobenius norm of the individual weight matrices of a network of depth 3}
			\label{fig:Norms_of_Wj_adaptivestepsize_r4_beta7_batchsize30_lambdapara0_00005_epochs500000_N3}
		\end{subfigure}
           				\begin{subfigure}{0.45\textwidth}
					\centering\includegraphics[width=0.99\textwidth]{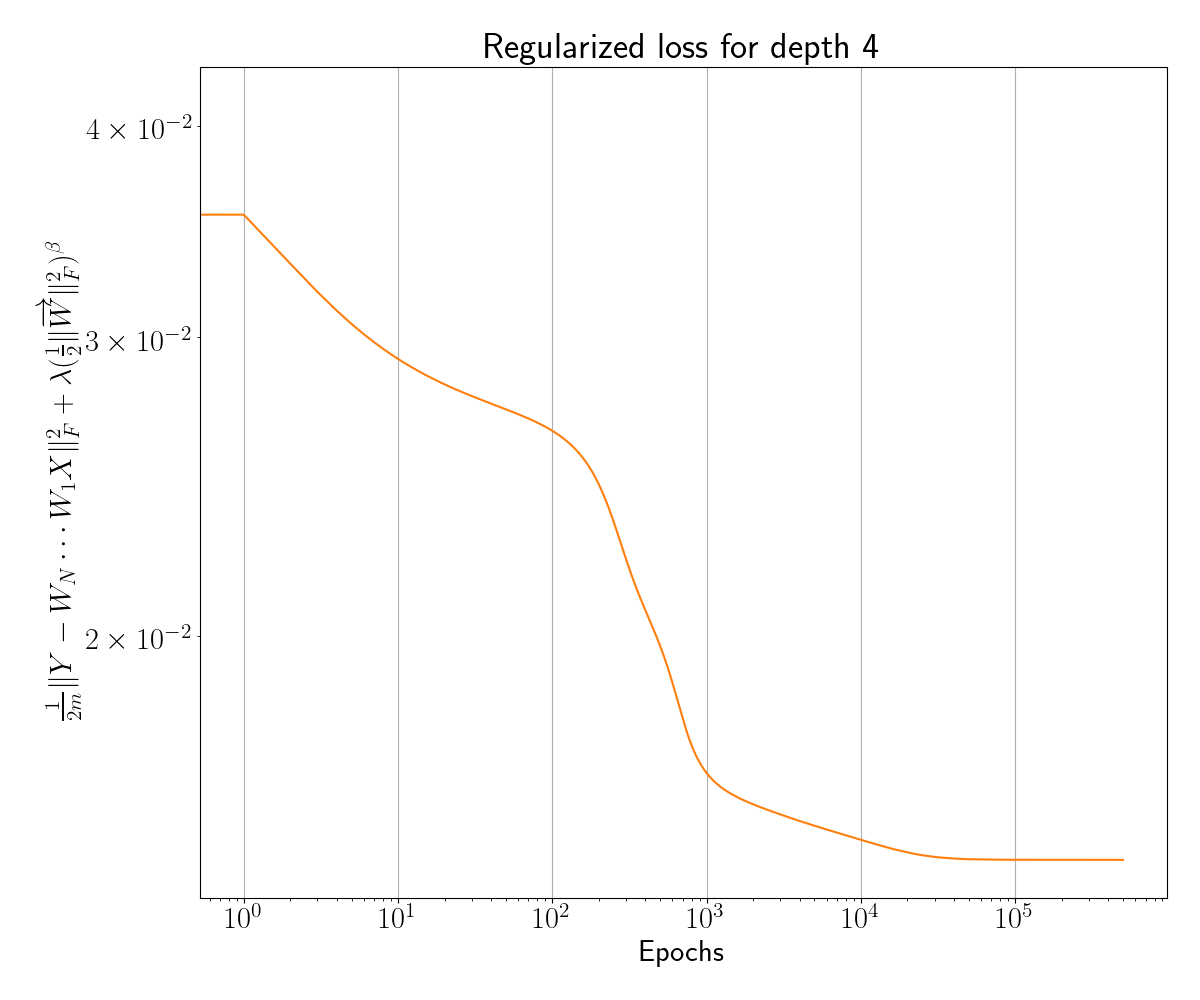}
					\caption{Evolution of the regularized loss for a network of depth 4}
				\label{fig:Regularized_loss_adaptivestepsize_r4_beta7_batchsize30_lambdapara0_00001_epochs500000_N4}
			\end{subfigure}\hfill\begin{subfigure}{0.45\textwidth}\centering\includegraphics[width=0.99\textwidth]{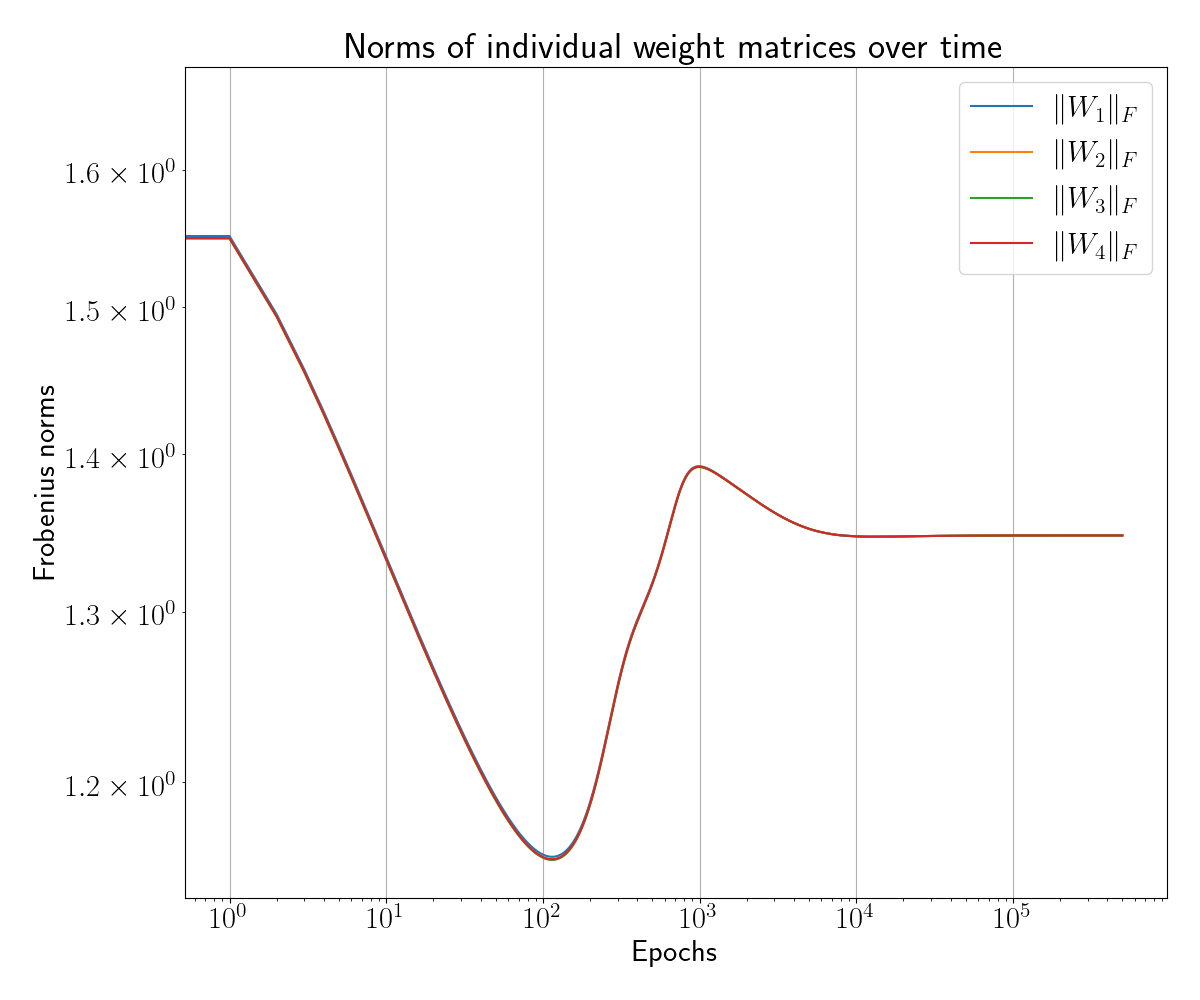}
				\caption{
					Frobenius norm of the individual weight matrices of a network of depth 4}
			\label{fig:Norms_of_Wj_adaptiveetak_r4_beta7_batchsize30_lambdapara0_00001_epochs500000_N4}
		\end{subfigure}
		\caption{The left panel illustrates the evolution of 
       the regularized loss across networks with adaptive stepsizes $\eqref{exp: adaptive etak}$. The right panel presents the Frobenius norms of individual weight matrices for networks of depth 3 and 4. 
      }\label{fig:regularized losses and Frobenius norm of individual weights with adaptive stepsize}
\end{figure}
In this experiment, the 
constant $D$ is set to $1$ for $N=3$ and $N=4.$ 
The regularization parameters are $\lambda=10^{-5}$ for $N=3,$ and $\lambda=5\cdot10^{-5}$ for $N=4.$ 
The parameter $r$ in \eqref{eq:def:dimensions} is set to $5$ for $N=3,$ and $4$ for $N=4.$ 


Figure~\ref{fig:regularized losses and Frobenius norm of individual weights without theo conditions adaptive stepsize} displays the training dynamics using a modified adaptive stepsize, differing from $\eqref{exp: adaptive etak}$ 
and defined as
   \begin{align}\label{exp: modified adaptive etak}
       \eta_k = \frac{\alpha}{\frac{1}{2}\|\W(k)\|_F^2(1+k)^{\gamma}},
   \end{align}
 where $\alpha>0.$ 
Subfigures~\ref{fig:Regularized_loss without theo conditions adaptive etak N3} and \ref{fig:Regularized_loss without theo conditions adaptive etak N4} presents the regularized training loss, while Subfigures~\ref{fig:Norms_of_Wj 
without theo conditions adaptive etak N3} and \ref{fig:Norms_of_Wj without theo conditions adaptive etak N4} display the training norms of the weight matrices. The corresponding parameter values are $\lambda=10^{-5}$ and $r=5$ for $N=3,$ $\lambda=5\cdot10^{-5}$ and $r=4$ for $N=4,$ with $\beta=2$ for all networks, $\alpha=0.9$ and $\gamma=0.4.$ 
\begin{figure}[h!] 
				\begin{subfigure}{0.45\textwidth}
					\centering\includegraphics[width=0.99\textwidth]{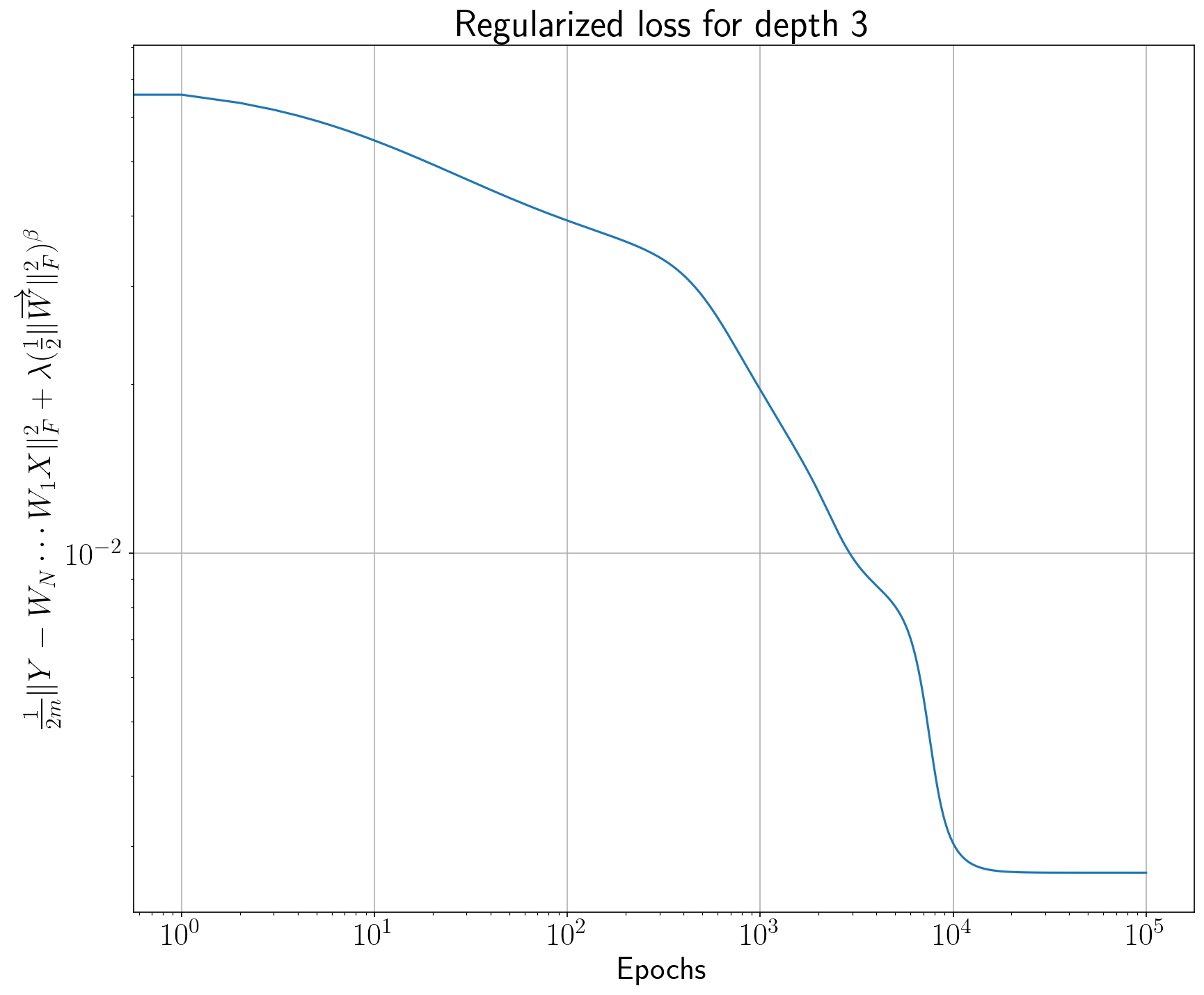}
					\caption{Evolution of the regularized losses in the network with depth 3}
				\label{fig:Regularized_loss without theo conditions adaptive etak N3}
			\end{subfigure}\hfill\begin{subfigure}{0.45\textwidth}\centering\includegraphics[width=0.99\textwidth]{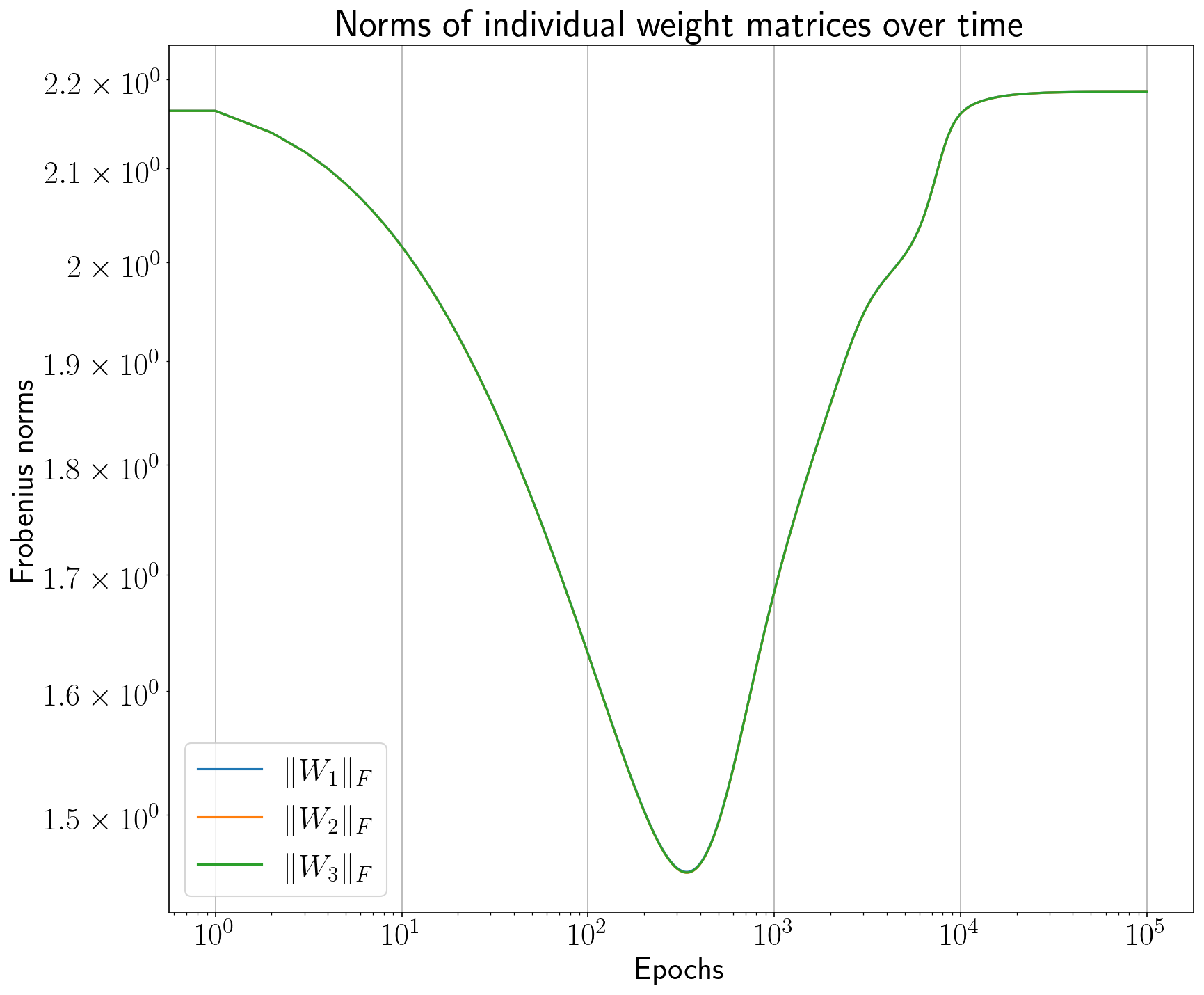}
				\caption{
					Frobenius norm of the individual weight matrices of a network of depth 3}
			\label{fig:Norms_of_Wj 
without theo conditions adaptive etak N3}
		\end{subfigure}
           				\begin{subfigure}{0.45\textwidth}
					\centering\includegraphics[width=0.99\textwidth]{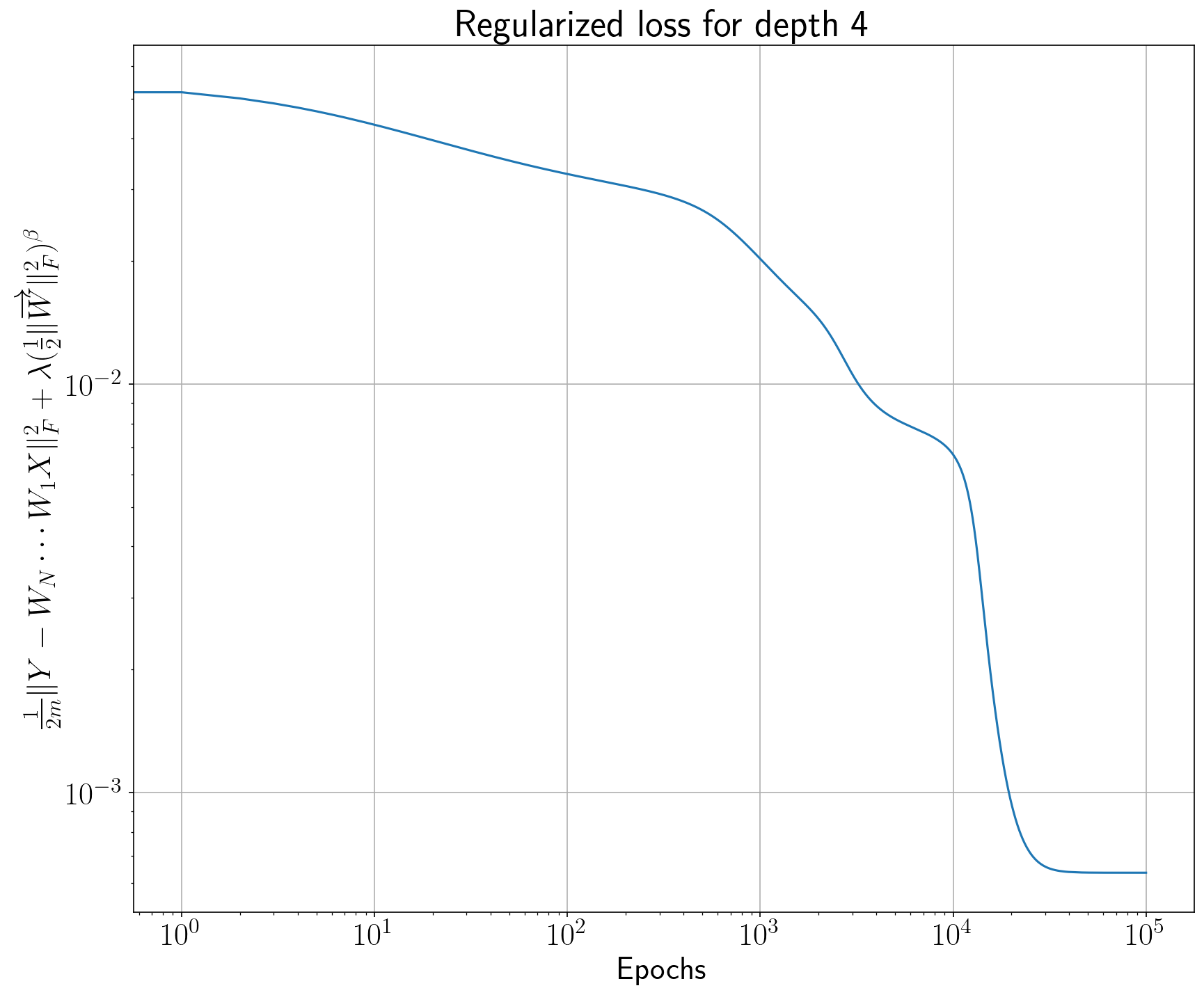}
					\caption{Evolution of the regularized loss for a network of depth 4}
				\label{fig:Regularized_loss without theo conditions adaptive etak N4}
			\end{subfigure}\hfill\begin{subfigure}{0.45\textwidth}\centering\includegraphics[width=0.99\textwidth]{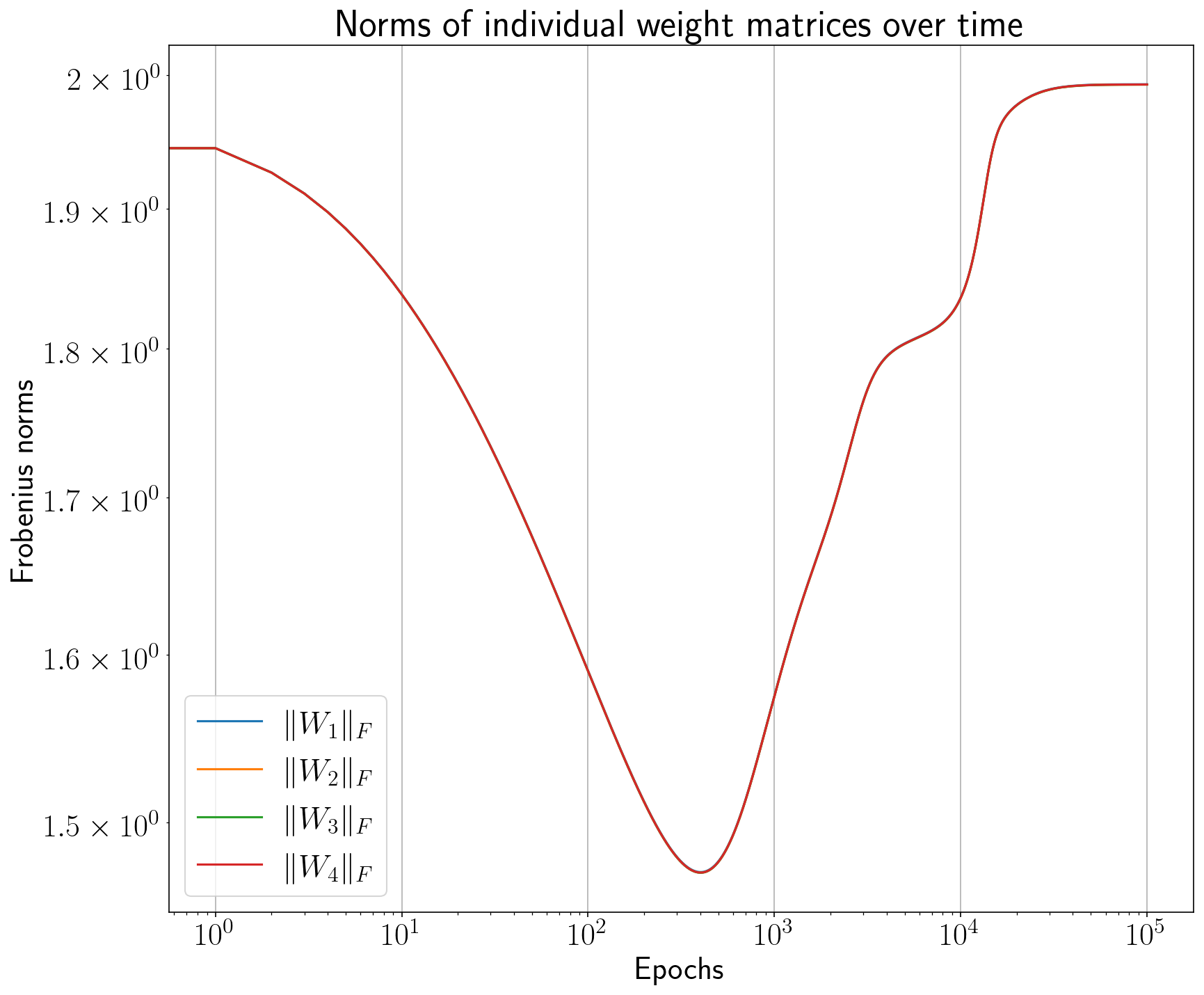}
				\caption{
					Frobenius norm of the individual weight matrices of a network of depth 4}
			\label{fig:Norms_of_Wj without theo conditions adaptive etak N4}
		\end{subfigure}
		\caption{The left panel shows the evolution of the regularized loss across networks with the modified adaptive step sizes $\eqref{exp: modified adaptive etak}$. The right panel displays the Frobenius norms of individual weight matrices for networks of depth 2, 3, and 4 during training with the same adaptive step sizes. 
       }\label{fig:regularized losses and Frobenius norm of individual weights without theo conditions adaptive stepsize}
\end{figure}

Figure~\ref{fig:Comparison of deterministic and adaptive step sizes} illustrates the convergence speed for both deterministic and adaptive step sizes. Subfigure~\ref{fig: relative loss N3} presents the relative loss, which is computed by normalizing the regularized loss iterates $\left(\tilde{L}^N\big(\W(k)\big)\right)_k$ with their initial value, whereas Subfigure~\ref{fig:side-by-side-comparison N3} displays the regularized loss for both deterministic (left panel) and adaptive step sizes (right panel). These plots indicate that the adaptive step sizes described in $\eqref{exp: deterministic etak}$ achieve faster convergence compared to the deterministic step sizes outlined in $\eqref{exp: deterministic etak}$. 

\begin{figure}[htbp]
    \centering
    \begin{subfigure}{\linewidth}
        \centering
        \includegraphics[width=0.45\linewidth]{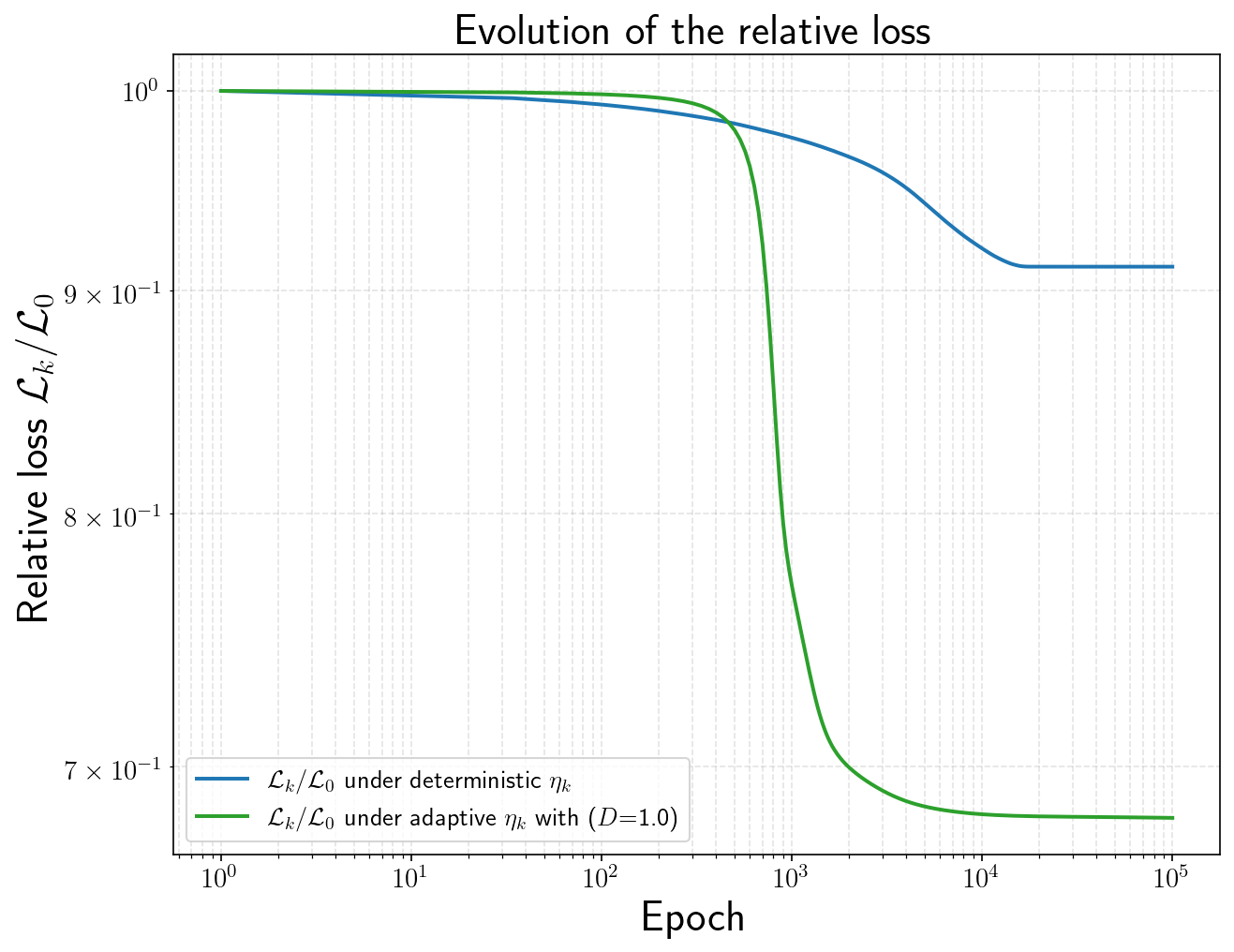}
        \caption{Evolution of the relative loss obtained by normalizing $\Tilde{L}^N\big(\W(k)\big)$ by its initial value ($\mathcal{L}_k/\mathcal{L}_0$) under deterministic and adaptive ($D=1$) stepsizes}
        \label{fig: relative loss N3}
    \end{subfigure}
    \begin{subfigure}{\linewidth}
        \centering
        \includegraphics[width=\linewidth]{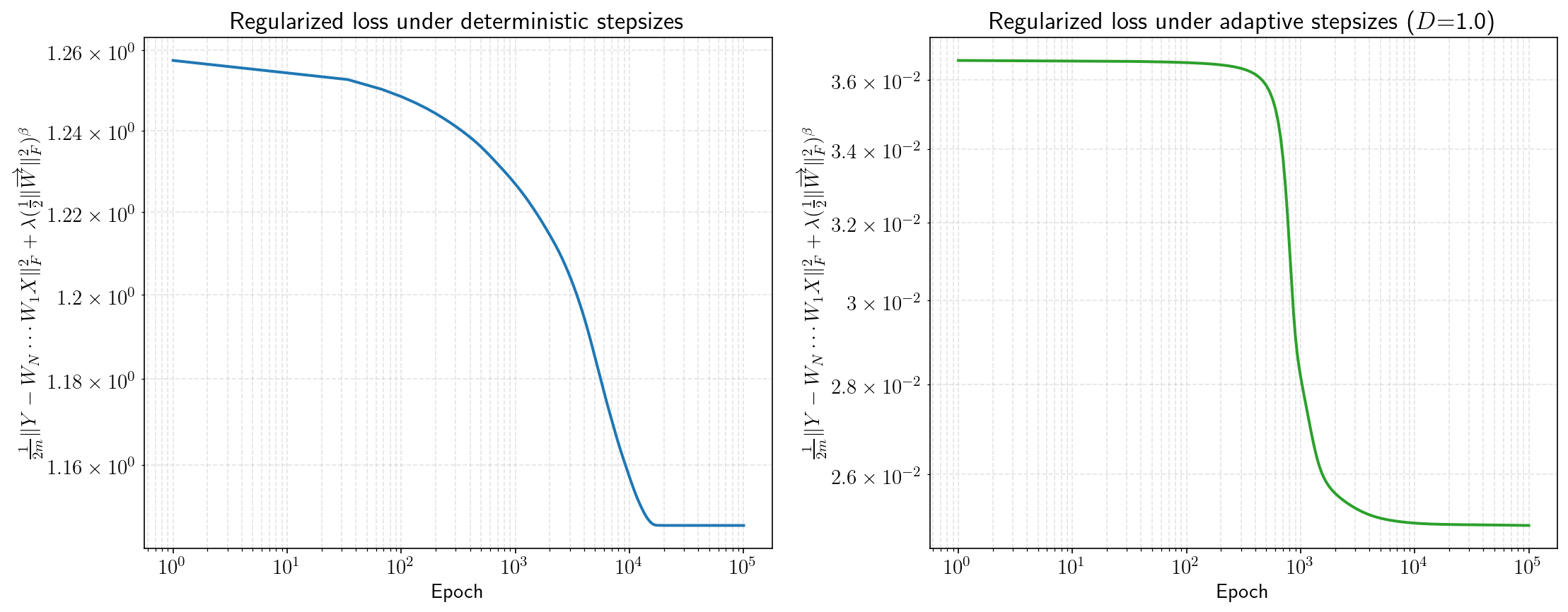}
        \caption{Evolution of the regularized loss under deterministic stepsizes (left panel) and adaptive stepsizes with $D=1$ (right panel)}
        \label{fig:side-by-side-comparison N3}
    \end{subfigure}

    \caption{Convergence comparison for deterministic versus adaptive stepsizes on a linear network of depth $3$ with parameters $\beta=5$, $\lambda=5\cdot10^{-5}$, $r=5$, and mini-batch size $T=30$.}
    \label{fig:Comparison of deterministic and adaptive step sizes}
\end{figure}

   \medskip
   
Finally, we assess convergence of SGD also for nonlinear neural networks, although these are not covered by our present theory.
More precisely, we check whether the conditions established in Theorem~\ref{thm:sgd-conv2cset} for training linear neural networks with SGD could transfer to nonlinear networks. For this purpose, we use a hyperbolic tangent activation function on the last layer, and a linear activation function on the remaining layers. This approach results in the regularized loss function
\begin{align}\label{exp:nonlinear regularized loss}
    \mathcal{L}_{\sigma}^N\big(\W\big) =\frac{1}{2m}\|Y - \sigma\big(W_N\cdots W_1X\big)\|_F^2 + \lambda\left(\frac{1}{2}\|\W\|_F^2\right)^{\beta},
\end{align}
where $\sigma(x)=\tanh(x)$ is applied componentwise. 


Figure~\ref{fig:nonlinear NN  regularized losses and Frobenius norm of individual weights with deterministic stepsize} presents the learning behavior of SGD in two settings of nonlinear neural networks. The values of $\mathcal{L}_{\sigma}^N\big(\W(k)\big)$ versus iterations are 
shown in Subfigures~\ref{fig:nonlinear NN regularized_loss deterministicetak N3} and \ref{fig:nonlinear NN regularized_loss deterministicetak N4} for networks of depth $3$ and $4$ respectively. Plots~\ref{fig:nonlinear NN Norms_of_Wj deterministicetak N3} and \ref{fig:nonlinear NN Norms_of_Wj deterministicetak N4} display the norms of the training weights under the deterministic stepsize rule in \eqref{exp: deterministic etak}.  Figure~\ref{fig:nonlinear NN  regularized losses and Frobenius norm of individual weights with deterministic stepsize} empirically demonstrate the boundedness of SGD in nonlinear networks and the convergence of $\left(\mathcal{L}_{\sigma}^N\big(\W(k)\big)\right)_k$ 
under the deterministic stepsizes \eqref{exp: deterministic etak}. These suggests that certain 
theoretical properties established for SGD with deterministic stepsizes in linear networks 
may extend to nonlinear networks. Furthermore, a comparison of Figure~\ref{fig:regularized losses and Frobenius norm of individual weights with deterministic stepsize} with Figure~\ref{fig:nonlinear NN  regularized losses and Frobenius norm of individual weights with deterministic stepsize} indicates that similar patterns are observed in the SGD experiments 
for both linear and nonlinear networks. 
The values of the regularization parameters $\lambda,$ $\beta$ and the rank $r$ are consistent with those used in Figure~\ref{fig:regularized losses and Frobenius norm of individual weights with deterministic stepsize}. 
\begin{figure}[h!] 
				\begin{subfigure}{0.45\textwidth}
					\centering\includegraphics[width=0.99\textwidth]{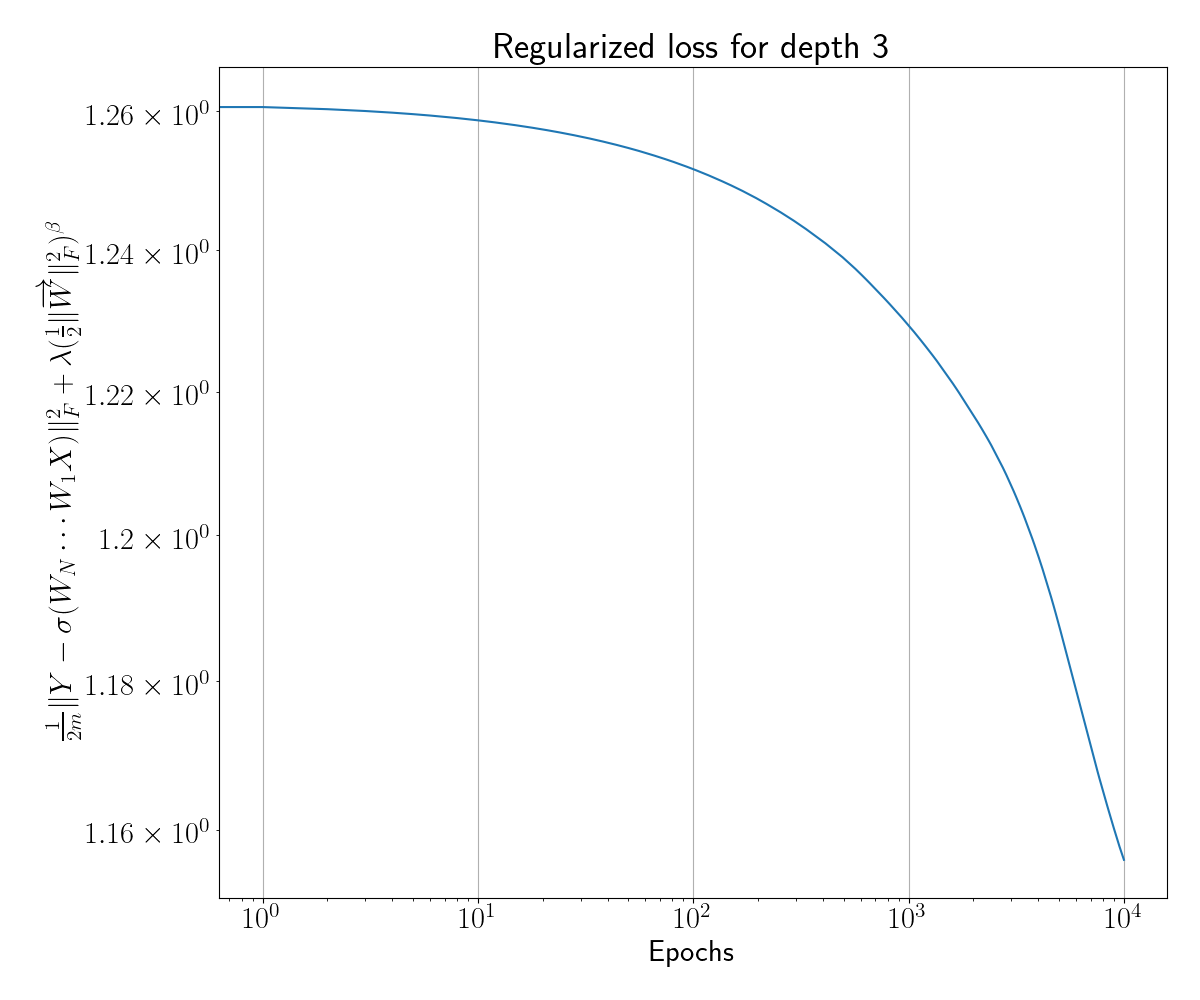}
					\caption{Evolution of the 
                   regularized 
                   losses for the network of depth 3}
				\label{fig:nonlinear NN regularized_loss deterministicetak N3}
			\end{subfigure}\hfill\begin{subfigure}{0.45\textwidth}\centering\includegraphics[width=0.99\textwidth]{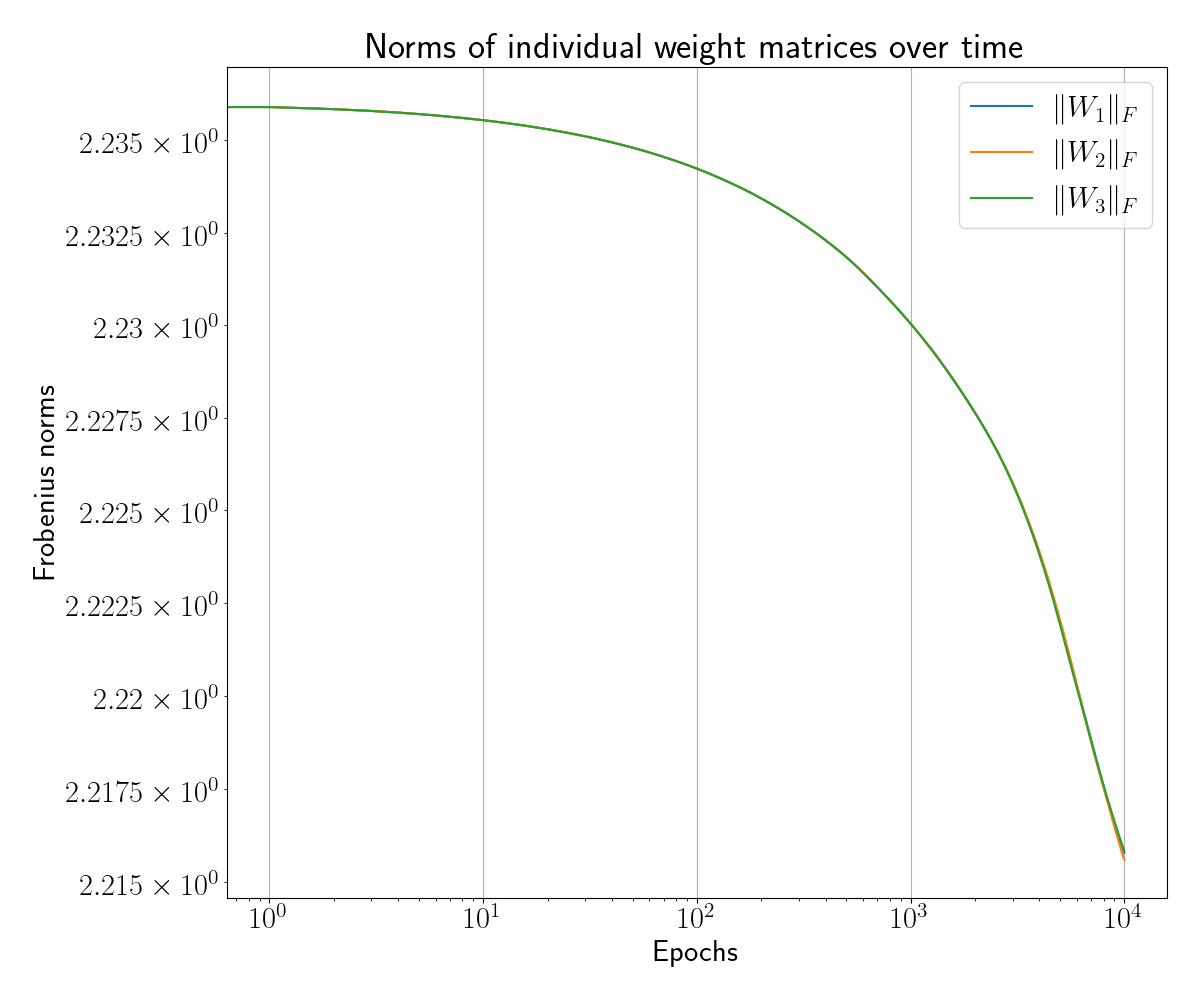}
				\caption{
					Frobenius norm of the individual weight matrices of the network of depth 3}
			\label{fig:nonlinear NN Norms_of_Wj deterministicetak N3}
		\end{subfigure}
           				\begin{subfigure}{0.45\textwidth}
					\centering\includegraphics[width=0.99\textwidth]{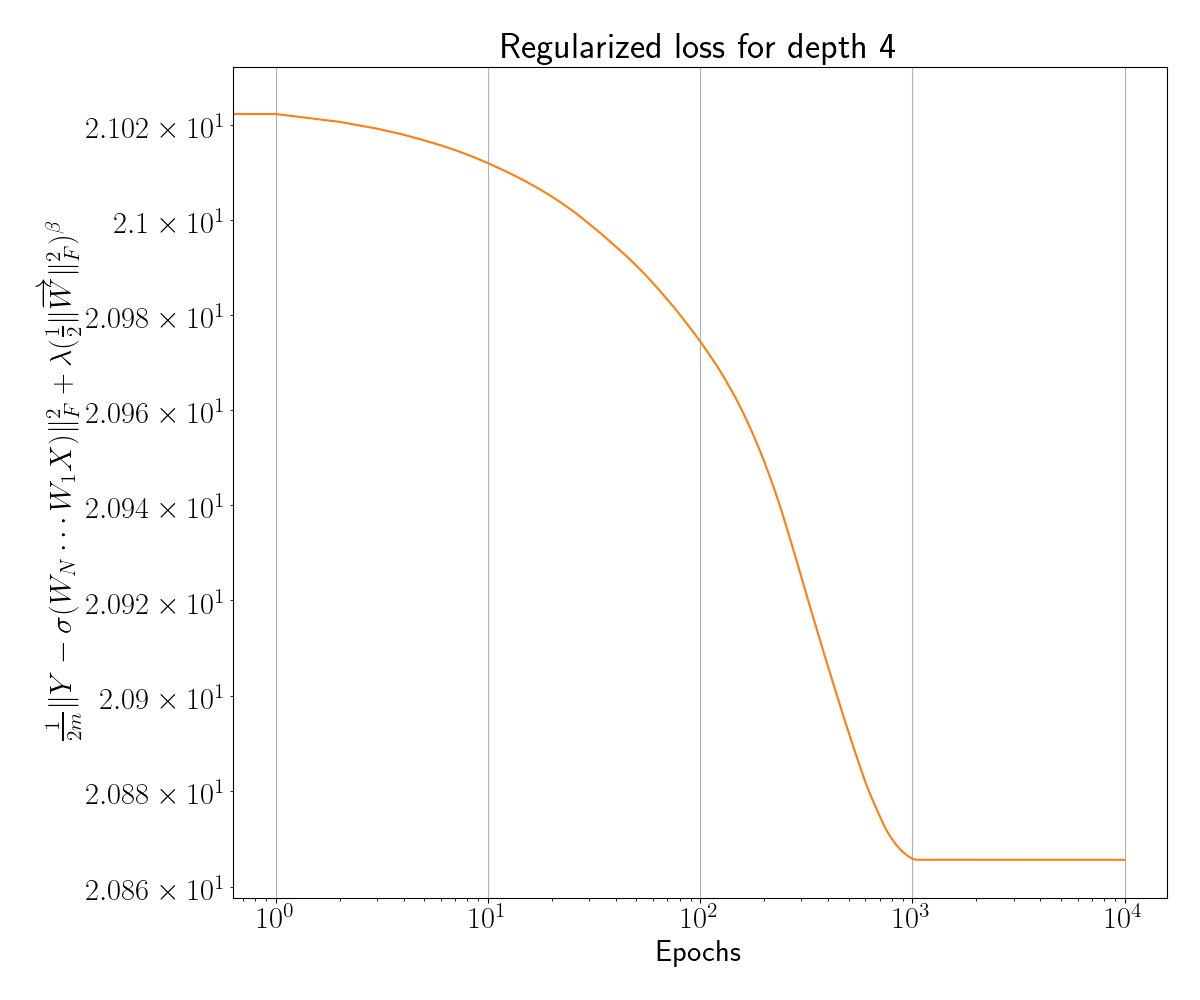}
					\caption{Evolution of the 
                   regularized 
                   losses for the network of depth 4}
				\label{fig:nonlinear NN regularized_loss deterministicetak N4}
			\end{subfigure}\hfill\begin{subfigure}{0.45\textwidth}\centering\includegraphics[width=0.99\textwidth]{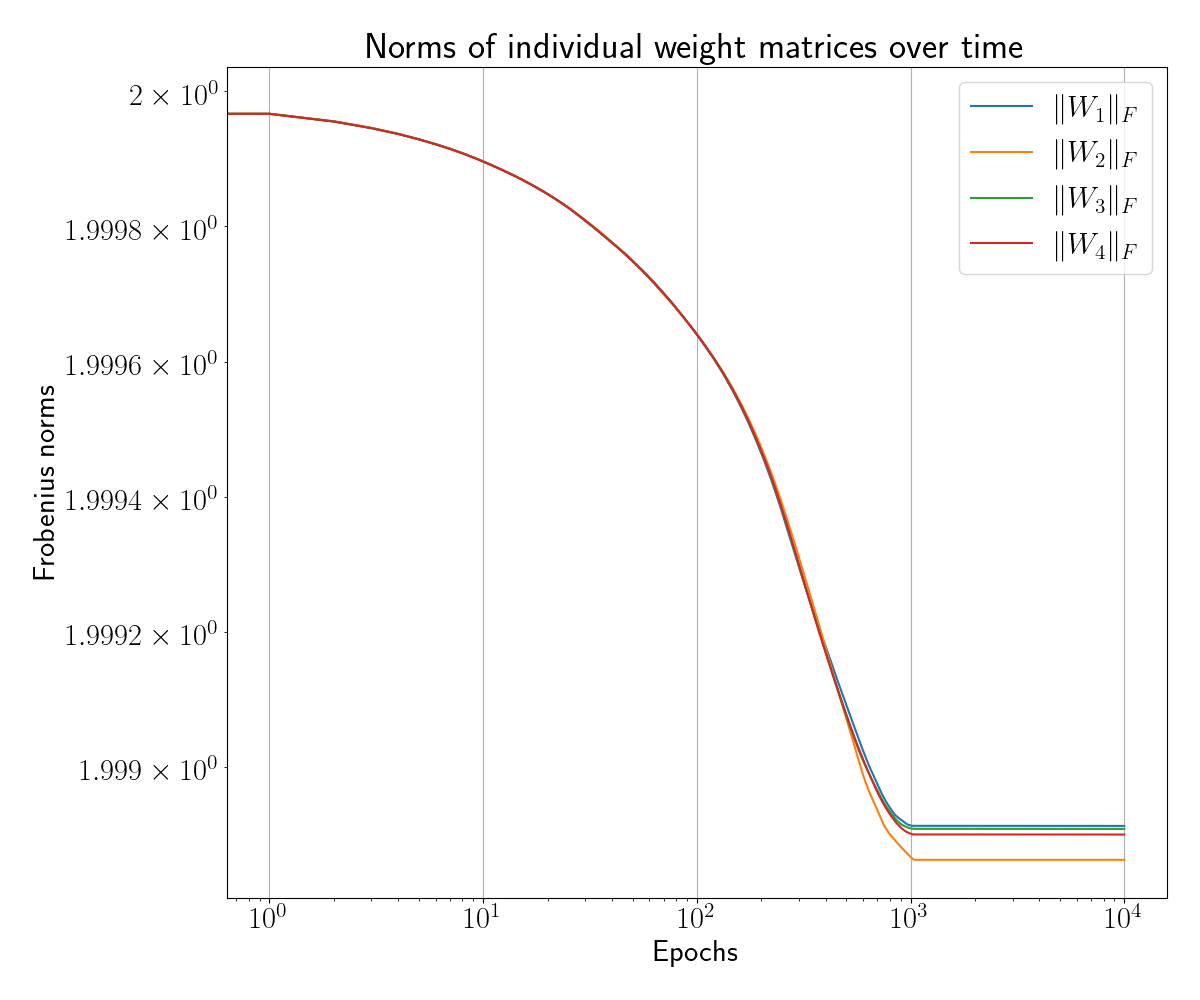}
				\caption{
					Frobenius norm of the individual weight matrices of the network of depth 4}
			\label{fig:nonlinear NN Norms_of_Wj deterministicetak N4}
		\end{subfigure}
		\caption{Training nonlinear neural networks: evolution 
       of the regularized loss functions across networks under deterministic step sizes in $\eqref{exp: deterministic etak}$ on the left panel. Frobenius norms of individual weight matrices of the corresponding networks for depth 3, and 4 
       on the right panel.}\label{fig:nonlinear NN  regularized losses and Frobenius norm of individual weights with deterministic stepsize}
\end{figure}

Figure~\ref{fig:nonlinear NN  regularized losses and Frobenius norm of individual weights with adaptive stepsize} depicts the behavior of SGD with adaptive stepsizes scheduled as in \eqref{exp: adaptive etak}, during the training of nonlinear neural networks. 
Plots~\ref{fig:nonlinear NN regularized_loss adaptiveetak N3} and \ref{fig:nonlinear NN regularized_loss adaptiveetak N4} show that the regularized loss functions for nonlinear networks of depth $3$ and $4$ converge when applied to SGD iterates. Plots~\ref{fig:nonlinear NN Norms_of_Wj adaptiveetak N3} and \ref{fig:nonlinear NN Norms_of_Wj adaptiveetak N4} show that the SGD iterates remain bounded during the training. 
The parameters $\lambda,$ $\beta,$ and $r$ used in Figure~\ref{fig:nonlinear NN  regularized losses and Frobenius norm of individual weights with adaptive stepsize} are consistent with those in Figure~\ref{fig:regularized losses and Frobenius norm of individual weights with adaptive stepsize}. A comparison of these two figures, 
indicates that the training trajectories of nonlinear neural networks are similar to those observed in their linear counterparts. 

\begin{figure}[h!] 
				\begin{subfigure}{0.45\textwidth}
					\centering\includegraphics[width=0.99\textwidth]{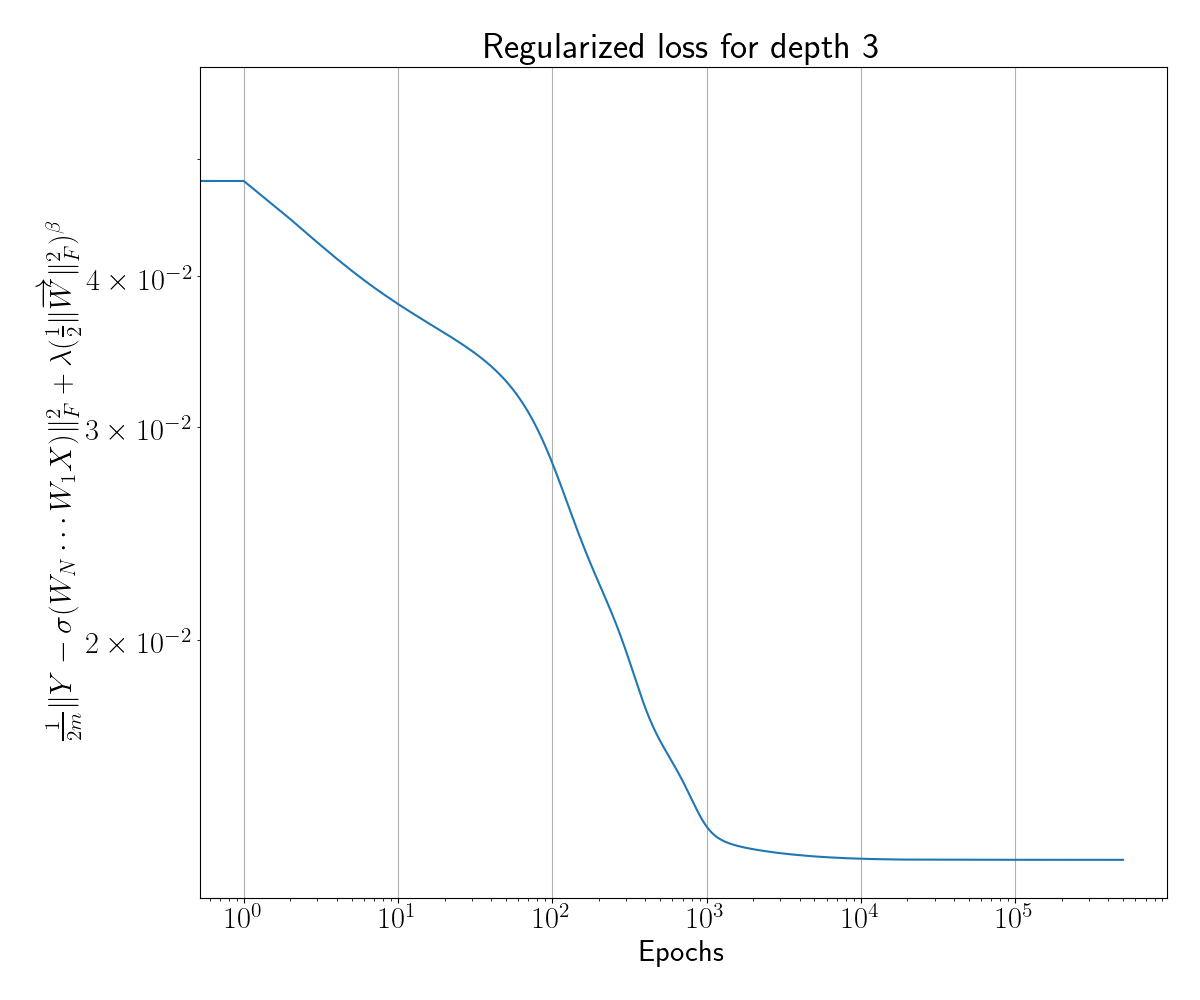}
					\caption{Evolution of the 
                   regularized 
                   losses for the network of depth 3}
				\label{fig:nonlinear NN regularized_loss adaptiveetak N3}
			\end{subfigure}\hfill\begin{subfigure}{0.45\textwidth}\centering\includegraphics[width=0.99\textwidth]{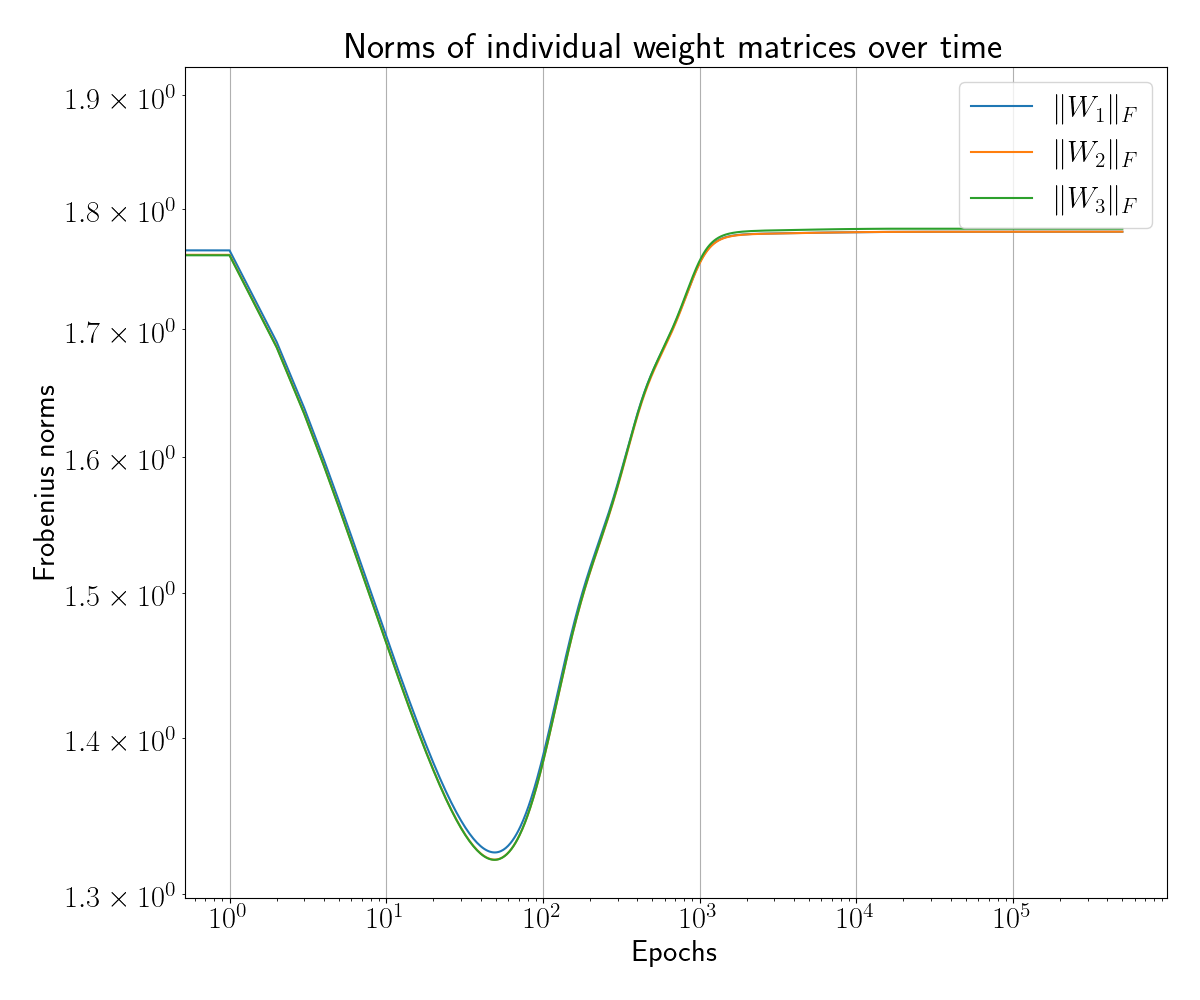}
				\caption{
					Frobenius norm of the individual weight matrices of the network of depth 3}
			\label{fig:nonlinear NN Norms_of_Wj adaptiveetak N3}
		\end{subfigure}
           				\begin{subfigure}{0.45\textwidth}
					\centering\includegraphics[width=0.99\textwidth]{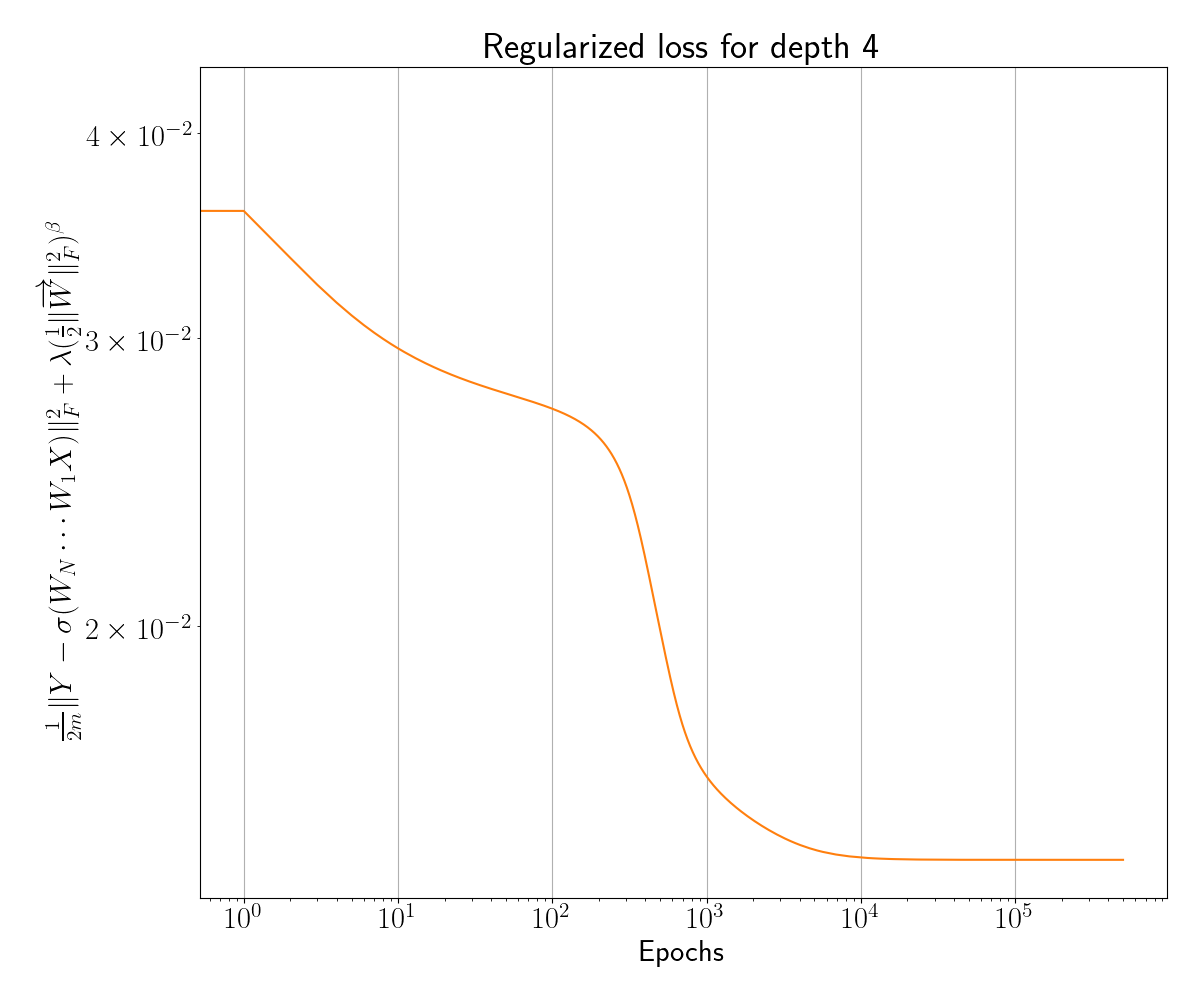}
					\caption{Evolution of the 
                   regularized 
                   losses for the network of depth 4}
				\label{fig:nonlinear NN regularized_loss adaptiveetak N4}
			\end{subfigure}\hfill\begin{subfigure}{0.45\textwidth}\centering\includegraphics[width=0.99\textwidth]{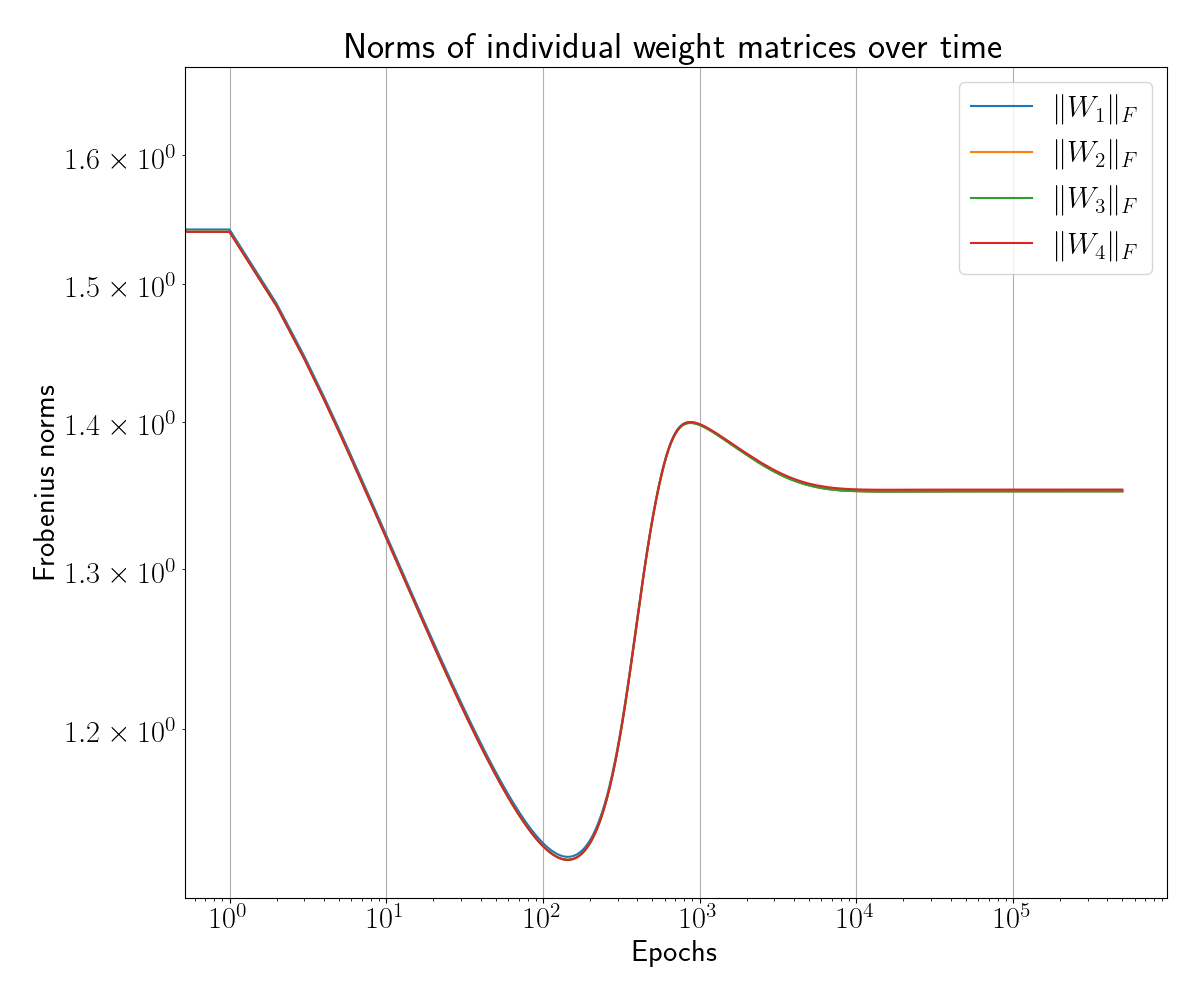}
				\caption{
					Frobenius norm of the individual weight matrices of the network of depth 4}
			\label{fig:nonlinear NN Norms_of_Wj adaptiveetak N4}
		\end{subfigure}
		\caption{Training nonlinear neural networks: evolution 
       of the regularized loss functions across networks under adaptive stepsizes in $\eqref{exp: adaptive etak}$ on the left panel. Frobenius norms of individual weight matrices of the corresponding networks for depth 3 and 4 
       on the right panel.}\label{fig:nonlinear NN  regularized losses and Frobenius norm of individual weights with adaptive stepsize}
\end{figure}

\FloatBarrier

  \section{Proofs}\label{sec:proofs}
  
  \subsection{Proof of Proposition~\ref{prop:crucial_apt_tool_update}} \label{prf:crucial_apt_tool_update}

\begin{proof}  
  The proof is adapted from the proof of \cite[Proposition 4.2]{benaim1999dynamics} additionally taking into account the randomness of the stepsizes.
  
  Let $t\geq 0,$ let $ T>0,$ and 
  define for $i\in \mathbb N_0$ the truncated step size $a_i$ as the length (i.e., the  Lebesgue measure) of the interval $[\tau_i, \tau_{i+1})\cap [t, t+T]$. Since $\eta_{i+1}$ is $\F_i$ measurable, it follows that $a_i$  is also $\F_i$ measurable (use that $a_i = \max\left\{\min\{\tau_{i+1},t+T\}-\max\{\tau_i,t\},0\right\}$). Also, since almost surely the sum of the $\eta_i$ diverges, it  holds almost surely that  only finitely many $a_i$ are nonzero. Indeed, we have $a_i=0$ for $i<m(t)$ and for $i>m(t+T)$.  (Recall that $m(s):=\sup\{k\geq0:~s\geq\tau_k\}$.) 
  We consider the process $\{S_k\}$ defined by
  	\begin{align*}
  		S_k =\sum_{i=0}^{k-1} a_{i} U_{i+1},\quad \text{with $S_0=0.$}
  	\end{align*}
  	Then the process $\{S_k\}$ is a martingale with respect to the filtration $(\F_k)$.
    Indeed, obviously $S_k$ is $\F_k$ measurable and $\E[S_{k+1}|\F_k]=\E[S_{k}+a_kU_{k+1}|\F_k]=S_k+a_k\E[U_{k+1}|\F_k]=S_k$, where we used that $S_k$ and $a_k$ are $\F_k$ measurable.

Recall that by definition $\displaystyle\Delta(t,T)=\sup_{0\leq h\leq T}\left\|\int_{t}^{t+h}\overline{U}(s)ds\right\|$ with $\overline{U}$ defined in \eqref{def: U-bar eta-bar}. Since the integral is a continuous and piecewise affine function in $h$ and   the norm is convex, the supremum is attained at a break point, i.e., one of the points $\{t,\tau_{m(t)+1},\tau_{m(t)+2},\hdots,\tau_{m(t+T)}, t+T\}$. Thus 
$$\displaystyle\Delta(t,T)=\sup_{k\geq 0}\left\|\sum_{i=0}^{k-1}a_iU_{i+1}\right\|.$$
   Now we  apply Burkholder's inequality  
   \cite[Theorem 6.25]{Stroock_2024}
  to obtain the estimate 
\begin{align}\label{qth moment of S_k}
  		\E\left[ (\displaystyle\Delta(t,T) )^q \right] 
  		\leq C_q \E\left[ \left( \sum_{i= 0}^{\infty}a_i^2 \| U_{i+1} \|^2 \right)^{\frac{q}{2}} \right]
  	\end{align}
    for some constant $C_q > 0$.
      To further estimate the upper bound in $\eqref{qth moment of S_k}$,
  	we 
    use the following consequence of Hölder's inequality valid for any $\alpha_i, \beta_i \in \RR$, $u>1$ and $0 < \delta < 1$:
  	\begin{align}\label{variant of Holder inequality}
  		\sum_i |\alpha_i \beta_i| 
  		\leq \left( \sum_i |\alpha_i|^{\frac{\delta u}{u-1}} \right)^{\frac{u-1}{u}} \left(\sum_i |\alpha_i|^{(1-\delta)u} |\beta_i|^u\right)^{\frac{1}{u}}. 
  	\end{align}
     Assuming $q>2$ and setting $u = \frac{q}{2}$, $\delta =\frac{q-2}{2q}$, $\alpha_i = a_{i}^2$ and $\beta_i = \| U_{i+1} \|^2$, inequality $\eqref{variant of Holder inequality}$ 
  	applied to the estimate in $\eqref{qth moment of S_k}$ 
  	gives
  	\begin{align*}
  		&\E\left[ (\displaystyle\Delta(t,T) )^q \right] 
  		\leq C_q \E \left[ \left( \sum_{i=0}^{\infty} a_{i} \right)^{\frac{q}{2} - 1} \left( \sum_{i=0}^{\infty} a_i^{1+\frac{q}{2}} \| U_{i+1} \|^q \right) \right].  
  	\end{align*} Note  that the involved  sums are a.s. finite sums.
  	Since $\sum_{i=0}^{\infty} a_{i}=T$ a.s., we obtain
  	\begin{align}\label{E-est1}
  		&\E\left[ (\displaystyle\Delta(t,T) )^q \right] 
  		\leq C_q T^{\frac{q}{2} - 1}\E \left[  \sum_{i=0}^{\infty} a_i^{1+\frac{q}{2}} \| U_{i+1} \|^q  \right]. 
  	\end{align}
   Note that this inequality also holds for $q=2$ because in this case it coincides with \eqref{qth moment of S_k}.
Next, observe that 
\begin{equation}
\begin{aligned}\label{E-est2}
   \E \left[  \sum_{i=0}^{\infty} a_i^{1+\frac{q}{2}} \| U_{i+1} \|^q  \right]&=\sum_{i=0}^{\infty} \E \left[ a_i^{1+\frac{q}{2}} \| U_{i+1} \|^q  \right] = \sum_{i=0}^{\infty} \E \left[ \E \left[a_i^{1+\frac{q}{2}} \| U_{i+1} \|^q \mid\F_i \right] \right] \\& = \sum_{i=0}^{\infty} \E\left[ a_i^{1+\frac{q}{2}}\E \left[  \| U_{i+1} \|^q \mid\F_i \right]\right]\leq C\sum_{i=0}^{\infty} \E\left[ a_i^{1+\frac{q}{2}}\right]=C\E\left[\sum_{i=0}^{\infty}  a_i^{1+\frac{q}{2}}\right]. 
\end{aligned}
\end{equation}
Here, the first and the last equality  hold by Tonelli's theorem (or monotone convergence), the third equality holds because  $a_i$ is $\F_i$-measurable, and the  inequality is due to the assumption $\E \left[  \| U_{i+1} \|^q \mid\F_i \right]\leq C$.

Now, since $a_i\leq \eta_{i+1}$ by definition of  $a_i$, we obtain (on the a.s. event that the sum of the $\eta_i$ diverges) $$\sum_{i=0}^{\infty} a_i^{1+\frac q2}=\sum_{i=0}^{\infty} a_ia_i^{\frac q2}\leq \sum_{i=0}^{\infty} a_i\eta_{i+1}^{\frac q2}=\sum_{i=0}^{\infty} \int_{[\tau_i, \tau_{i+1})\cap [t, t+T]} \overline{\eta}(s)^{q/2} ds=\int_t^{t+T} \overline{\eta}(s)^{q/2} ds,$$    
which implies 
\begin{align}\label{E-est3}
  \E\left[\sum_{i=0}^{\infty}  a_i^{1+\frac{q}{2}}\right]\leq \E\left[  \int_t^{t+T} \overline{\eta}(s)^{q/2} ds\right].
\end{align} 
Combining \eqref{E-est1}, \eqref{E-est2}, and \eqref{E-est3}, 
we obtain 
    \begin{align*}
    \E \left[\Delta(t,T)^q\right] & \leq 
    CC_q T^{\frac q2-1} \E \left[\int_{t}^{t+T} \overline{\eta}(s)^{q/2} ds\right].
    \end{align*}    
    We conclude that
  	\begin{align}\label{sum pth moment delta}
  		\sum_{k=0}^{\infty}\E[\Delta(kT,T)^q] &\leq C C_q T^{\frac q2-1} \E\left[\int_{0}^{\infty} \bar{\eta}^{\frac{q}{2}}(s) ds\right]
  		 = C C_q T^{\frac q2-1}\E\left[ \sum_{i=0}^{\infty}\eta_{i+1}^{1+\frac{q}{2}}\right] < \infty
  	\end{align}
  	since by assumption $\displaystyle \E\left[ \sum_{i=0}^{\infty}\eta_{i+1}^{1+\frac{q}{2}}\right]<\infty.$ 
    
    Given $\varepsilon>0$, define the event $A_k=\{\Delta(kT,T)>\varepsilon\}$. 
  	Markov's inequality implies that
  		$\mathbb{P}(A_k) \leq \E[\Delta(kT,T)^q]\varepsilon^{-q}$,
  	and taking into account $\eqref{sum pth moment delta}$ gives
  	\begin{align*}
  		\sum_{k=0}^{\infty}P(A_k) < \infty.  
  	\end{align*}
  	The Borel-Cantelli Lemma implies then that with probability one, for all sufficiently large $k$ we have
  	\begin{align*}
  		\Delta(kT,T)\leq\varepsilon.
  	\end{align*}
    Since $\varepsilon>0$ is arbitrary this shows that
  	\begin{align*}
  		\lim_{k\to\infty}\Delta(kT,T)=0
  	\end{align*}
  	with probability one. Observing that for any $t\in [kT, (k+1)T]$ we have 
  	\begin{align*}
  		\Delta(t,T)\leq 2 \Delta(kT,T)+\Delta((k+1)T,T),
  	\end{align*}
  	we conclude
  	\begin{align*}
  		\lim_{t\to\infty}\Delta(t,T)=0
  	\end{align*}
  	with probability one. This completes the proof.
  \end{proof}
  
 \subsection{Proof of Corollary ~\ref{cor:apt_convergence}} 
 We prove Corollary ~\ref{cor:apt_convergence}
 with assumption (1) replaced by the slightly weaker assumptions given in Remark \ref{adapted_assumption}.
 \label{prf:apt_convergence}
  \begin{proof} The proof strategy leverages Sard's theorem together with Theorem \ref{thm:chain_transitive_set} and Proposition \ref{prop:chain_transitive_set} in the appendix.  
  
  For any initial value $x$ the semiflow $\Phi$ is defined on $[0, \infty)$ via the gradient flow equation 
  	\begin{align*}
  		\frac{d\Phi_x(t)}{dt} &= -\nabla V(\Phi_x(t)), \ \Phi_x(0)=x.
  	\end{align*}
   Since we assume uniqueness of the solutions, the equilibrium points coincide with the critical points of $V$.

    Let $\operatorname{Crit}(V)$ be the set of critical points of $V$ and 
let $\Lambda:= \operatorname{Crit}(V)\cap L(X)$. (Recall that $L(X)$ is the set of limit points of $X$.)
We note that $L(X)$ is  compact  (by Theorem \ref{thm:chain_transitive_set} whose application uses the assumption that $X$ is precompact)   and since $V$ is of class $C^1$, the set $\operatorname{Crit}(V)$ is closed and therefore $\Lambda$ is compact. The set $\Lambda$ is also invariant under the semiflow because $\Lambda \subseteq \operatorname{Crit}(V)$. 

Now we consider 
the restriction of $\Phi$ to $\mathbb R_+\times L(X)$. Because $L(X)$ is invariant under  $\Phi$ (again by Theorem \ref{thm:chain_transitive_set}), this restriction defines indeed  again  a semiflow and obviously $\Lambda$ is the set of equilibrium points of this restricted semiflow.

Now we claim that $V_{|L(X)}$ (i.e., the restriction of $V$ to $L(X)$) is a Lyapunov function for $\Lambda$ (wrt.\ the restricted flow on $\mathbb R_+\times L(X)$).
Indeed, by the assumption that the flow trajectories are unique (see Remark \ref{adapted_assumption}), it follows that for $x\in \Lambda$ the map $\mathbb R_+\ni t\mapsto V(\Phi_x(t))$ is constant because $x$ is a equilibrium.
Furthermore, 
since \begin{align}\label{Vdot} 
  		\frac{dV(\Phi_x(t))}{dt} &= \nabla V(\Phi_x(t))\cdot \frac{d \Phi_x(t)}{dt} = -\left\|\nabla V(\Phi_x(t))\right\|^2\leq 0,
  	\end{align}
 for $x\in L(X)\setminus \Lambda$ the expression $\frac{dV(\Phi_x(t))}{dt} $ is negative at $t=0$. Additionally, it is also negative for all $t>0$ because otherwise there would exist a minimal $t>0$ with $\frac{dV(\Phi_x(t))}{dt}=0$ and then by \eqref{Vdot} $y:=\Phi_x(t)$ would be a critical point of $V$. But by assumption   (see Remark \ref{adapted_assumption}), the trajectory $\Phi_y(t)$ is uniquely determined and hence would constantly be equal to $y$ on an interval  $(-\varepsilon(y),\infty)$ for some $\varepsilon(y)>0$, implying that $\Phi_x(t)$ would  constantly be equal to $y$  on the interval $(t-\varepsilon(y),\infty)$   contradicting the minimality of $t$.

    Since the smoothness of $V$ satisfies $r\geq m>\max\{0, m-1\}$, Sard's theorem \cite[Chapter 3, Proposition 1.3]{hirsch2012differential}
    
    implies that the Lebesgue measure in $\RR$ of  $V(\operatorname{Crit}(V))$ is zero implying that also $V(\Lambda)$ has Lebesgue measure zero and in particular has
    empty interior.
   
    To summarize, we have shown that $\Lambda$ is a compact invariant set in $L  (X)$ (wrt. the restricted flow), that  $V_{|L(X)}$  is a Lyapunov function for $\Lambda$ and that $V(\Lambda)$ has empty interior. Thus we can apply Propostion  \ref{prop:chain_transitive_set}  which implies that $L(X)$ (which is internally chain  transitive by Theorem \ref{thm:chain_transitive_set})) is contained in $\Lambda$ and that $V_{|L(X)}$ is constant. 
    In particular, $L(X)$ consists of equilibria.

  	It now remains to prove that $V(X(t))$ converges. For every sequence $(t_k)$ of nonnegative real numbers with $\lim_{k\to \infty} t_k =\infty$, the compactness of $\overline{\operatorname{im}(X)}$ yields a subsequence of the sequence $\left(X(t_k)\right)_k$ that converges in $\overline{\operatorname{im}(X)}$. 
    Let $x$ be the limit of such a converging subsequence, hence, $x\in L(X)$ by definition of $L(X)$. 
    Set $v=V(x)$ and by way of contradiction, let us assume that for some  
  	sequence of nonnegative real numbers $(\tilde{t}_k)_k,$ with $\lim_{k\to \infty} \tilde{t}_k =\infty$ the sequence $\left(V(X(\tilde{t}_k))\right)_k$ does not converge. Then there exists a subsequence $(\tilde{t}_{k_\ell})_\ell,$ and $\epsilon>0$ such that 
  	\begin{align}\label{no-convergence}
  		\left|V(X(\tilde{t}_{k_\ell})) - v\right|\geq \epsilon\quad \text{for all $\ell$}.
  	\end{align}
  	Moreover, by compactness of $\overline{\operatorname{im}(X)}$, the sequence $\left(X(\tilde{t}_{k_\ell})\right)_\ell$ has a subsequence that converges to a point $\bar x\in L(X)$. Given that the restriction of $V$ to $L(X)$ is constant (as argued above), we have $V(\bar x)=v$ which contradicts \eqref{no-convergence}.  
  \end{proof}
  
  \subsection{Proof of Proposition~\ref{prop:bounded}} \label{prf:bounded}
  \begin{proof}
	A key ingredient of this proof is the following technical lemma, which estimates the Frobenius norm of the stochastic gradient with respect to Frobenius of the tuple $\W$, and the spectral norm of the mini batch data $X_S,~Y_S$.
	\begin{lemma}\label{lem:grad-bound} Let $S \subset [m]$ and $\W = (W_1, \hdots, W_N)$. Setting $W = W_N \cdots W_1$, it holds
		\begin{align*}
			\|\nabla L_S^N(\W) \|_F & \leq \frac{\sqrt{N}}{T}\|\W\|_F^{N-1} \left( \|W\| \|X_S\|^2 + \| Y_S X_S^T\|\right)\\
			& \leq \frac{\sqrt{N}}{T}\left(\|X_S\|^2 \|\W\|_F^{2N-1} + \| Y_S X_S^T\| \|\W\|_F^{N-1}\right).
		\end{align*}
	\end{lemma}
    Here $\| \cdot  \|=\| \cdot  \|_{2\to2}$ denotes the operator norm wrt. the $2$-norm.
	\begin{proof} Recalling $\nabla L^N_{S}\big(\W\big)=\begin{pmatrix}
			\nabla_{W_1}L_S^N\big(\W\big)\\
			\vdots\\
			\nabla_{W_N}L_S^N\big(\W\big)
		\end{pmatrix},$ we have
		\begin{align}\label{grad L_S^N norm sq}
			\left\|\nabla L^N_{S}\big(\W\big)\right\|_F^2 = \left\|\nabla_{W_1} L^N_{S}\big(\W\big)\right\|_F^2 + \cdots + \left\|\nabla_{W_N} L^N_{S}\big(\W\big)\right\|_F^2.
		\end{align}
        The gradients of $L_S^1$ and $L_S^N$ are given by 
        \cite{arora2018convergence,arora2018optimization,bah2022learning}
		\begin{align}
			\nabla L^1_S(W)&=\frac{1}{T}\big(W X_SX_S^T - Y_SX_S^T\big) ,\label{stochastic grad L1}\\
			\nabla_{W_j} L^N_{S}\big(\W\big)&=W_{j+1}^T \hdots W_N^T \nabla L^1_S(W) W_1^T \hdots W_{j-1}^T. \label{stochastic grad LN}
		\end{align}
		It follows that
		\begin{align*}
			\left\|\nabla_{W_j} L^N_{S}\big(\W\big)\right\|_F
			& = \|W_{j+1}^T \hdots W_N^T \nabla L^1_S(W) W_1^T \hdots W_{j-1}^T\|_F
			\leq \prod_{i=1, i \neq j}^N \|W_j\|_F \|\nabla L^1_S(W)\|
		\end{align*}
		whereby we used the Frobenius norm property $\| A B \|_F \leq \|A \| \|F\|_F$ for matrices $A$, $B$ (of matching dimensions) to get the second inequality. Further, using the fact that $\|A\| \leq \|A\|_F$ and $\|W_j\|_F \leq \|\W\|_F$ we obtain
		\begin{align*}
			\left\|\nabla_{W_j} L^N_{S}\big(\W\big)\right\|_F
			& \leq \frac{1}{T}\|\W\|_F^{N-1} \|(W X_S - Y_S) X_S^T\|
			\leq \frac{1}{T}\|\W\|_F^{N-1} (\|W_N \cdots W_1\| \|X_S X_S^T\| + \|Y_S X_S^T\|)\\
			& \leq \frac{1}{T}\|\W\|_F^{N-1} (\|\W\|_F^N \|X_S X_S^T\| + \|Y_S X_S^T\|).
		\end{align*}
		Substituting the above estimates into $\eqref{grad L_S^N norm sq}$ gives
		\begin{align*}
			\|\nabla L_S^N(\W) \|_F & \leq \frac{\sqrt{N}}{T}\|\W\|_F^{N-1} \left( \|W\| \|X_S\|^2 + \| Y_S X_S^T\|\right)
			\leq \frac{\sqrt{N}}{T}\big(\|X_S\|^2 \|\W\|_F^{2N-1} + \| Y_S X_S^T\| \|\W\|_F^{N-1}\big).
		\end{align*}
		This completes the proof.
	\end{proof}
	
	We now prove  Proposition~\ref{prop:bounded} using an induction argument.
	For $k=0$, inequality \eqref{eq:bound_w} holds trivially since $\|\W(0)\|_F \leq \max\{ \|\W(0)\|, M+B\}$. 
	Let us assume that \eqref{eq:bound_w} holds for $0,1,\dots,k$ and prove the statement for $k+1$. The stochastic gradient descent scheme \eqref{eq:sgd} gives
	\begin{align}
		\W(k+1) &=\W(k)-\eta_{k}\nabla\Tilde{L}^N_{S_{k}}\big(\W(k)\big)
		=\W(k)-\eta_{k} \left(\nabla R_{\beta,\lambda}\big(\W(k)\big)+ \nabla L^N_{S_{k}}\big(\W(k)\big)\right)\nonumber\\
		&=\W(k) - \lambda \beta \eta_{k}\W(k)\left(\frac{1}{2}\|\W(k)\|_F^2\right)^{\beta-1}-\eta_{k}\nabla L^N_{S_{k}}\big(\W(k)\big)\nonumber\\
		&= \left(1 - \beta\lambda\eta_{k} \left(\frac{1}{2} \|\W(k)\|_F^2\right)^{\beta-1} \right) \W(k) 
		-\eta_{k} \nabla L^N_{S_{k}}\left( \W(k) \right), \label{W-iterate}
	\end{align}
	where we substituted the gradient of $R_{\beta,\lambda}$ to get the third equation. By using the definition of the stepsize as $\eta_k = \frac{b_k}{\beta\lambda \max\left\{\frac{1}{2}\|\W(0)\|_F^2, \frac{1}{2}(M+B)^2 \right\}^{\beta-1}}$ and the induction hypothesis together with $b_k\leq 1$, we obtain
	\begin{equation}\label{bound1}
		\beta\lambda\eta_{k}\left(\frac{1}{2}\|\W(k)\|_F^2\right)^{\beta-1}
		= \frac{b_k\left(\frac{1}{2}\|\W(k)\|_F^2\right)^{\beta-1}}{\max\left\{\frac{1}{2}\|\W(0)\|_F^2,\frac{1}{2}(M+B)^2\right\}^{\beta-1}} \leq 1.
	\end{equation}
	At this point, we distinguish two cases, namely $\|\W(k)\|_F \leq M$ and $\|\W(k)\|_F > M.$
	
	We proceed with the first case assuming $\|\W(k)\|_F \leq M$. Then \eqref{bound1} together with the triangle inequality applied to 
	\eqref{W-iterate} gives
	\begin{align}
		\| \W(k+1) \|_F & \leq 
		\left| 1 - \beta\lambda\eta_{k}\left(\frac{1}{2}\|\W(k)\|_F^2\right)^{\beta-1}\right| \|\W(k)\|_F + \eta_k \left\|\nabla L^N_{S_{k}}\big(\W(k)\big)\right\|_F\nonumber\\
		& \leq \|\W(k)\|_F + \eta_k \frac{\sqrt{N}}{T}\left(\|X_S\|^2 \|\W(k)\|_F^{2N-1} + \| Y_S X_S^T\| \|\W(k)\|_F^{N-1}\right)\nonumber\\
		&\leq\|\W(k)\|_F + \eta_k\sqrt{N}\left(P_T \|\W(k)\|_F^{2N-1} + Q_T \|\W(k)\|_F^{N-1}\right)\nonumber\\
		& \leq M + \frac{\sqrt{N}b_k(P_T M^{2N-1} + Q_T M^{N-1})}{\beta\lambda \max\left\{\frac{1}{2}\|\W(0)\|_F^2, \frac{1}{2}(M+B)^2\right\}^{\beta-1}} \label{Wk-bound1}
	\end{align}
	where we used the bound of Lemma \ref{lem:grad-bound} to get the second inequality, and use the definitions \eqref{def:p_t} and \eqref{def:q_t} of $P_T$ and $Q_T$ to get the third inequality. We then used $\|\W(k)\|_F \leq M$ to obtain the last inequality. The definition \eqref{def:B} of $B$ yields
	\begin{align}\label{eq:derive from def B}
		\sqrt{N}\left(P_T M^{2N-1} + Q_T M^{N-1}\right) = B \lambda \beta \max\left\{\frac{1}{2}\|\W(0)\|_F^2, \frac{1}{2}M^2\right\}^{\beta-1}.
	\end{align}
	Substituting this  
	into \eqref{Wk-bound1} and using $b_k\leq 1$ 
    leads to 
	\begin{align*}
		\|\W(k+1)\|& \leq M + \frac{B \lambda \beta \max\left\{\frac{1}{2}\|\W(0)\|_F^2, \frac{1}{2}M^2\right\}^{\beta-1}}{\beta\lambda \max\left\{\frac{1}{2}\|\W(0)\|_F^2, \frac{1}{2}(M+B)^2\right\}^{\beta-1}}
		\leq M+B,
	\end{align*}
	which proves \eqref{eq:bound_w} for $k+1$ in the case where $\|\W(k)\|_F \leq M.$
	
	For the second case we assume that $\|\W(k)\|_F > M$.
	Applying the triangle inequality on $\eqref{W-iterate}$ and then using \eqref{bound1} and the bound in Lemma~\ref{lem:grad-bound} gives
	\begin{align}
		&\|\W(k+1)\|_F \nonumber\\
		\leq& \left(1-\beta \lambda \eta_k \left( \frac{1}{2} \|\W(k)\|_F^2\right)^{\beta-1}\right)\|\W(k)\|_F + \eta_k \|\nabla L^N_{S_k}(\W(k))\|_F\nonumber\\
		\leq& \left(1-\beta \lambda \eta_k \left( \frac{1}{2} \|\W(k)\|_F^2\right)^{\beta-1}\right)\|\W(k)\|_F + \eta_k\frac{\sqrt{N}}{T}\left(\|X_S\|^2 \|\W(k)\|_F^{2N-1} + \| Y_S X_S^T\| \|\W(k)\|_F^{N-1}\right)\nonumber\\
		\leq& \left(1 - \eta_k\left( \beta \lambda \left(\frac{1}{2} \|\W(k)\|_F^2\right)^{\beta-1} - P_T\sqrt{N}\|\W(k)\|_F^{2N-2} - Q_T\sqrt{N}\|\W(k)\|_F^{N-2}\right) \right)\|\W(k)\|_F,\label{bound SGD at k+1}
	\end{align}
	where we used $\eqref{def:p_t}$ and $\eqref{def:q_t}$ estimating $\|X_S\|^2,$ and $\| Y_S X_S^T\|,$ respectively with $P_T$  and $Q_T$ to get the last inequality.
	
	We claim that $\|\W(k)\|_F \geq M$ implies 
	\begin{equation}\label{WM-bound}
		\beta \lambda \left(\frac{1}{2} \|\W(k)\|_F^2\right)^{\beta-1} - P_T\sqrt{N}\|\W(k)\|_F^{2N-2} - Q_T\sqrt{N} \|\W(k)\|_F^{N-2} \geq 0.
	\end{equation}
	Set $z = \|\W(k)\|_F$. Then $z \geq M > 0$ and inequality $\eqref{WM-bound}$ is 
	equivalent to 
	\begin{equation}\label{z:eq}
		\lambda c_\beta z^{2\beta-N} - P_T\sqrt{N} z^N - Q_T\sqrt{N} \geq 0,
	\end{equation}
	with $c_\beta = \beta 2^{1-\beta}.$
	
	If $\beta = N$ then, by assumption, $\lambda c_\beta > P_T\sqrt{N}$ and 
		$z \geq \left(\frac{Q_T\sqrt{N}}{\lambda c_\beta - P_T\sqrt{N}}\right)^{\frac{1}{N}} = M$
    implies $ \eqref{z:eq}$.
	We now consider the case $\beta > N$.
	If $z \geq \left(\frac{2P_T\sqrt{N}}{\lambda c_\beta}\right)^{\frac{1}{2(\beta - N)}}$ then
	$
	\lambda c_\beta z^{2\beta-N} \geq 2 P_T\sqrt{N}z^N.
	$
	If $z \geq \left(\frac{2Q_T\sqrt{N}}{\lambda c_\beta}\right)^{\frac{1}{2\beta-N}}$ then
	$
	\lambda c_\beta z^{2\beta-N} \geq 2 Q_T\sqrt{N}.
	$
	Consequently, if
	\[
	z \geq \max\left\{\left(\frac{2P_T\sqrt{N}}{\lambda c_\beta}\right)^{\frac{1}{2(\beta - N)}},  \left(\frac{2Q_T\sqrt{N}}{\lambda c_\beta}\right)^{\frac{1}{2\beta-N}}\right\} = M,
	\]
	then it holds
	\[
	\lambda c_\beta z^{2\beta-N} \geq 2 \max\left\{P_T\sqrt{N}z^N, Q_T\sqrt{N}\right\} \geq P_T \sqrt{N}z^N + Q_T\sqrt{N},
	\]
	which also implies $\eqref{z:eq}$. 
	Altogether, we proved our claim that $\|\W(k)\|_F \geq M$ implies $\eqref{WM-bound}$. Therefore, inequality $\eqref{bound SGD at k+1}$ implies $\|\W(k+1)\|_F \leq \|\W(k)\|_F$ and the induction hypothesis yields 
	\[
	\|\W(k+1)\|_F \leq \|\W(k)\|_F \leq \max\{\|\W(0)\|_F,M+B\}.
	\]
	This finishes the proof of Proposition~\ref{prop:bounded}.
\end{proof}

\subsection{Proof of Proposition~\ref{prop:bounded_adaptive}} \label{prf:bounded_adaptive}
\begin{proof}
    The proof is similar to the proof of Proposition~\ref{prop:bounded}. We prove the claimed bound 
	\begin{equation*}
		\|\W(k)\|_F \leq \max\left\{ \|\W(0)\|_F, M +D\sqrt N(P_TM^{2N-1}+Q_TM^{N-1})\right\} \quad\mbox{ for all } k \in \NN_0.
	\end{equation*} by induction on $k$. The case $k=0$ is clear.
	Now we assume that the bound holds for $k$ and show that this implies the bound for $k+1$.
	From equation $\eqref{W-iterate}$ we have
	\begin{align}
		\W(k+1) = \left(1 - \beta\lambda\eta_{k} \left(\frac{1}{2} \|\W(k)\|_F^2\right)^{\beta-1} \right) \W(k) 
		-\eta_{k} \nabla L^N_{S_{k}}\left( \W(k) \right) \label{W-iterate 2_alt}
	\end{align}
	We note that  the definition of the stepsize $\eta_k= b_k \cdot \min\left\{  \frac{1}{\beta\lambda(\frac12\|\W(k)\|_F^2)^{\beta-1}}, D\right\}$ immediately implies
	\begin{align}\label{2nd bound1_alt}
		\beta\lambda\eta_{k}\left(\frac{1}{2}\|\W(k)\|_F^2\right)^{\beta-1}
		\leq 1.
	\end{align}
	We now  again distinguish the cases  $\|W(k)\|_F\leq M$ and $\|W(k)\|_F>M.$ 
	First we assume $\|W(k)\|_F\leq M$. Then by \eqref{W-iterate 2_alt},\eqref{2nd bound1_alt} and the triangle inequality, it follows that
	\begin{align*}
		\| \W(k+1) \|_F & \leq 
		\left| 1 - \beta\lambda\eta_{k}\left(\frac{1}{2}\|\W(k)\|_F^2\right)^{\beta-1}\right| \|\W(k)\|_F + \eta_k \left\|\nabla L^N_{S_{k}}\big(\W(k)\big)\right\|_F\\
		&\leq\|\W(k)\|_F + \eta_k\sqrt{N}\left(P_T \|\W(k)\|_F^{2N-1} + Q_T \|\W(k)\|_F^{N-1}\right)\\
		& \leq M + D\sqrt{N}\left(P_T M^{2N-1} + Q_T M^{N-1}\right),
	\end{align*}
	which gives the desired bound on $\| \W(k+1) \|_F$.
	Here we used Lemma \ref{lem:grad-bound} together with \eqref{def:p_t} and $\eqref{def:q_t}$ to get the second inequality and  we used that $\eta_k\leq D$ (by definition of $\eta_k$)  and the assumption  $\|W(k)\|_F\leq M$ in the last inequality.
	Now assume $\|W(k)\|_F>M.$ In this case, by the same reasoning as in the proof of Proposition~\ref{prop:bounded}, we have
	\begin{align*}
		\| \W(k+1) \|_F\leq \| \W(k) \|_F,
	\end{align*} which by the induction hypothesis gives the desired bound.
	This completes the proof. 
	\end{proof}

\subsection{Proof of Theorem ~\ref{thm:sgd-conv2cset}} \label{prf:sgd-conv2cset}
\begin{proof}
	 The central argument of the proof is to show that the interpolated process $\overrightarrow{X} (t)$, see \eqref{eq:interp_process}, is a precompact asymptotic pseudotrajectory of the gradient flow defined via 
	\begin{align}\label{sgf_w2}
		\frac{d\W}{dt}= -\nabla\Tilde{L}^N\big(\W\big). 
	\end{align}
	Corollary~\ref{cor:apt_convergence} then leads to desired convergence properties.

	Let us start by showing that the stochastic gradient descent algorithm is a Robbins-Monro algorithm. We recall from equation \eqref{eq:interp_process} that in our setup
	\begin{align*}
		\overrightarrow{X} (t)=\W(k)+\frac{t-\tau_k}{\tau_{k+1}-\tau_k}\left(\W(k+1)-\W(k)\right) \quad \text{ for } t\in [\tau_k, \tau_{k+1}],~ k\in\NN.
	\end{align*}    
	In the case of deterministic stepsizes, we have $\displaystyle\eta_k = \frac{b_k}{\beta\lambda \max\left\{\left(\frac{1}{2}\|\W(0)\|_F^2\right), \frac{1}{2}(M+B)^2 \right\}^{\beta-1}}$ so that $\sum_{k=0}^{\infty}b_k=\infty$ implies $\sum_{k=0}^{\infty}\eta_k=\infty$.
	By the definition of the natural filtration $\mathcal{F}_k,$ we have that $\W(k)$ is adapted, which implies that $\overrightarrow{U}(k)$ is adapted as well. 
    Recalling $\eqref{grad noise}$, it holds
	\begin{align*}
		\E\left[\overrightarrow{U}(k+1) | \mathcal{F}_k\right]&=\E\left[\nabla \Tilde{L}^N_{S_k}\big(\W(k)\big) |       
		\mathcal{F}_k\right]-\nabla \Tilde{L}^N\big(\W(k)\big)=0.
	\end{align*} 
	Altogether, the stochastic gradient descent sequence $\left(\W(k)\right)_k$ fulfills  the Robbins-Monro requirements given in Definition~\ref{def:robbins_monro_algo}.     
	
    We will now prove that the continuous process $\overrightarrow{X}(t)$ is a.s. an APT. We start with the case of deterministic step sizes. 
	Proposition~\ref{prop:bounded} indicates that all iterates $\W(k)$ of SGD remain bounded for all $k\in\NN$. Therefore, using the definition of $\Tilde{L}^N$, Proposition~\ref{prop:bounded} implies that $\nabla \Tilde{L}^N_{S_k}\big(\W(k)\big)$ and $\nabla \Tilde{L}^N\big(\W(k)\big)$ are bounded. Hence, $\overrightarrow{U}(k+1)=\nabla \Tilde{L}^N_{S_k}\big(\W(k)\big)-\nabla \Tilde{L}^N\big(\W(k)\big)$ forms a bounded sequence as well. This implies that
	\begin{align*}
		\sup_{k}\E\left[\norm{\overrightarrow{U}(k+1)}_F^2\right]<\infty
	\end{align*} 
	and 
	\begin{align*}
		\sum_{k=0}^{\infty}\eta_k^2=\frac{1}{\beta^2\lambda^2 \max\left\{\left(\frac{1}{2}\|\W(0)\|_F^2\right), \frac{1}{2}(M+B)^2 \right\}^{2(\beta-1)}}\sum_{k=0}^{\infty}b_k^2<\infty
	\end{align*}
	since $\displaystyle\sum_{k=0}^{\infty}b_k^2<\infty$ by assumption. Altogether, the conditions of  Proposition \ref{prop:crucial_apt_tool} are fulfilled for $q=2$ and hence  condition ($1$) in Proposition~\ref{prop:apt_property} holds with probability one. 
    
    In the case of adaptive stepsizes, we have on the one hand,
	\begin{align}
		\eta_k &= b_k \cdot \min\left\{  \frac{1}{\beta\lambda(\frac12\|\W(k)\|_F^2)^{\beta-1}},\, D \right\} \leq D b_k.
        \label{up_adaptive_etak}
	\end{align}
	On the other hand, the upper bound on $\|W(k)\|_F$ of Proposition~\ref{prop:bounded_adaptive} implies that
	\begin{align}
		\eta_k &\geq \frac{b_k}{\beta\lambda\left(\frac12 \max\left\{ \|\W(0)\|_F, M +D\sqrt N(P_TM^{2N-1}+Q_TM^{N-1})\right\}\right)^{\beta-1}+1/D}\label{lb_adaptive_etak}
	\end{align}
    Since $\displaystyle \sum_{k=0}^{\infty}b_k= \infty,$ inequality \eqref{lb_adaptive_etak} implies $\displaystyle \sum_{k=0}^{\infty}\eta_k=\infty.$

    Moreover,
	inequality $\eqref{up_adaptive_etak}$ yields 
	\begin{align*}
		\sum_{k=0}^{\infty}\eta_k^2 &\leq D^2 \sum_{k=0}^{\infty}b_k^2.
	\end{align*}
	Taking expectation gives
	\begin{align*}
		\E\left[ \sum_{k=0}^{\infty}\eta_k^2\right] \leq D^2 \E\left[ \sum_{k=0}^{\infty}b_k^2\right] < \infty.
	\end{align*}
	
	The boundedness of the SGD iterates $\W(k)$ as stated in Proposition~\ref{prop:bounded_adaptive} implies that 
    $$   \norm{\overrightarrow{U}(k+1)}_F^2
    = \norm{\nabla \Tilde{L}^N_{S_k}\big(\W(k)\big)-\nabla \Tilde{L}^N\big(\W(k)\big)}_F^2 \leq K
    $$ 
    for all $k$ and some (deterministic) constant $K$ so that
$\sup_{k}\E\left[\norm{\overrightarrow{U}(k+1)}_F^2\mid \F_k\right]\leq C$
	for some deterministic constant $C<\infty$. Altogether, SGD with adaptive stepsizes satisfies the requirements of Proposition~\ref{prop:crucial_apt_tool_update} for $q=2$, implying that also in the case of adaptive stepsizes  condition (1) in Proposition~\ref{prop:apt_property} holds with probability one.

    Using boundedness of the $\W(k)$  once more, the $\W(k)$ 
	fulfill condition ($2$) of Proposition \ref{prop:apt_property}. Therefore, almost surely the continuous process $\overrightarrow{X}(t)$ is an asymptotic pseudotrajectory of the gradient flow. 
    
    We now use Corollary~\ref{cor:apt_convergence} to establish property ($1$)  of Theorem \ref{thm:sgd-conv2cset}. 
    Let us check the assumptions of  Corollary~\ref{cor:apt_convergence} with assumption (1) in the sense of Remark \ref{adapted_assumption}.
    The existence and uniqueness requirement follow from Proposition~\ref{global_existence}.
    Boundedness of the SGD iterates implies that the closure $\overline{\operatorname{im}(X)}$ is compact. Thus we obtain that almost surely 
    $\overrightarrow{X}(t)$ is a precompact asymptotic pseudotrajectory, verifying that condition ($2$) in Corollary~\ref{cor:apt_convergence} holds with probability one. Clearly, $\tilde L^N$ is $C^{\infty}$ so that condition (3) of Corollary~\ref{cor:apt_convergence} is also satisfied. Hence, $L^N(\overrightarrow{X}(t))$ converges a.s. as $t \to \infty$ and since $\overrightarrow{X}(t)$ interpolates $\W(k)$, also $L^N(\W(k))$ converges a.s. as $k \to \infty$. It also follows from  Corollary~\ref{cor:apt_convergence} that the limit set of the asymptotic pseudotrajectory $\overrightarrow{X}(t)$ consists of equilibria.

    Let us proceed with the proof of property ($2$) in Theorem \ref{thm:sgd-conv2cset}, i.e., that the sequence of gradients $\nabla\Tilde{L}^N\big(\W(k)\big)$ converges a.s. to zero.
    By contradiction, let us assume that this does not hold. Then there exists $\varepsilon>0$ and a subsequence $\left(\W(k_\ell)\right)_{\ell\in \NN}$ such that \begin{align}\label{lower bound to contradict}
		\norm{\nabla\Tilde{L}^N\big(\W(k_\ell)\big)}_F>\varepsilon\quad \text{for all $\ell$.}
	\end{align}  By Propositions~\ref{prop:bounded} resp.  
	\ref{prop:bounded_adaptive} $\left(\W(k_\ell)\right)_{\ell\in \NN}$ is bounded and therefore, has a converging subsequence. By Corollary~\ref{cor:apt_convergence}, a.s. the limit of this subsequence is an equilibrium, i.e., a critical point, which contradicts $\eqref{lower bound to contradict}.$ This proves that$\nabla\Tilde{L}^N\big(\W(k)\big)$ converges to zero with probability 1.
	
 It  remains to prove property ($3$).
 
    For $\varepsilon >0$, we denote by $S_{\varepsilon}$ the set $\{\W\ |\ \operatorname{d}(\W,\operatorname{Crit}(\Tilde{L}^N))\geq \varepsilon\}$. This is a closed set.
    Indeed, by the triangle inequality, if $\operatorname{d}(\W,\operatorname{Crit}(\Tilde{L}^N))< \varepsilon$, then this inequality also holds in an open neighborhood of $\W$.
    Let $K$ be a compact set with the property that all iterates are in $K$. Such a set $K$  exists by Proposition~\ref{prop:bounded} resp. \ref{prop:bounded_adaptive}. Now 
	suppose that $\operatorname{d}(\W(k), \operatorname{Crit}(\Tilde{L}^N))$ does not converge to zero. Then there exists a positive number $\varepsilon$ and a subsequence $(\W(i_k))_{k\in \mathbb N}$ such that $\W(i_k)\in S_{\varepsilon}$ for all $k\in \mathbb{N}$. Since 
	$S_{\varepsilon}\cap K$ is compact, the continuous map $S_{\varepsilon}\cap K\ni \W\mapsto \norm{\nabla\Tilde{L}^N(\W)}_2$ assumes a minimum $m_{\varepsilon}$ and by the definition of $S_{\varepsilon}$, we have $m_{\varepsilon}>0$. But we already know 
	that $\norm{\nabla\Tilde{L}^N(\W(i_k))}_2$ converges to zero 
	(almost surely) which contradicts $m_{\varepsilon}>0$. Thus indeed $\operatorname{d}(\W(k), \operatorname{Crit}(\Tilde{L}^N)))$ converges to zero almost surely.
 
\end{proof}	

\section{Appendix}\label{sec:appendix}

In this appendix we collect definitions such as equilibrium points, Lyapunov functions and pseudo-orbit as well as some general results about dynamical systems and asymptotic pseudotrajectories (mainly from \cite{benaim1999dynamics}) that are used in our proofs.

\textcolor{black}{ \begin{definition}[Equilibrium Point]\label{def:eqlbrm_pt}
		Let $\Phi$ be a flow or semiflow on the metric space $(M, d).$ An equilibrium point $p$ is a point that satisfies $\Phi_t(p)=p,$ for all $t.$
	\end{definition}
	Critical points of the loss function $\Tilde{L}^N$ are equilibria for the system \eqref{eq:sgf_w}.}
	
\begin{definition}[Lyapunov Function]\label{def:lyapunov_fn} 
Let $(M,d)$ be a metric space and let $\Lambda \subseteq M$ be a compact invariant set of the semiflow $\Phi$.
A Lyapunov function for 
    $\Lambda$ is a continuous function $V: M\to\RR$ that satisfies:
	\begin{enumerate}
		\item The function $\RR_+\ni t\mapsto V(\Phi_t(x))$ is constant for $x\in\Lambda.$ 
		\item For $x\in M\setminus\Lambda,$ the function $\RR_+\ni t\mapsto V(\Phi_t(x))$ is strictly decreasing.
	\end{enumerate}
	If $\Lambda$ is the set of equilibrium points then $V$ is said to be a strict Lyapunov function and $\Phi$  a gradientlike system.

\end{definition}

\begin{definition}[Pseudo-orbit]\label{def:pseudo_orbit}
	Let $\delta>0,$ $T>0$ and let $\Phi$ be a flow or semiflow on the metric space $\big(M, d\big)$ with $a,b\in M.$
	\begin{enumerate}
		\item A finite sequence of partial trajectories 
		\begin{align*}
			\left\{\Phi_t(y_{\ell}): \quad 0\leq t\leq t_{\ell} \right\}; ~\ell=0,\dots,k-1; ~t_{\ell}\geq T
		\end{align*}
		that satisfies 
		\begin{align*}
			&d(y_0,a)<\delta,\\
			&d\big(\Phi_{t_j}(y_j),y_{j+1}\big)<\delta, \quad j=0,\dots,k-1,\\
			&y_k=b
		\end{align*}
		is a $(\delta,T)$-pseudo-orbit from $a$ to $b.$  
		\item If there exists a $(\delta,T)$-pseudo-orbit from $a$ to $b$, we write $a\hookrightarrow_{\delta,T} b.$ If $a\hookrightarrow_{\delta,T} b$ for all $\delta>0,$ $T>0$ then we write $a\hookrightarrow b.$ The flow $\Phi$ is said to be chain transitive if $a\hookrightarrow b$ for every $a,b\in M.$
		\item A subset $N\subset M$ that satisfies $\Phi_t(N)=N,$ for all $t,$ is an invariant set.
		\item An internally chain transitive set is an invariant compact set upon which $\Phi$ is chain transitive.
	\end{enumerate}
\end{definition} 

\begin{theorem}[Theorem 5.7 in \cite{benaim1999dynamics}] \label{thm:chain_transitive_set}
	Let $M$ be a metric space.
	\begin{enumerate}
		\item Let $X$ be an asymptotic pseudotrajectory of $\Phi$ such that the image of $X$ has a compact closure in M. Then, the limit set $L(X)$ is internally chain transitive.
		\item Let $L\subset M$ be an internally chain transitive set, and suppose that M is locally path connected. Then, there exists an asymptotic pseudotrajectory $X$ that satisfies $L(X)=L.$
	\end{enumerate}
\end{theorem}

\begin{proposition}[Proposition 6.4 in \cite{benaim1999dynamics}]\label{prop:chain_transitive_set}
	Let $M$ be a metric space and $\Lambda\subset M$ be a compact invariant set. Let $V:~M\to\RR$ be a Lyapunov function for $\Lambda$ and suppose the interior of $V\left(\Lambda\right)\subset\RR$ is empty. Then every internally chain transitive set $L$ is contained in $\Lambda$ and the restriction of $V$ to $L$ 
    is constant.
\end{proposition}

The following result about global existence of gradient flow is  well-known, cf. e.g. \cite[Theoorem 4.1]{neuberger2010sobolev}. For completeness we include a proof. 
Another similar  statement and proof can be found in \cite[Proposition 2.54]{gallon2022blowphenomenagradientdescent}. 
    \begin{proposition}\label{global_existence}
Let $L:\mathbb R^n\to \mathbb R$ be a $C^2$-function and assume that $L(x)\geq 0$ for all $x\in \mathbb R^n$. Let $x_0\in \mathbb R^n$. Then the  gradient flow 
\[\dot x(t) = -\nabla L(x(t)), \ x(0)=x_0 \]
is globally defined in the sense that it has a solution that is defined for all  $t\geq -\varepsilon(x_0)$, for some positive $\varepsilon(x_0)$. Furthermore $x(t)$ is unique everywhere (i.e., it is uniquely determined for all $t$ in any interval $I$ containing $0$ such that a solution exists on $I$).
\end{proposition}

\begin{proof}
	The existence of a solution in an interval $(-\varepsilon(x_0),+\varepsilon(x_0))$ and the uniqueness statement follow in the standard way from the Picard-Lindelöf theorem because $\nabla L$ is of class $C^1$ and hence locally Lipschitz continuous.
	
	Next we show  existence for all $t\geq 0$. Suppose the solution exists in $[0,t]$ for some $t>0$.
Since
\[
\dot x(t)=-\nabla L(x(t)),
\]
we have as usual 
\[
\frac{d}{dt}L(x(t))
=
\langle \nabla L(x(t)),\dot x(t)\rangle
=
-\|\dot x(t)\|_2^2.
\]
Hence
\[
L(x(t))+\int_0^t \|\dot x(s)\|_2^2\,ds=L(x_0).
\]
Since \(L\ge 0\), we obtain
\[
\int_0^t \|\dot x(s)\|_2^2\,ds\le L(x_0).
\]
Hence for any component  $x_j$
\[
|x_j(t)-x_j(0)|=\left|\int_0^t \dot x_j(s)\,ds\right|\leq  \int_0^t \|\dot x(s)\|_2\,ds
\le
\sqrt{t}\,\Bigl(\int_0^t \|\dot x(s)\|_2^2\,ds\Bigr)^{1/2}
\le
\sqrt{t\,L(x_0)}.
\]
Thus, for any finite $T$ such that the solution exists in $[0,T)$, the solution $x(t)$ is bounded for $t\in [0,T)$.   This implies global existence.
\end{proof}

\section{Acknowledgements}
HR and UT acknowledge funding by DFG through the CRC "Sparsity and Singular Structures" (SFB 1481) -- project number 44204750.

All authors acknowledge funding by DAAD (German Foreign Exchange Service) through the project
Understanding stochastic gradient descent in deep learning (grant no: 57417829).

Use of AI: We used ChatGPT to discuss about the proofs of Proposition \ref{prop:crucial_apt_tool_update} and Corollary \ref{cor:apt_convergence}. These statements were essentially proved in Bena\"\i m's paper \cite{benaim1999dynamics}, but we explain his proofs here again providing a few more details. Similarly, ChatGPT was involved in suggesting a reference and proof for the statement of Proposition \ref{global_existence} (which is essentially Theorem 4.1 in \cite{neuberger2010sobolev} and in finding the references \cite{CoddingtonLevinson1955} and \cite{gallon2022blowphenomenagradientdescent}.

\bibliographystyle{abbrv}
\bibliography{references}
	
\end{document}